\documentclass[12pt]{article}
\usepackage{smile_1}
\usepackage{mathtools}
\usepackage{booktabs}
\usepackage[english]{babel}
\usepackage[protrusion=true,expansion=true]{microtype}
\usepackage{amsmath,amsfonts,amsthm}
\usepackage{ amssymb }
\usepackage{enumitem}
\usepackage{graphicx}
\usepackage{array}
\newcommand{\Halmos}{ $\square$}
\usepackage[toc,page]{appendix}
\usepackage{booktabs}
\usepackage{natbib}
\usepackage{microtype}
\usepackage[font = small,labelfont=bf,textfont=it]{caption}
\usepackage{footnote}
\usepackage{multirow}
\usepackage{algorithm}
\usepackage{mathrsfs}
\usepackage[noend]{algorithmic}

\newcommand{\FIGURE}[3]{
  \begin{center}
    #1
    \caption{#2}
    \par\bigskip
    \begin{minipage}{0.8\textwidth}
      \small #3 
    \end{minipage}
  \end{center}
}

\long\def\TABLE#1#2#3{
  \begin{center}
    \small
    \captionof{table}{#1}
    \medskip
    #2
    \par\smallskip
    \begin{minipage}{0.95\textwidth}
      \footnotesize #3
    \end{minipage}
  \end{center}
}

\newcommand{\up}{\rule{0pt}{2.5ex}}
\newcommand{\down}{\rule[-1.2ex]{0pt}{0pt}}

\usepackage{caption}
\usepackage[CJKbookmarks=true,
            bookmarksnumbered=true,
			bookmarksopen=true,
			colorlinks=true,
			citecolor=blue,
			linkcolor=blue,
			anchorcolor=red,
			urlcolor=blue]{hyperref}
\usepackage{rotating}
\usepackage{fancyvrb}
\usepackage[dvipsnames]{xcolor}

 \makeatletter
\def\mathcolor#1#{\@mathcolor{#1}}
\def\@mathcolor#1#2#3{\protect\leavevmode
  \begingroup
    \color#1{#2}#3\endgroup
}
\makeatother

\usepackage[margin=1in,hmarginratio=1:1,top=20mm,columnsep=20pt]{geometry}
\usepackage{microtype}
\usepackage[font = small,labelfont=bf,textfont=it]{caption}
\usepackage{footnote}
\usepackage{multirow}

\usepackage{color}

\title{Attack-Resistant Uniform Fairness for Linear and Smooth Contextual Bandits}

\author{Qingwen Zhang\\
    Department of Mathematics, \\National University of Singapore\\
   Wenjia Wang\footnotemark[1]\\\
    Department of Industrial Engineering and Decision Analytics, \\National University of Singapore\\
}
\footnotetext[1]{Corresponding author: wenjiawang@nus.edu.sg}
\date{}

\begin{document}
\maketitle

\begin{abstract}
Modern digital platforms use contextual bandits to allocate valuable exposure and opportunities among competing participants. 
Fair treatment is therefore an important concern, yet reward maximization alone does not ensure that preferential allocation reflects participants' merits.
We develop algorithms for linear and smooth contextual bandits under uniform merit-based fairness, requiring the reward ordering to justify preferential allocation across all contexts and rounds, and study how these guarantees are affected by adversarial reward corruption.
Our algorithms achieve \((1-\widetilde O(1/T))\)-fairness, with regret that is minimax optimal among fair
policies for linear rewards and nearly minimax optimal for
smooth rewards. In the linear setting, matching lower bounds identify the price of fairness exactly: minimax regret increases from \(\log T\) to \(\log^2 T\). For smooth rewards, the cost of fairness
is at most polylogarithmic. We further establish a separation between regret and fairness robustness: an \(\widetilde O(1)\) corruption budget can cause substantial fairness violations without worsening the regret order. We therefore develop robust algorithms that adapt sampling, estimation, and fairness certification to corruption, which preserve uniform fairness and achieve minimax-optimal and nearly optimal regrets for linear and smooth rewards, respectively.
Numerical and semi-synthetic experiments illustrate these findings.
\end{abstract}
\vfill
\newpage
\section{Introduction}

Digital platforms, service systems, and data-driven marketplaces increasingly rely on sequential decision algorithms to determine how resources, opportunities, and exposure are allocated among competing entities \citep{singh2024future}. Consider, for example, a short-video platform that decides which creator's content to show to
an arriving viewer \citep{violot2024shorts}. The platform observes the
viewer's profile and session information, selects one creator from a
candidate pool, and receives feedback such as watch time or engagement.
The same structure appears in online advertising, where a platform chooses
which advertiser to display \citep{choi2020online}; in online marketplaces,
where a platform decides which seller to feature; and in on-demand labor
platforms, where a matching system determines which worker is offered a job
\citep{benjaafar2022labor}. In each case, the platform allocates a scarce
operational resource---exposure or opportunity---among competing alternatives
whose values depend on the arriving demand and must be learned over time.
Contextual bandits provide a natural framework for such sequential decisions
\citep{LattimoreSzepesvari2020}.

Exposure, however, matters not only for the platform's reward but also for the opportunities available to supply-side participants. Creators, sellers, advertisers, and workers benefit from the attention or assignments they receive, and persistent differences in treatment that are not justified by merit can undermine trust and participation. Classical contextual
bandit policies balance information acquisition and reward maximization, but
do not require preferential allocation to be justified by the arms'
true rewards \citep{joseph2016fairness}. Therefore, statistical uncertainty can translate into differences in exposure even when those differences
are not supported by the underlying reward ordering.

Motivated by this concern, we study \emph{item-level merit-based fairness}.
Each arm represents a supply-side alternative, and an arm receives a
higher selection probability than another only if its true expected reward is strictly higher for the same context. This
supply-side criterion differs from user-level fairness aimed at protecting
demographic or other demand-side groups, and is related to the principle
that ``similar items be treated similarly'' \citep{dwork2012fairness} and builds
on meritocratic fairness in sequential decision-making
\citep{joseph2016fairness}. Contextual heterogeneity makes this requirement
more demanding than in the classical multi-armed bandit setting: the
relative merits of two arms can change across contexts, and aggregate
criteria may conceal violations concentrated in rarely observed regions
of the context space.

We therefore impose a \emph{uniform fairness} requirement. With high
probability, the merit-based reward ordering must hold simultaneously
for every context in the support and every decision round. Compared with
weaker variants that impose the similar merit-based principle only on
realized contexts, on average, or over selected subsets of contexts or
rounds, the uniform criterion provides a single guarantee over the entire context space and learning horizon. A central question is whether such a strong fairness commitment
can be made without substantially degrading learning efficiency.

A second challenge arises because the reward signals used to establish merit
may themselves be manipulated. Sellers, advertisers, and content creators
may have incentives to influence ratings, click-through rates, or engagement
signals that feed into future allocation decisions
\citep{garcelon2020adversarial,jun2018adversarial,liu2019data}. We model such
behavior through adversarial reward corruption subject to a total budget. In a standard
bandit, the consequences of manipulation are typically evaluated through
reward or regret. Under merit-based fairness, however, reward observations
play an additional role: they provide the statistical evidence used to justify preferential treatment. Corruption can therefore distort relative treatment
even when its effect on cumulative reward is small. This distinction is
particularly sharp for tied or nearly tied arms, where a substantial exposure
imbalance need not generate comparable regret. Fairness robustness can thus
become a separate requirement from regret robustness.

These observations motivate three questions:
\textit{Can uniform merit-based fairness be achieved at regret rates close to what is achievable without fairness constraints? Can strategic manipulation substantially distort fairness without being revealed by regret? And can fairness and reward guarantees be preserved simultaneously under adversarial corruption?}

Our results lead to three main conclusions. First, uniform merit-based
fairness can be achieved at a limited statistical cost: in the regimes where
the classical contextual bandit benchmarks are characterized \citep{goldenshluger2013linear,hu2022smooth}, our fair algorithms match the corresponding unconstrained regret benchmarks up to logarithmic factors. Second, fairness and
regret can respond very differently to strategic manipulation: a small
corruption budget can create persistent exposure distortions without changing
the order of regret.  
Third, corruption-aware algorithms preserve uniform fairness while achieving
minimax-optimal regret in the linear setting and minimax-optimal regret up to
logarithmic factors for smooth rewards when $\alpha\beta\leq d$, where $\alpha$ is the margin exponent, $\beta$ is the smoothness parameter, and $d$ is the context dimension.
The
corresponding rates and optimality results are summarized in
Table~\ref{tab:theory-summary}.

Taken together, these findings have direct implications for platforms that
allocate exposure through learning algorithms. A strong merit-based fairness
commitment can be maintained with limited loss of learning efficiency in the
regimes studied. At the same time, reward or regret monitoring alone may fail
to reveal persistent distortions in exposure, making fairness robustness a separate design requirement from regret robustness. Robust implementation
further requires calibrating exploration, estimation, and fairness
certification to the anticipated level of manipulation, making corruption
tolerance an operational design parameter of the allocation policy.

\subsection{Main Contributions}

We develop a framework for uniform item-level fairness in contextual bandits
and study how such fairness interacts with adversarial reward corruption.
Our analysis covers both linear rewards and nonparametric H\"older-smooth
rewards with smoothness parameter $\beta>1$.

\subsubsection{Algorithmic Design for Uniform and Robust Fairness.}

Building on meritocratic fairness in sequential decision-making
\citep{joseph2016fairness}, we formulate $(1-\delta)$-uniform fairness,
requiring preferential allocation to be justified by the true reward ordering
simultaneously over contexts and decision rounds. We develop fair algorithms that
introduce unequal selection probabilities only after the corresponding
reward ordering can be statistically certified.

For linear rewards, the algorithm combines the two-estimator structure used
in ordinary least squares (OLS) contextual bandits
\citep{goldenshluger2013linear,bastani2020online} with chaining-based arm
selection. For smooth rewards, we follow the epoch-based local-polynomial
framework of \citet{hu2022smooth} while redesigning the active-set updates,
epoch schedule, and estimation tolerances to support uniform reward
comparisons over the context space. We then extend both algorithms to accommodate corruption.  In the linear setting, this requires corruption-aware
exploration and comparison thresholds; in the smooth setting, corruption
additionally changes the spatial scale of local estimation, requiring
coordinated adaptation of epoch lengths, bandwidths, and
fairness-certification thresholds.

\subsubsection{Sharp Theoretical Guarantees and Technical Mechanisms.}

\begin{table}[t]
\TABLE
{Summary of Main Theoretical Guarantees.
\label{tab:theory-summary}}
{\small
\renewcommand{\arraystretch}{1.28}
\setlength{\tabcolsep}{2.5pt}
\begin{tabular}{@{}
>{\centering\arraybackslash}m{0.10\textwidth}
>{\centering\arraybackslash}m{0.12\textwidth}
>{\centering\arraybackslash}m{0.13\textwidth}
>{\centering\arraybackslash}m{0.17\textwidth}
>{\centering\arraybackslash}m{0.26\textwidth}
>{\centering\arraybackslash}m{0.15\textwidth}
@{}}
\hline\up
Reward model
& Setting
& Fairness
& Upper bound
& Lower bound
& Optimality
\down\\
\hline\up

\multirow{2}{*}{Linear}
& Uncorrupted
& $1-1/T$
& $O(\log^2 T)$
& $\Omega(\log^2 T)$
& Optimal
\\

\cline{2-6}

& Corrupted
& $1-1/T$
& $O(\log^2 T+C)$
& $\Omega(\log^2 T+C)$
& Optimal
\\[1mm]

\hline\up

\multirow{2}{*}{Smooth}
& Uncorrupted
& $1-\widetilde O(1/T)$
& $\widetilde O(\psi_T)$
& $\Omega(\psi_T)^\dagger$
& Up to logs$^\dagger$
\\

\cline{2-6}

& Corrupted
& $1-\widetilde O(1/T)$
& $\widetilde O(\psi_T+\chi_{C,T})$
&
\begin{tabular}[c]{@{}c@{}}
$\Omega(\psi_T+\chi_{C,T})^\dagger$\\
$\Omega(C)$ otherwise
\end{tabular}
&
\begin{tabular}[c]{@{}c@{}}
Full rate: up to logs$^\dagger$\\
$C$-term: optimal
\end{tabular}
\down\\

\hline
\end{tabular}
}
{Here, $T$ is the horizon, $C$ is the total corruption budget,
$\beta$ is the H\"older smoothness parameter, $\alpha$ is the margin exponent,
and $d$ is the context dimension. We define
$\psi_T:=T^{(\beta+d-\alpha\beta)/(2\beta+d)}$, $\chi_{C,T}:=
C^{\frac{\beta(1+\alpha)}{\beta+d}}
T^{\frac{d-\alpha\beta}{\beta+d}}$ if $ \alpha\beta<d$ and $\chi_{C,T}:=
C$ if $ \alpha\beta\ge d$. For smooth rewards, $^\dagger$ denotes the regime $\alpha\beta\le d$, in which
the full regret rate is minimax optimal up to logarithmic factors. When
$\alpha\beta>d$, the full minimax rate is not characterized; under corruption,
the corruption-dependent penalty nevertheless has the optimal order $C$.
Optimality in the table refers to minimax rate optimality. The corruption
lower bounds are stated for $1\le C\le T$.
}
\end{table}

For linear rewards, the matching upper and lower bounds show that uniform
fairness entails one additional logarithmic factor relative to the classical
unconstrained benchmark. This additional logarithmic term in lower bound is driven by a ``Fairness-Induced Confusion Zone'': while competing reward orderings remain
statistically unresolved, a fair learner must continue randomizing rather
than prematurely favoring one arm. 
Thus, the additional logarithmic cost is intrinsic to the uniform fairness requirement, rather than an artifact of a particular algorithm.
For H\"older-smooth rewards, the main difficulty is that each arm is learned only from an adaptively generated subset of the context space, while fairness must hold uniformly over all contexts. We show that a
local boundary-growth condition on the reward functions, together with the
active-set dynamics, preserves the local coverage needed for uniform
estimation across epochs. This allows the fair smooth algorithm to attain the
benchmark nonparametric rate up to logarithmic factors.

We further establish a separation between fairness and regret under strategic
reward manipulation. A covert attack with only $\widetilde O(1)$ corruption
budget can induce persistent unfairness and linearly many unfair decisions
while leaving the regret order unchanged. A catastrophic attack can induce both
persistent unfairness and linear regret. Thus, corruption can affect
merit-based allocation through channels that are not captured by regret
alone.
The robust algorithms preserve uniform fairness under attacks with total corruption budget $C$. In the linear setting, the resulting regret bound is minimax optimal.
For smooth rewards, the corruption term exhibits a phase transition at
$\alpha\beta=d$: below this boundary, corruption interacts with the
probability mass of contexts near decision boundaries and produces a coupled
$C$--$T$ term, whereas at and above the boundary, the additional corruption
cost reduces to order $C$. Matching lower bounds establish the minimax optimality statements summarized in Table~\ref{tab:theory-summary}.

\subsubsection{Numerical and Semi-Synthetic Evidence.}

We evaluate the proposed methods for both linear and smooth contextual
bandits with and without adversarial reward corruption. In the absence
of corruption, the fair algorithms substantially reduce unfair decisions while maintaining
competitive regret. Under corruption, the experiments illustrate the
fairness-specific vulnerability identified by the theory and show that the
robust variants substantially mitigate both fairness and reward degradation.
We further conduct a semi-synthetic wine-brokerage study based on the UCI
Wine Quality Dataset \citep{wine_quality_186}, combining real covariates with
controlled reward and corruption mechanisms. 
The results exhibit similar qualitative patterns as the synthetic experiments.

\subsection{Related Work}

Our work connects the operations literature on fair allocation with research
on fairness in sequential learning, contextual bandits, and adversarial
robustness.

\textbf{Fairness in Online Allocation and Sequential Decision-Making.}
Classical work in operations research has examined the efficiency--fairness
trade-off in resource allocation, including the price of fairness
\citep{bertsimas2011price,bertsimas2012efficiency}.
Fairness considerations have been incorporated into online allocation,
dynamic rationing, and dynamic pricing
\citep{ma2020group,balseiro2021regularized,manshadi2023fair,
cohen2022price,cohen2025dynamic,chen2025utility}.
Related work also studies the evaluation of fairness when protected-class
information is unavailable
\citep{chen2019fairness,kallus2022assessing}. 
Fair exposure has also been studied in ranking and recommendation systems
\citep{biega2018equity,singh2018fairness}. This literature is particularly
related to our supply-side perspective, where exposure itself is an
economically meaningful allocation. Our setting introduces two additional
features: the merits of the competing alternatives are initially unknown and
must be learned sequentially, and the reward evidence used to infer those
merits may itself be strategically manipulated.

Fairness-aware bandit models differ in whether fairness is imposed on the
demand side or the supply side. User-level approaches seek to control
disparities across protected groups, as in Fair-LinUCB
\citep{huang2022achieving}. \citet{grazzi2022group} study group meritocratic
fairness in linear contextual bandits, where candidates are compared through
their relative ranks within sensitive groups. Item-level approaches instead
link an arm's treatment directly to its merit. \citet{wang2021fairness} study merit-based exposure allocation in stochastic multi-armed and linear bandits, where exposure is allocated according to arm merit while reward is optimized simultaneously. \citet{liu2017calibrated} study smooth and calibrated fairness in multi-armed bandits, while other work imposes minimum-selection or rate constraints
\citep{patil2021achieving,claure2020multi,chen2020fair}. Most closely related, \citet{joseph2016fairness} introduce meritocratic fairness for contextual bandits, and \citet{joseph2018meritocratic} extend it to linear contextual and infinite-bandit settings. Their model allows arbitrary, potentially arm-specific context sequences, whereas we consider fixed arms facing a common stochastic context; accordingly, their fairness criterion applies to the contexts presented in each round, while ours requires the fairness condition to hold simultaneously for every context in the support and every round on a single high-probability event. Moreover, we establish sharp logarithmic guarantees for linear rewards and develop new uniform local-estimation arguments for H\"older-smooth rewards. We further study robustness to adversarial reward corruption.

\textbf{Linear and Smooth Contextual Bandits.}
The contextual bandit literature provides the statistical benchmarks against
which we measure the cost of uniform fairness. Linear contextual bandits
constitute a foundational setting \citep{LattimoreSzepesvari2020}, with sharp
regret guarantees under various exploration schemes
\citep{goldenshluger2009woodroofe,goldenshluger2013linear,
bastani2020online}; under additional covariate-diversity conditions, greedy
policies can also attain optimal rates \citep{bastani2021mostly}. For nonparametric rewards, smooth contextual bandits have been studied by
\citet{rigollet2010nonparametric,slivkins2014contextual}.
\citet{perchet2013multi} develop adaptive binning for $\beta\in(0,1]$,
\citet{hu2022smooth} establish minimax-optimal guarantees for higher-order
smoothness using epoch-based local-polynomial estimation, and
\citet{gur2022smoothness} study adaptation when the smoothness level is
unknown.
These works
characterize achievable regret rates under smooth or structured reward
models. Our setting builds on these statistical benchmarks while additionally
requiring merit-based reward ordering to hold uniformly over the context
space.

\textbf{Robustness and Adversarial Attacks.}
A growing literature studies how reward manipulation or corruption affects
bandit learning. Adversarial attack constructions show that adaptive
manipulation can substantially alter the behavior or regret of standard
algorithms
\citep{garcelon2020adversarial,jun2018adversarial,liu2019data,zuo2024near}.
Robust algorithms have correspondingly been developed under $C$-total
corruption
\citep{lykouris2018stochastic,gupta2019better,bogunovic2021stochastic,
zimmert2019optimal}. For contextual models, \citet{he2022nearly} establish a
minimax lower bound for linear contextual bandits, \citet{ye2023corruption}
use Eluder dimension to characterize regret for general nonlinear function
classes, and \citet{liu2024corruption} study minimax regret under strong and
weak corruption. Related work also provides optimal guarantees for Lipschitz
reward structures under strong adversaries \citep{kang2023robust}.

These studies primarily evaluate robustness through reward and regret. In our setting, reward observations play an additional role: they provide the evidence used to justify preferential treatment. Corruption can therefore compromise the fairness certification itself, creating persistent distortions in relative treatment even when its effect on regret is small. This introduces a robustness objective beyond standard reward protection. We study how to preserve both fairness and regret guarantees under corrupted feedback for linear and H\"older-smooth contextual rewards.

\subsection{Notation and Organization}

Throughout this paper, let $a\vee b:=\max(a,b)$ and
$a\wedge b:=\min(a,b)$ for two real numbers $a$ and $b$. We write
$\lceil a\rceil$ for the smallest integer greater than or equal to $a$ and
$\lfloor a\rfloor$ for the largest integer smaller than or equal to $a$.
For a positive integer $n$, let $[n]=\{1,\ldots,n\}$. We write
$b_n\lesssim a_n$ if $b_n\leq Ca_n$ for some constant $C>0$. We use
$c,c_1,c_2,\ldots$ to denote generic positive constants whose values may
change from line to line. Let
$v_d=\pi^{d/2}/\Gamma(d/2+1)$ denote the volume of the unit ball in
$\RR^d$, and let $\mathbb I(\cdot)$ denote the indicator function.

The remainder of the paper is organized as follows.
Section~\ref{sec:banditwithdeffair} introduces the problem setting and the
uniform fairness criterion. Section~\ref{sec:fairalg} develops fair
algorithms for linear and smooth contextual bandits and analyzes their
optimality. Section~\ref{sec:effectiveattack} characterizes their
vulnerability to adversarial reward manipulation.
Section~\ref{sec:fairrobustalg} develops robust fair algorithms and
establishes the corresponding upper and lower bounds.
Section~\ref{sec_numerical} presents numerical and semi-synthetic evidence,
and Section~\ref{sec_conclusion} concludes with limitations and directions
for future research.

\section{Preliminaries}\label{sec:banditwithdeffair}

We consider a contextual bandit with a fixed finite arm set
$\cK=\{1,\ldots,K\}$, where $K\geq2$, as in
\citet{goldenshluger2013linear,bastani2020online}.\footnote{This differs from
contextual-bandit formulations with time-varying action sets, in which the
available arms themselves may change across rounds.}
At each round $t=1,\ldots,T$, a context $\bx_t\in\cX$ is drawn independently
from a fixed distribution $\PP_X$ supported on $\cX$. After observing $\bx_t$,
the decision-maker selects an arm $\pi_t\in\cK$ and receives the corresponding
reward. Each arm $k\in\cK$ is associated with an unknown reward function
$f_k^*:\cX\rightarrow\RR$. Specifically, the potential reward from arm $k$ is
\begin{equation}\label{eq:yi}
    y_t(k)
    =
    f_k^*(\bx_t)+\varepsilon_{k,t},
    \qquad \mbox{if }\pi_t=k,
\end{equation}
where $\varepsilon_{k,t}$ are mean-zero independent
$\sigma$-sub-Gaussian random variables and are independent of the context
sequence.

Let $\cF_{t-1}
=
\sigma(\bx_1,y_1,\pi_1,\ldots,\bx_{t-1},y_{t-1},\pi_{t-1})$
denote the history available before round $t$. An admissible policy $\pi$
selects $\pi_t$ using only $\cF_{t-1}$, the current context $\bx_t$, and
possibly exogenous randomization. For notational convenience, $p(\pi_t=k\mid \bx_t=\bx,\cF_{t-1})$ denotes the probability assigned by policy $\pi$ to arm $k$ after observing context $\bx$ and history $\cF_{t-1}$. Moreover, let $\cF_{t-1}^+:=\sigma(\cF_{t-1},\bx_t)$
denote the information available after observing the current context
$\bx_t$ and before selecting the arm at round $t$.

An oracle decision-maker knows all reward functions and selects an arm in
$\argmax_{k\in\cK}f_k^*(\bx_t)$ at each round. The performance of an
admissible policy $\pi$ is measured by its expected cumulative regret relative
to this oracle:
\[
R_T(\pi)
:=
\mathbb{E}
\left[
\sum_{t=1}^T
\left(
\max_{k\in\cK}f_k^*(\bx_t)
-
f_{\pi_t}^*(\bx_t)
\right)
\right],
\]
where the expectation is taken over the contexts, reward noise, and any
randomization of the policy. While $R_T(\pi)$ measures the performance of a
particular policy, the \emph{minimax regret} characterizes the intrinsic
difficulty of the learning problem by taking the infimum over admissible
policies and the supremum over problem instances in the specified class.

The conventional objective is to minimize cumulative regret. Regret, however, is an aggregate performance measure and does not control how merit-inconsistent allocation decisions are distributed across arms. A policy can achieve small cumulative regret while repeatedly disadvantaging the same arms, causing some participants to receive less favorable treatment over time. For a participant, receiving substantially less exposure than an equally meritorious peer over an extended period can be consequential even when the platform's overall regret remains small. This motivates us to control relative treatment directly during learning. We adopt a merit-based criterion under which preferential treatment must be supported by the underlying reward ordering. Throughout, we use merit to refer to an arm's expected reward for the given context.

\begin{definition}[$(1-\delta)$-fairness]\label{def:fairness}
An algorithm $\mathcal{A}$ is $(1-\delta)$-fair if, with probability at least
$1-\delta$, for all rounds $t\in[T]$, all contexts $\bx\in\cX$, and all pairs
of arms $i,j\in\cK$,
\begin{align}\label{eq:defoffairness}
p(\pi_t=i\mid \bx_t=\bx,\cF_{t-1})
>
p(\pi_t=j\mid \bx_t=\bx,\cF_{t-1})
\text{ only if }
f_i^*(\bx)>f_j^*(\bx).
\end{align}
\end{definition}

Definition~\ref{def:fairness} formalizes a natural commitment to the supply side of a platform: differential treatment should be justified by differential merit uniformly across contexts and rounds. The criterion is related to existing notions of merit-based fairness
\citep{dwork2012fairness,joseph2016fairness,liu2017calibrated,biega2018equity}. Under our stochastic-context model, it requires a single high-probability event on which the merit-ordering condition holds over the entire context support and learning horizon. Fairness in expectation, minimum-exposure constraints, and realized-context guarantees can still overlook merit-inconsistent treatment concentrated in rarely observed regions. By constraining the decision rule beyond a particular sample path, our criterion protects arms even in contexts that are rarely observed. The parameter $\delta$ controls the failure probability of this joint guarantee.

\section{Fair $K$-armed Contextual Bandit Algorithms}
\label{sec:fairalg}

This section develops uniformly fair algorithms for linear and smooth
contextual bandits. We first study the linear setting, which provides a
tractable benchmark for understanding the statistical cost of fairness, and
then extend the analysis to nonparametric smooth rewards. For each setting,
we present the algorithm and establish corresponding fairness and regret
guarantees.

\subsection{Linear Contextual Bandit Problem}
\label{sec:fairalglinear}

We start with a contextual bandit whose reward functions are linear in the
covariates. For each arm $k\in\cK$, let $\beta_k\in\RR^d$ be an unknown
parameter vector. Writing $\bz_t=(1,\bx_t^{\mathrm T})^{\mathrm T}$, the
expected reward of arm $k$ is
\begin{align}
    f_k^*(\bx_t)=\bz_t^{\mathrm T}\beta_k.
    \label{eq:linearreward}
\end{align}

Let $\mathcal{J}\subseteq[T]$ denote a set of observations used for
estimation. Given the design matrix
$Z=\{\bz_s:s\in\mathcal{J}\}\in\RR^{|\mathcal{J}|\times d}$
and response vector
$Y=\{y_s:s\in\mathcal{J}\}\in\RR^{|\mathcal{J}|}$,
the ordinary least squares (OLS) estimator is
\begin{align}
    \hat{\beta}_k(\mathcal{J})
    =
    \argmin_{\beta\in\RR^d}
    \frac{1}{|\mathcal{J}|}\|Y-Z\beta\|_2^2
    =
    (Z^{\mathrm T}Z)^{-1}Z^{\mathrm T}Y.
    \label{eq:betahat}
\end{align}

Before presenting the algorithm, we impose the following assumptions.

\begin{assumption}
\label{assum_parabound}
The contexts $\{\bx_t:t=1,2,\ldots\}$ are i.i.d.\ random variables having a
density with respect to Lebesgue measure and drawn from a fixed distribution
$\PP_X$ with support $\cX$. There exist positive constants $r\geq1$ and $b$
such that $\cX\subseteq[-r,r]^{d-1}$ and
$\|\beta_k\|_2\leq b$ for all $k\in\cK$.
\end{assumption}

\begin{assumption}[Margin Condition]
\label{assum_margin1}
There exists a constant $L$ such that, for any pair $i,j\in\cK$ with
$i\neq j$ and any $\rho>0$, $\PP\left(
    0<
    |\bz_t^{\mathrm T}(\beta_i-\beta_j)|
    \leq \rho
    \right)
    \leq L\rho.$
\end{assumption}

\begin{assumption}
\label{assum_excite}
There exist positive constants $h$ and $\tilde p$ such that, for each arm
$k\in\cK$, $\PP\left[\bz_t\in Q_k\right]\geq\tilde p,$
where $Q_i
=
\left\{
\bz:
\bz^{\mathrm T}\beta_i
>
\max_{j\neq i}\bz^{\mathrm T}\beta_j+h
\right\}.$
Moreover, there exists a positive constant $\lambda^*$ such that $\min_{i\in\cK}
\lambda_{\min}
\left\{
\mathbb{E}
\left(
\bz_t\bz_t^{\mathrm T}
\mid
\bz_t\in Q_i
\right)
\right\}
\geq\lambda^*.$
\end{assumption}

The above conditions are closely related to benchmark assumptions used in
the analysis of linear contextual bandits
\citep{goldenshluger2013linear,bastani2020online}.
Assumption~\ref{assum_parabound} controls the context support and parameter
magnitudes, while Assumption~\ref{assum_margin1} limits the probability of
contexts with nearly tied rewards. Assumption~\ref{assum_excite} requires each arm to be uniquely optimal by a
fixed margin on a positive-probability region with a nondegenerate
conditional covariance matrix. This provides a persistent source of
informative samples for each arm and is used to control the all-sample OLS
estimators under adaptive arm selection. Thus, the margin condition governs
the regret contribution from near-tie contexts, whereas Assumption~\ref{assum_excite} supplies the design information needed for estimation.

\subsubsection{Linear Fair Algorithm}

To enforce fairness under statistical uncertainty, we use an
$\epsilon$-chaining device in the spirit of
\citet{joseph2016fairness}. 
Two arms are treated as locally indistinguishable when their estimated rewards differ by at most $\epsilon$. Taking the transitive closure of this relation groups together arms that are connected through a sequence of such comparisons.
The threshold $\epsilon$ is calibrated to the current estimation
uncertainty and decreases as more information becomes available. This allows
the policy to assign different selection probabilities only when the data
provide sufficient evidence to establish the corresponding reward ordering.

\begin{definition}[$\epsilon$-chaining]
\label{def:chain}
For a set $\mathcal{A}\subseteq\mathbb{R}$, if $u,v\in\mathcal{A}$ satisfy
$|u-v|\leq\epsilon$, then $u$ and $v$ are called
\textbf{$\bepsilon$-linked} in $\mathcal{A}$. Moreover, if $u$ and $v$
belong to the same component of the transitive closure of the
$\epsilon$-linked relation, then they are called
\textbf{$\bepsilon$-chained} in $\mathcal{A}$.
\end{definition}

\begin{algorithm}[ht]
\caption{Fair OLS Contextual Bandit Algorithm}
\label{alg:fairols}
\begin{algorithmic}[1]
\STATE \textbf{Input parameters}: $C_a$, $C_b$, $h$, $T$.
\STATE Initialize $|\cT_0|=C_a\log T$, and for all $k\in\cK$,
$\mathcal{I}_{k,0}=\mathcal{I}_{k,t}=\emptyset$.
\FOR{$t\in\cT_0$}
\STATE Randomly pull arm $\pi_t\in\cK$ with equal probability, and receive
reward $y_t$.
\STATE Update index set
$\mathcal{I}_{\pi_t,0}\leftarrow
\mathcal{I}_{\pi_t,0}\cup\{t\}$
and
$\mathcal{I}_{\pi_t,t}\leftarrow
\mathcal{I}_{\pi_t,t-1}\cup\{t\}$.
\ENDFOR
\STATE Compute the initial estimation
$\hat{\beta}_{k,0}=\hat{\beta}(\mathcal{I}_{k,0})$
as in \eqref{eq:betahat}.
\FOR{$t\in\{|\cT_0|+1,\ldots,T\}$}
\STATE Observe covariate vector $\bx_t$.
\STATE Compute
$\hat{\cK}_{\bx_t}
=
\left\{
k\in\cK:
\bz_t^{\mathrm T}\hat{\beta}_{k,0}
\text{ and }
\max\limits_{l\in\cK}\bz_t^{\mathrm T}\hat{\beta}_{l,0}
\text{ are $h/2$-chained in }
\{\bz_t^{\mathrm T}\hat{\beta}_{k,0}:k\in\cK\}
\right\}$.
\IF{$\hat{\cK}_{\bx_t}=\{k\}$}
\STATE Pull arm $\pi_t=k$.
\ELSE
\STATE Let
$\epsilon_t=C_b\sqrt{\frac{\log T}{t}}$.
Compute the all-sample estimation
$\hat{\beta}_{k,t-1}
=
\hat{\beta}(\mathcal{I}_{k,t-1})$
as in \eqref{eq:betahat}.
\STATE Compute the candidate arm set with chaining relation:
\begin{equation}
    \cK_c(\bx_t)
    =
    \left\{
    k\in\hat{\cK}_{\bx_t}:
    \bz_t^{\mathrm T}\hat{\beta}_{k,t-1}
    \text{ and }
    \max\limits_{l\in\hat{\cK}_{\bx_t}}
    \bz_t^{\mathrm T}\hat{\beta}_{l,t-1}
    \text{ are $\epsilon_t$-chained in }
    \{
    \bz_t^{\mathrm T}\hat{\beta}_{k,t-1}:
    k\in\hat{\cK}_{\bx_t}
    \}
    \right\}.
\end{equation}
\STATE Randomly pull arm $\pi_t\in\cK_c(\bx_t)$ with equal probability,
and receive reward $y_t$.
\ENDIF
\STATE Update index set
$\mathcal{I}_{\pi_t,t}
\leftarrow
\mathcal{I}_{\pi_t,t-1}\cup\{t\}$;
$\mathcal{I}_{k,t}
\leftarrow
\mathcal{I}_{k,t-1}$
for
$k\in\cK\setminus\{\pi_t\}$.
\ENDFOR
\end{algorithmic}
\end{algorithm}

Algorithm~\ref{alg:fairols} combines the two-estimator structure used in OLS
contextual bandits
\citep{goldenshluger2013linear,bastani2020online}
with fairness-aware arm selection. The initial estimator, constructed from
the $O(\log T)$ random-exploration sample, screens out arms that can already
be separated by a fixed margin. For the remaining arms, the all-sample
estimator provides increasingly accurate reward comparisons, with
$\epsilon_t$ shrinking according to the estimation uncertainty. Rather than
greedily selecting the arm with the largest estimate, the algorithm samples
uniformly from the resulting $\epsilon_t$-chained candidate set. Thus, arms
that cannot yet be reliably ordered receive equal treatment, while
preferential selection is introduced only after their reward ordering can be
statistically distinguished.

\subsubsection{Fairness Guarantee and Regret Analysis}

We next establish the fairness and regret guarantees of
Algorithm~\ref{alg:fairols}. Throughout this subsection, we take
$C_a\geq C_{a,0}$ and $C_b\geq C_{b,0}$, where
$C_{a,0},C_{b,0}>0$ are fixed model-dependent constants whose explicit values
are given in the supplement.

\begin{theorem}
\label{thm:linearfairness}
Suppose Assumptions~\ref{assum_parabound}--\ref{assum_excite} hold and
$T>2d+2\sqrt{3K}$. Then the policy defined by
Algorithm~\ref{alg:fairols} satisfies $(1-\frac{1}{T})$-fairness.
\end{theorem}

\begin{theorem}
\label{thm:minimaxoptimallinear}
Assume that the conditions in Theorem~\ref{thm:linearfairness} are satisfied.
Then Algorithm~\ref{alg:fairols} has cumulative regret
\[
R_T=O(\log^2 T).
\]
\end{theorem}

Theorems~\ref{thm:linearfairness} and
\ref{thm:minimaxoptimallinear} show that
Algorithm~\ref{alg:fairols} achieves $(1-1/T)$-fairness with
$O(\log^2T)$ regret. We next show that this regret order is unavoidable
among policies satisfying the same uniform fairness requirement.

\subsubsection{Price of Fairness and Optimality Analysis}
\label{subsubsec_linearfairlb}

Without fairness constraints, the linear contextual bandit class considered
here admits an $\Omega(\log T)$ minimax lower bound
\citep{goldenshluger2013linear}. Our $O(\log^2T)$ upper bound therefore raises
the question of whether the additional logarithmic factor is an artifact of
Algorithm~\ref{alg:fairols} or an inherent cost of uniform fairness. We
formalize this comparison through the \textit{price of fairness}. The
following theorem shows that policies satisfying
$(1-\delta)$-uniform fairness with $\delta=\Theta(1/T)$ necessarily incur
$\Omega(\log^2T)$ regret.

\begin{theorem}
\label{thm:price}
Let $\Pi$ denote the class of all admissible policies satisfying
$(1-\delta)$-fairness with $\delta=\Theta(1/T)$. For the problem class
$\mathcal{P}$ satisfying
Assumptions~\ref{assum_parabound}--\ref{assum_excite},
\begin{align*}
    \inf_{\pi\in\Pi}
    \sup_{\mathbb{P}\in\mathcal{P}}
    R_T(\pi)
    =
    \Omega(\log^2 T).
\end{align*}
\end{theorem}

This lower bound is driven by a ``Fairness-Induced Confusion Zone,'' in which
the available information is insufficient to reliably determine the reward
ordering between competing arms. Uniform fairness then requires 
the learner to continue to assign equal selection probabilities to statistically unresolved arms rather than prematurely favor one over another. 
Although this ambiguity region shrinks as information accumulates, we prove that the resulting cumulative regret is $\Omega(\log^2 T)$. 
Thus, an additional logarithmic factor relative to the classical unconstrained benchmark is unavoidable.

The lower bound in Theorem~\ref{thm:price} matches the
$O(\log^2T)$ upper bound in
Theorem~\ref{thm:minimaxoptimallinear}. Hence
Algorithm~\ref{alg:fairols} is minimax rate-optimal among policies satisfying
the same uniform fairness requirement.

\subsection{Smooth Contextual Bandit Problem}\label{sec:fairalgsmooth}

Although the linear model permits global estimation, it may be too restrictive to capture complex relationships between contexts and rewards.
We therefore build on the smooth
contextual bandit framework of \citet{hu2022smooth}, in which each arm is
associated with a H\"older-smooth reward function. Throughout this
subsection, let $\beta>1$, and let $\mathfrak b(\beta)$ denote the largest
integer strictly smaller than $\beta$. A function
$\eta:\cX\rightarrow[0,1]$ belongs to the $(\beta,L,\cX)$-H\"older class if
it is $\mathfrak b(\beta)$ times continuously differentiable and, for all
$\bx,\bx'\in\cX$,  $$ \left|\eta\left(\bx^{\prime}\right)-\sum_{|r| \leq \mathfrak{b}(\beta)} \frac{\left(\bx^{\prime}-\bx\right)^r}{r!} D^r \eta(\bx)\right| \leq L\left\|\bx^{\prime}-\bx\right\|^\beta.$$ For a comprehensive review of the H\"older class, one can refer to \cite{gilbarg1977elliptic,evans2022partial}. Local polynomial regression is a standard non-parametric method for estimating unknown functions in the  H\"older space. Given observations $S=\left\{\left(\bx_t, y_t\right)\right\}_{t=1}^n$, the bandwidth $h>0$, and integer $l\geq 0$, the local weight is computed by $$
    \hat{\vartheta}_{\bx} \in \arg \min _{\vartheta} \sum_{t: \bx_t \in \mathcal{B}(\bx, h)}\left(y_t-\theta\left(\bx_t ; \bx, \vartheta, l\right)\right)^2,$$
where $\theta(\bu ; \bx, \vartheta, l)=\sum_{|r| \leq l} \vartheta_r(S)(\bu-\bx)^r$ is a degree-$l$ polynomial model. The local polynomial estimator for function $\eta(\bx)$ is defined by $\hat{\eta}(\bx ; S, h, l)=\theta(\bx ; \bx, \hat{\vartheta}_{\bx}, l)$. We denote the estimator for expected reward function $f_k^*(\bx)$ as $\hat{f}^{\mathrm{LP}}_k(\bx ; S, h, l)$. Under an i.i.d. random design, the standard choice
$h\asymp n^{-1/(2\beta+d)}$ gives the offline rate
$O_p(n^{-\beta/(2\beta+d)})$; see, e.g., Theorem 3.2 of
\citet{tsybakov2007fast}. Let $M_\beta=\#\{r\in\mathbb Z_{\geq0}^d:|r|\leq\mathfrak b(\beta)\}$ denote the number of local-polynomial coefficients.

Before introducing the fairness framework, we first formalize the smooth bandit problem by introducing several basic assumptions. The problem class defined by Assumptions \ref{assump:iiddensity} to \ref{assum:margin} falls within the smooth bandit framework introduced by \cite{hu2022smooth}.

\begin{assumption}\label{assump:iiddensity}
The contexts $\{\bx_t:t= 1,2,\ldots\}$ are i.i.d random variables, drawn from a fixed distribution $\mathsf{P}_X$ with a compact support $\cX\subseteq[0,1]^d$ and density function $p(\bx)$. In addition, $p_{\min}\leq p(\bx) \leq p_{\max}$ for all $\bx\in\cX$ and some positive constants $p_{\min}$ and $p_{\max}$.
\end{assumption}

\begin{assumption}\label{assump:betaholder}
    For all $k\in\mathcal{K}$, $f_k^*$ is $(\beta, L, \cX)$-H\"older and $\left(1, L_1, \cX\right)$-H\"older.
\end{assumption}
\begin{assumption}\label{assump:c0r0regular}
 A Lebesgue-measurable set $S$ is \textit{weakly $(c_0,r)$-regular} at point $\bx\in S$ if
    \begin{align*}
        \operatorname{Leb}[S \cap \mathcal{B}(\bx, r)] \geq c_0 \operatorname{Leb}[\mathcal{B}(\bx, r)].
    \end{align*} Define the optimal region of arm $k$ as
    $\mathcal{R}_k=\left\{\boldsymbol{x} \in \cX \mid f_k^*(\boldsymbol{x})=\max _{j \in \mathcal{K}} f_j^*(\boldsymbol{x})\right\}$. Then $\mathcal{R}_k$ is a non-empty set satisfying, for all $0\leq r\leq r_0$, $\mathcal{R}_k$ is weakly $(c_0,r)$-regular at all $\bx\in \mathcal{R}_k$. In this context, $\mathcal{R}_k$ is referred to as a \textit{$(c_0,r_0)$-regular} set.
\end{assumption}

\begin{remark}
    Assumptions \ref{assump:iiddensity} and \ref{assump:c0r0regular} together imply that for all $k\in\mathcal{K}$, there exists a positive constant $p^*$ such that $\PP(\bx_t\in\mathcal{R}_k)> p^*$. 
\end{remark}

\begin{assumption}[Margin Condition]\label{assum:margin}
    There exist positive constants $C_0$ and $\alpha\leq 1$ such that for any $\delta \geq 0$ and $i\neq j$,  $\PP_{X\sim \mathsf{P}_X}(0<|f_i^*(X)-f_j^*(X)|\leq \delta)\leq C_0\delta^\alpha$.
\end{assumption}

Assumption \ref{assump:iiddensity} ensures the contexts are i.i.d. and well-bounded. Assumption \ref{assump:betaholder} guarantees the smoothness of the reward functions. Assumption \ref{assump:c0r0regular} enforces a geometric regularity on the optimal regions. Finally, the margin condition (Assumption \ref{assum:margin}) controls the probability of contexts where the arms are close.
Collectively, these assumptions describe a well-defined problem class for analyzing the regret of smooth bandit algorithms.

\textbf{Challenges and ``inestimable regions''.} The principal difficulty in the nonparametric setting is that estimation is
local. At any epoch, the accuracy of the estimator for arm $k$ near a
context $\bx$ depends on how often that arm has been selected near $\bx$.
Because these selections are adaptive, the arm-specific design can develop
\emph{inestimable regions} in which the local support is too sparse for a
uniform error bound; see \citet{hu2022smooth}. A regret-minimizing policy
may eliminate an arm once it is known not to be optimal. Uniform fairness
requires more: whenever two arms receive different selection probabilities,
their pairwise ordering must be justified. For $K>2$, knowing that an arm
is not optimal does not determine how it should be ranked relative to the remaining arms, so uncontrolled elimination in an inestimable region can violate
fairness.

The two-armed case is exceptional. Once one arm is certified to be
suboptimal, the other arm is necessarily optimal, and there is no third arm
whose relative ordering remains unresolved. This observation yields the
following benchmark result as shown in Proposition \ref{thm:smoothbanditfair}.

\begin{proposition}\label{thm:smoothbanditfair}
    When $K=2$, under Assumptions \ref{assump:iiddensity}-\ref{assum:margin}, for sufficiently large $T$, the policy defined by Algorithm 1 in \cite{hu2022smooth} satisfies ($1-\delta$)-fairness with $\delta = \tilde{O}(1/T)$. 
\end{proposition}

However, this fairness property does not extend directly to the general
multi-armed setting ($K>2$). 
To address this challenge, we develop a fair smooth bandit algorithm
for general $K$ based on epoch-wise local estimation and chaining-based
active-set updates. The conservative update rule prevents an arm from being
discarded without sufficient evidence of inferiority. The main technical
challenge is then to establish that the local estimates required by these
updates remain uniformly valid over the evolving arm-specific supports.
Section~\ref{eq:globalestima} develops the conditions and analysis needed
for this global estimability guarantee.

\subsubsection{Smooth Fair Algorithm}

We present the fair smooth bandit algorithm for general $K\geq2$ in
Algorithm \ref{alg:fairsmoothalg}. Following the epoch-based approach used
in smooth contextual bandits
\citep{Simchi-LeviXu22,hu2022smooth}, the algorithm learns a fresh local
estimator from each epoch and updates context-specific active arm sets.
Uniform fairness changes these updates: an arm can be removed only when the
available estimates reliably certify its inferiority. We therefore design
the epoch schedule, error tolerances, and active-set update rule to support
uniform reward comparisons over the context space. Unlike
\citet{hu2022smooth}, Algorithm \ref{alg:fairsmoothalg} does not explicitly
detect and exclude inestimable regions, thereby avoiding an additional
region-detection procedure. The analysis below instead establishes a
static-to-dynamic support guarantee showing why this simpler implementation
retains the local estimability required for fairness.

 \begin{algorithm}[!h]
\caption{Fair Smooth Bandit Algorithm}\label{alg:fairsmoothalg}
\begin{algorithmic}[1]
\STATE \textbf{Input}: 
Grid representatives $G$, hypercubes $\mathcal C$, H\"older smoothness $\beta$, boundary-growth exponent $\beta'$, lower mass bound $p^*$, regularity constant $c_0$, constant $C_K$, context dimension $d$, and horizon $T$.
\STATE Initialize $\mathcal{K}_{1,j}=\mathcal{K}$ for any $j \in\{1, \ldots,|\mathcal{C}|\}$.
\STATE For each $q\geq1$, define the scheduled epoch length
$\bar n_q=\left\lceil\frac{2K}{p^*}\left(\frac{4^q\log \left(T \delta_A^{-d}\right)}{C_{K}}\right)^{\frac{2 \beta+d}{2 \beta}}\left(\log T\right)^{\frac{2\beta+d}{\beta^{\prime}-1}-\frac{2\beta+d}{2\beta}}+\frac{K^2}{2 p^{*2}} \log T\right\rceil$.
\STATE Set $Q=\min\left\{q\geq1:\sum_{\ell=1}^{q}\bar n_\ell\geq T\right\}$.
\FOR{$q=1,2,\ldots,Q$}
\STATE Set the realized epoch length
$\displaystyle |\cT_q|=
\begin{cases}
\bar n_q,&q<Q,\\
T-\sum_{\ell=1}^{q-1}\bar n_\ell,&q=Q.
\end{cases}$
\STATE Set error tolerance as $\epsilon_q=2^{-q}(\log T)^{\frac{\beta^{\prime}-1-2\beta}{2\beta^{\prime}-2}}$.
\IF{$q=1$}
\FOR{$t \in \mathcal{T}_1$}
\STATE Pull $\mathcal{\pi}_t \in \mathcal{K}$ randomly, equiprobably.
\ENDFOR
\ELSE 
\STATE Set $N_{q-1,k}=\left|\mathcal{T}_{q-1,k}\right|, H_{q-1,k}=N_{q-1,k}^{-1 /(2 \beta+d)}$ for $k \in \mathcal{K}$.
\STATE Compute index of exploration hypercubes $\mathcal{I}_q=\{j:j\in \{1, \ldots,|\mathcal{C}|\},|\cK_{q-1,j}|>1\}$.
\STATE For each $j\in \mathcal{I}_q$, construct the local polynomial estimate for $\bx\in \text{Cube}_j$ and each $k\in \cK_{q-1,j}$:
$$
\hat{f}_{q-1,k}(\bx)=\hat{f}^{\mathrm{LP}}\left(g(\bx) ; \mathcal{T}_{q-1,k}, H_{q-1,k}, \mathfrak{b}(\beta)\right).
$$

\STATE Update active arm sets: for $j\in \{1, \ldots,|\mathcal{C}|\}\setminus \mathcal{I}_q$, $\cK_{q,j}\leftarrow\cK_{q-1,j}$; for $j \in\mathcal{I}_q$ and the representative point $G_j$,
$$
\mathcal{K}_{q,j} \leftarrow \left\{k \in \mathcal{K}_{q-1,j} \left|
    \begin{array}{l}
        \hat{f}_{q-1,k}(G_j) \text{ and } \max\limits_{l \in \mathcal{K}_{q-1,j}}\hat{f}_{q-1,l}(G_j) \\
        \text{are } (2\epsilon_{q-1})\text{-chained in } \{\hat{f}_{q-1,k}(G_j):k\in \mathcal{K}_{q-1,j}\}
    \end{array}
\right.\right\}.
$$
\FOR{$t\in\mathcal{T}_q$}
\STATE Observe context $\bx_t$, pull arm $\pi_t \in \cK_{q,u(\bx_t)}$ with equal probability, and receive reward $y_t$. 
\ENDFOR
\ENDIF
\STATE Log the samples $\mathcal{T}_{q,k}=\left\{\left(\bx_t, y_t\right): t \in \mathcal{T}_q, \mathcal{\pi}_t=k\right\}$ for $k \in \mathcal{K}$.
\ENDFOR
\end{algorithmic}
\end{algorithm}

Algorithm \ref{alg:fairsmoothalg} maintains a context-specific active arm
set on each grid cell. At the beginning of epoch $q$, the active sets are
updated using $(2\epsilon_{q-1})$-chaining based on estimates from the
preceding epoch; within the epoch, the algorithm samples uniformly from the
active set associated with the observed context. We next specify the grid and
local-polynomial estimator used in these updates. Define the auxiliary lattice
$G^{\mathrm{aux}}$ on $[0,1]^d$ by
$G^{\mathrm{aux}}=\left\{\left(\frac{2 j_1+1}{2} \delta_A, \ldots, \frac{2 j_d+1}{2} \delta_A\right): j_i \in\left\{0, \ldots,\left\lceil\delta_A^{-1}\right\rceil-1\right\}, i=1, \ldots, d\right\}$,
where $\delta_A=T^{-\frac{\beta}{2 \beta+d}}(\log T)^{-1}$. For $\bx\in\cX$, let
$\bar g(\bx)=\arg\min_{\bx'\in G^{\mathrm{aux}}}\|\bx-\bx'\|$, with ties resolved by choosing the point closest to the origin. Enumerate the auxiliary grid points with nonempty cells as $\bar G_1,\ldots,\bar G_{|\cC|}$, where
$\operatorname{Cube}_j=\{\bx\in\cX:\bar g(\bx)=\bar G_j\}$.
For each nonempty cell $\operatorname{Cube}_j$, fix an arbitrary
representative point $G_j\in\operatorname{Cube}_j$, and set
$G=\{G_1,\ldots,G_{|\cC|}\}$, $\cC=\{\operatorname{Cube}_1,\ldots,\operatorname{Cube}_{|\cC|}\}$,
$u(\bx)=j$ when $\bx\in\operatorname{Cube}_j$, and $g(\bx)=G_{u(\bx)}$. Then $G\subseteq\cX$, the cells cover $\cX$, and
$\|\bx-g(\bx)\|_2\leq\sqrt d\,\delta_A$ for all $\bx\in\cX$. The local polynomial estimator is
$\widehat f_{q-1,k}(\bx)=\widehat f^{\mathrm{LP}}(g(\bx);\mathcal T_{q-1,k},H_{q-1,k},\mathfrak b(\beta))$.
The parameter $C_K$ is a positive constant given in the supplement.
In practice, one can choose $C_K$ to be a sufficiently small constant.

\subsubsection{Technical Guarantee of Global Estimability}\label{eq:globalestima}

The accuracy of local polynomial estimators depends on the local regularity
of the sampling support, as in Theorem 3.2 of \cite{audibert2007fast}. Let
$S_{q,k}=\{\boldsymbol{x}\in\cX:k\in\mathcal K_{q,u(\bx)}\}$ denote
the sample support of arm $k$ in epoch $q$. Following \citet{hu2022smooth},
an \textit{inestimable region} for arm $k$ consists of points at which
$S_{q,k}$ is not weakly $(c_0/2^d,H_{q,k})$-regular. The main difficulty is
to ensure that adaptive arm elimination does not create such regions within
the evolving support $S_{q,k}$, so that the local estimates needed for subsequent fairness-preserving updates remain uniformly reliable.

To bridge this mismatch, we impose a local condition on the underlying
reward functions. Near the boundary of an arm's optimal region, its reward
gap cannot remain arbitrarily small as one moves away from that region. 
For each arm $k\in\cK$, define
$\Delta_k(\bx):=\max_{j\in\cK}f_j^*(\bx)-f_k^*(\bx)$ and
$\Delta_{k,\max}:=\sup_{\bx\in\cX}\Delta_k(\bx)$.
\begin{assumption}
\label{assump:Qxliplower}
For a given threshold $\rho>0$, there exist constants
$\mathfrak r,\widetilde c>0$ and an exponent $1<\beta'\leq\beta$, such that,
for each arm $k\in\cK$, either $\Delta_{k,\max}\leq\rho$, or
\begin{align}\label{eq:boundary-growth}
\operatorname{dist}(\bx,\cR_k)
\leq \widetilde c\,\Delta_k(\bx)^{\beta'/\beta}
\end{align}
for all $\bx\in\cX\setminus\cR_k$ with
$\operatorname{dist}(\bx,\cR_k)\leq\mathfrak r$, where
$\operatorname{dist}(\bx,\cR_k):=
\inf_{\bu\in\cR_k}\|\bx-\bu\|_2$.
\end{assumption}

In the absence of corruption, we take
$\rho=T^{-\frac{\beta}{2\beta+d}}$, corresponding to the classical
nonparametric estimation scale up to logarithmic factors. The robust analysis
uses
$\rho=T^{-\frac{\beta}{2\beta+d}}+
(C/T)^{\frac{\beta}{\beta+d}}$.
For arms with $\Delta_{k,\max}>\rho$, Assumption~\ref{assump:Qxliplower} controls how the reward gap grows away from the optimal region, while still allowing arbitrarily small gaps near its boundary. No additional boundary-growth condition is imposed on arms with $\Delta_{k,\max}\leq\rho$. This contains the benchmark hard instances for smooth contextual bandits and yields the same lower-bound rate; see Proposition~\ref{thm:thm3inhusmooth}.

The assumption is imposed only on the reward geometry, which can be checked independently of the algorithm. Assumption~\ref{assump:Qxliplower} still allows a broad range of boundary geometries: transverse intersections satisfy it with $\beta'=\beta$, while finite-order flatter contacts are also permitted. Figure~\ref{fig:explainassump} illustrates this interpretation.  
Under this condition, Proposition~\ref{prop:sqkregular} establishes the regularity of the adaptive sampling supports from the algorithmic dynamics, thereby providing the local coverage needed for uniform local-polynomial estimation.

 \begin{figure} \FIGURE{\includegraphics[width=0.5\textwidth]{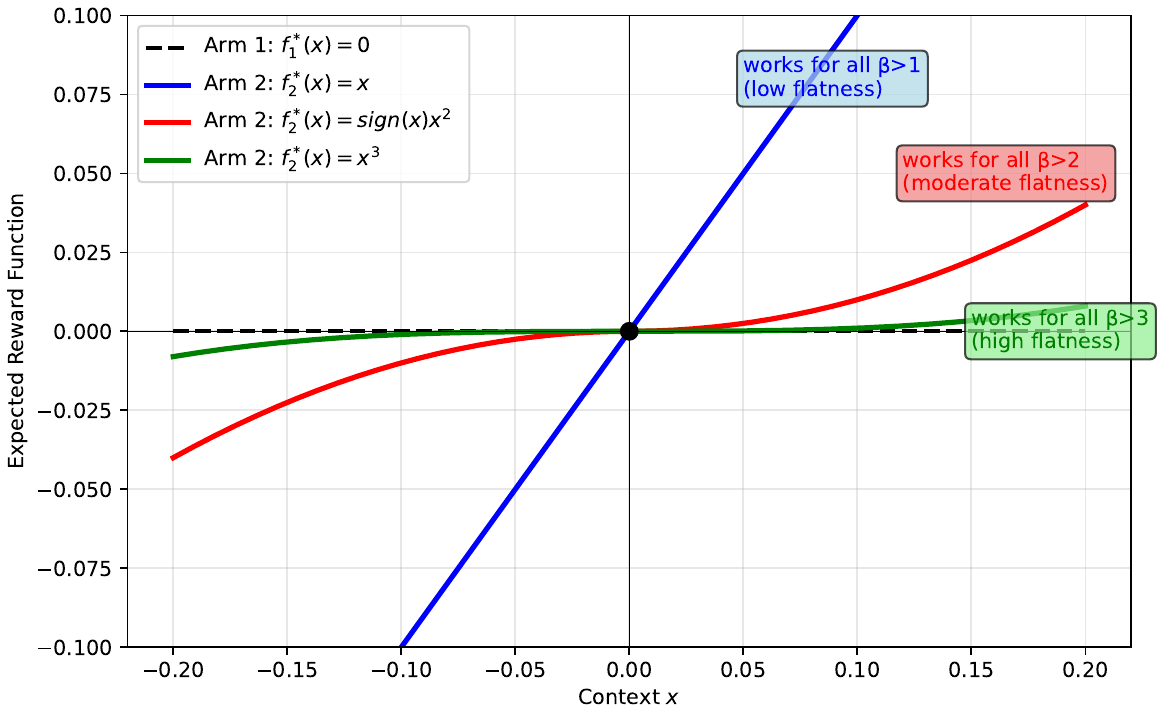}}
{Two-Arm Illustration: Flatness at Decision Boundary Adapts to $\beta$.  \label{fig:explainassump}The low flatness works for all $\beta>1$, the moderate flatness works for all $\beta>2$, and the high flatness works for all $\beta>3$. The plot shows function values near the decision boundary $x=0$ on a subregion $[-0.2, 0.2]$.}
{}
\end{figure}

Assumption~\ref{assump:Qxliplower} is stated in the gap-to-distance form required by the analysis, which may be difficult to verify directly. The following proposition provides a simple sufficient condition in terms of the local derivatives of the reward functions: near a boundary point of $\cR_k$, at least one tied competitor separates from arm $k$ at a uniformly positive first-order rate in the outward direction. 
This condition is considerably stronger than Assumption~\ref{assump:Qxliplower} and is introduced only to provide an easily verifiable sufficient condition.

\begin{proposition}[A Sufficient Condition for Assumption~\ref{assump:Qxliplower}]
\label{prop:explainassump3.8}
Suppose that each $f_k^*$ is continuously differentiable on a neighborhood
of $\cX$. Assume that there exist constants $r,a>0$ such that the following
holds. For any arm $k\in\cK$ and any point
$\bx\in\cX\setminus\cR_k$ satisfying
$\operatorname{dist}(\bx,\cR_k)\leq r$, let $\bx_0$ be a nearest point to
$\bx$ in $\cR_k$. Then there exists an arm $i\neq k$ such that $i$ and $k$
are both optimal at $\bx_0$ and
\[
\nabla(f_i^*-f_k^*)(\bx_0)^{\mathrm T}
(\bx-\bx_0)
\geq
a\|\bx-\bx_0\|_2.
\]
Then Assumption~\ref{assump:Qxliplower} holds with $\beta'=\beta$.
\end{proposition}

For each epoch $q$, let $N_{q,k}=|\cT_{q,k}|$. We introduce two
events: $\mathcal M_q$ controls the number of samples collected for each arm,
while $\mathcal G_q$ controls the regularity of the resulting sample supports
and the corresponding estimation errors:
\begin{align*}
&\mathcal{M}_q  =\left\{\min _{k \in \mathcal{K}} N_{q, k} \geq\left(\frac{4^q\log(T\delta_A^{-d})}{C_{K}}\right)^{\frac{2 \beta+d}{2 \beta}}\left(\log T\right)^{(2\beta+d)\frac{2\beta-\beta^{\prime}+1}{2\beta(\beta^{\prime}-1)}}\right\}, \\
&\mathcal{G}_q = \left\{
\begin{aligned}
& \text{(i) } \text { for all }k\in\cK,  S_{q,k} \text{ is weakly }  (\frac{c_0}{2^d},H_{q,k})\text{-regular at all }\bx\in S_{q,k}\cap G \\
& \text{(ii) } \left|\hat{f}_{q, k}(\bx)-f^*_{ k}(\bx)\right| \leq \epsilon_q / 2 \text { for all } \bx \in S_{q,k} \text{ and } k\in\cK 
\end{aligned}
\right\}.
\end{align*}
The intersections of events are defined as $\overline{\mathcal{G}}_q=\bigcap_{j=1}^q \mathcal{G}_j $ and $\overline{\mathcal{M}}_q=\bigcap_{j=1}^q \mathcal{M}_j.$

Proposition \ref{prop:sqkregular} shows that, conditional on the preceding
epochs being well controlled and the current sample sizes being adequate, the
adaptively generated support of each arm remains weakly
$(\frac{c_0}{2^d},H_{q,k})$-regular. This is the support property needed to
apply the uniform local-polynomial error bound in the next step.

\begin{proposition}\label{prop:sqkregular}
Suppose Assumptions \ref{assump:iiddensity}--\ref{assump:Qxliplower} hold with $\rho=T^{-\frac{\beta}{2\beta+d}}$. For sufficiently large $T$ and all $1\leq q\leq Q-1$, on $\overline{\mathcal{G}}_{q-1}\cap\overline{\mathcal{M}}_q$, the set $S_{q,k}$ is weakly $(\frac{c_0}{2^d},H_{q,k})$-regular at any $\bx\in S_{q,k}\cap G$ for all $k\in\cK$.
\end{proposition}

Proposition \ref{prop:sqkregular} connects the reward geometry to the
adaptively generated sample supports. 
For arms with $\Delta_{k,\max}>\rho$,
Assumption~\ref{assump:Qxliplower} controls how reward gaps grow away from optimality
boundaries, while the epoch schedule determines the local radius $H_{q,k}$
through the random sample size $N_{q,k}$. The proof shows that these two
scales remain compatible across epochs, so regions in which an arm remains
active retain the local support regularity required by the local-polynomial
estimator.

\begin{proposition}\label{thm:tailproboftwoevents}
Suppose Assumptions \ref{assump:iiddensity}--\ref{assump:Qxliplower} hold with $\rho=T^{-\frac{\beta}{2\beta+d}}$. Then, for sufficiently large $T$ and all $1\leq q\leq Q-1$,
\begin{align}\label{eq:thirdeq}
    \mathbb{P}\left(\overline{\mathcal{G}}_q^C \cup \overline{\mathcal{M}}_q^C\right) \leq \frac{K\left(5+2 M_\beta^2\right) q}{T}.
\end{align} 
\end{proposition}

Proposition \ref{thm:tailproboftwoevents} shows that the sample-size,
support-regularity, and estimation-error conditions hold simultaneously
through the first $q$ epochs with failure probability of order $q/T$. This
uniform control is used below to establish fairness across epochs and to bound
the cumulative regret.

\subsubsection{Fairness Guarantee and Regret Analysis}

We now state the fairness and regret guarantees of
Algorithm \ref{alg:fairsmoothalg}. Propositions \ref{prop:sqkregular} and
\ref{thm:tailproboftwoevents} provide the uniform support and estimation
control used in both results.

\begin{theorem} \label{thm:smoothbanditfairness}Suppose Assumptions \ref{assump:iiddensity}--\ref{assump:Qxliplower} hold with $\rho=T^{-\frac{\beta}{2\beta+d}}$. For sufficiently large $T$, the policy defined by Algorithm \ref{alg:fairsmoothalg} satisfies $(1-\delta)$-fairness with $\delta = \tilde{O}(1/T)$.
\end{theorem}

\begin{theorem}\label{thm:minimax} 
Suppose Assumptions \ref{assump:iiddensity}--\ref{assump:Qxliplower} hold with $\rho=T^{-\frac{\beta}{2\beta+d}}$. The expected cumulative regret is upper bounded by $R_T=\tilde{O}(T^{\frac{\beta+d-\alpha \beta}{2\beta+d}})$. Moreover, when $\alpha\beta\leq d$, the policy defined by Algorithm \ref{alg:fairsmoothalg} is rate-optimal up to polylogarithmic factors in the minimax sense.
\end{theorem}

Together, Theorems \ref{thm:smoothbanditfairness} and \ref{thm:minimax}
show that Algorithm \ref{alg:fairsmoothalg} achieves
$(1-\delta)$-fairness with $\delta=\tilde O(1/T)$ and minimax-optimal
regret up to logarithmic factors when $\alpha\beta\leq d$. The rate
improves with smoother rewards and stronger margin separation and deteriorates
with the context dimension, reflecting the usual nonparametric estimation
trade-off.

\section{Manipulating Meritocracy: Vulnerability to Adversarial Attacks}
\label{sec:effectiveattack}

While the preceding section establishes fairness guarantees in the absence of corruption, platform environments may expose learning algorithms to strategic
reward manipulation. In advertising, recommendation, and marketplace
applications, participants can have incentives to distort ratings,
click-through rates, or engagement signals that influence future allocations.
For the fair algorithms studied here, fairness violations and regret
degradation need not occur together. This creates a
distinct attack channel that cannot be captured by regret alone. We study two
such failure modes: a covert attack that primarily affects fairness and a
catastrophic attack that affects both fairness and regret.

\subsection{Adversarial Model and Persistent Unfairness}

We first formalize the adversarial reward model and the notion of persistent
unfairness used to characterize the resulting allocation distortions. Suppose
an attacker has a total corruption budget $C$. In contrast to
\eqref{eq:yi}, at round $t$ the observed reward becomes
\begin{equation}\label{eq:yi_corrup}
    \tilde{y}_t
    =
    f_k^*(\bx_t)+c_t+\varepsilon_{k,t},
    \qquad \mbox{if }\pi_t=k,
\end{equation}
subject to
\begin{align}\label{eq_corrup_budget}
    \sum_{t=1}^T |c_t|\leq C.
\end{align}
Let $\cC_C$ denote the class of adaptive corruption strategies satisfying
\eqref{eq_corrup_budget} on any sample path. For
$c\in\cC_C$, let $R_T(\pi;\mathbb P,c)$ denote the expected cumulative
regret of policy $\pi$ on instance $\mathbb P$ under corruption strategy
$c$.

Throughout this section, the adversary is assumed to know the true reward
functions $f_k^*$ for all $k\in\cK$ and, before choosing $c_t$, to observe
the current context $\bx_t$ and selected arm $\pi_t$. This information
structure follows the oracle-attacker model commonly used in robust learning
\citep{bogunovic2021stochastic,bogunovic2022robust} and serves as a
worst-case robustness benchmark. The corruption budget $C$ measures the
amount of reward manipulation required to compromise the algorithm.

We use the following notion to distinguish persistent allocation distortions
from isolated fairness violations.

\begin{definition}[Persistent Unfairness]\label{def:unfairness}
An algorithm is said to be \emph{persistently unfair} if there exist
constants $c_1,c_2>0$ and a horizon $N=\widetilde O(1)$ such that, for any
round $t>N$, with probability at least $c_2$ there exist arms
$i,j\in\cK$ satisfying
\begin{align}\label{eq:defofstrongperunfair}
f_i^*(\bx_t)\geq f_j^*(\bx_t),
\qquad
\PP(\pi_t=i\mid\cF_{t-1}^+)
<
\PP(\pi_t=j\mid\cF_{t-1}^+)-c_1.
\end{align}
\end{definition}

The selection-probability gap $c_1$ rules out vanishing violations, while requiring such violations to occur with probability at least $c_2$ at all round $t>N$ distinguishes persistent imbalance from transient estimation error.
The definition permits violations involving
either tied arms or strictly ordered arms, consistent with the merit-based
fairness criterion in Definition~\ref{def:fairness}. 
With the constants in Definition~\ref{def:unfairness}, let $\cU_t$ denote the event that
\eqref{eq:defofstrongperunfair} holds for some pair $i,j\in\cK$ at round
$t$.

\subsection{Covert Fairness-Only Attacks}

We first consider attacks whose primary effect is on fairness violations
rather than regret. When two arms have equal merit on part of the context
space, manipulating early reward observations can induce a persistent
selection imbalance between them while adding little regret, because either
arm is reward-optimal on the targeted region. Such attacks are therefore
difficult to diagnose using regret alone, in contrast to attacks whose
effect is measured primarily through deterioration in reward performance
\citep{jun2018adversarial,zuo2024near}. 

Smooth reward functions naturally allow optimal-arm regions with nontrivial
overlap. We therefore focus on Algorithm~\ref{alg:fairsmoothalg} and consider
two arms $i$ and $j$ whose optimal regions $\cR_i$ and $\cR_j$ overlap on a
set with nonnegligible probability. The following theorem shows that an
adversary can exploit such overlap to create persistent unfairness with only
a polylogarithmic corruption budget,
while the attacked algorithm retains the same-order regret bound as in the
uncorrupted setting.

\begin{theorem}\label{thm:covert}
Consider a two-armed contextual bandit instance for which the overlap of the
optimal-arm regions contains a nonempty open ball: there exist
$\bx_0\in\cX$ and $r_0>0$ such that
$B(\bx_0,r_0)\subseteq\cR_i\cap\cR_j$. Then when $T$ is large enough, an
adversary with corruption budget $C=\tilde{O}(1)$ can make Algorithm
\ref{alg:fairsmoothalg} persistently unfair under Assumptions
\ref{assump:iiddensity}--\ref{assump:Qxliplower} with
$\rho=T^{-\frac{\beta}{2\beta+d}}$, and with probability at least
$1-\frac{K\left(5+2 M_\beta^2\right)+1}{T}$, the total occurrence of
unfairness satisfies $\sum_{t=1}^T\mathbb{I}(\cU_t)=\Omega(T).$
Moreover, the attacked algorithm retains the same-order regret bound $\widetilde O
\left(
T^{\frac{\beta+d-\alpha\beta}{2\beta+d}}
\right).$
\end{theorem}

The attack concentrates corruption on an overlapping optimal region, where the two arms have identical expected rewards. By perturbing early reward observations, the adversary induces a persistent selection imbalance between the tied arms. A key part of the analysis is to show that this localized manipulation does not propagate to the surrounding uncorrupted region: the local estimators there retain sufficient support and the same uniform accuracy as in the uncorrupted analysis. Since both arms are optimal on the targeted region, the resulting fairness violations incur no regret there, while the regret outside the attacked region remains of the original order. Theorem~\ref{thm:covert} therefore shows that substantial distortions in relative treatment can remain hidden from regret.

\subsection{Catastrophic Dual-Failure Attacks}

We next consider attacks that target both fairness and learning performance.
In contrast to the covert construction, the adversary uses corrupted feedback
to induce persistent preferential errors on regions with nonzero reward gaps,
so that the resulting allocation distortion also accumulates regret. The
following theorem shows that a polylogarithmic corruption budget can be
sufficient to induce both persistent unfairness and linear regret.

\begin{theorem}\label{thm:constantbudgetunfairness}
When $T$ is sufficiently large, an adversary with budget
$C=\tilde{O}(1)$ can achieve:
\begin{itemize}
    \item Under Assumptions
    \ref{assum_parabound}--\ref{assum_excite},
    Algorithm~\ref{alg:fairols} is persistently unfair, and its cumulative
    regret satisfies $R_T=\Omega(T).$
    With probability at least
    $1-\frac{4K+1}{T^4}$, the total occurrence of unfairness satisfies $\sum_{t=1}^T\mathbb{I}(\cU_t)=\Omega(T).$
    \item Under Assumptions
    \ref{assump:iiddensity}--\ref{assump:Qxliplower} with
    $\rho=T^{-\frac{\beta}{2\beta+d}}$,
    Algorithm~\ref{alg:fairsmoothalg} is persistently unfair, and its
    cumulative regret satisfies $ R_T=\Omega(T)$
    whenever $    \max_{k\in\cK}\max_{\bx\in\cX}
    \left(
    f_k^*(\bx)-\max_{j\neq k}f_j^*(\bx)
    \right)
    >c^{\prime\prime}$
    for some positive constant $c^{\prime\prime}$. With probability at least
    $1-\frac{K\left(5+2M_\beta^2\right)+1}{T}$, the total occurrence of
    unfairness satisfies $\sum_{t=1}^T\mathbb{I}(\cU_t)=\Omega(T).$
\end{itemize}
\end{theorem}

The constant $c^{\prime\prime}>0$ ensures that the smooth instance contains
a region with a nonvanishing reward advantage, so that persistent
misallocation there incurs nonnegligible regret.
Theorem~\ref{thm:constantbudgetunfairness}\footnote{Theorems
\ref{thm:covert} and \ref{thm:constantbudgetunfairness} also apply to
Algorithm~1 of \citet{hu2022smooth} when $K=2$.}
shows that a polylogarithmic corruption budget can induce both persistent
unfairness and linear regret in the fair algorithms studied here.

Together, Theorems~\ref{thm:covert} and \ref{thm:constantbudgetunfairness} demonstrate that a small corruption budget can create persistent unfairness in fair learning algorithms, with or without substantial regret degradation. In particular, the covert construction establishes the existence of a fairness-specific vulnerability that need not be reflected in regret. This motivates the robust algorithms developed next.

\section{Robust Fairness for Contextual Bandits}\label{sec:fairrobustalg}
In this section, we develop robust fair bandit algorithms that preserve $(1-\widetilde O(1/T))$-fairness under adversarial reward corruption. We establish matching upper and lower regret bounds in the linear setting. For smooth rewards, the bounds match up to logarithmic factors when $\alpha\beta\leq d$, while the corruption-dependent term is minimax optimal for all regimes.

\subsection{From Base to Robust Algorithms}

We adopt the corruption model specified in Section \ref{sec:effectiveattack} and restrict attention to $C\leq T$, which covers the nontrivial regime for cumulative-regret analysis. Following \citet{bogunovic2021stochastic,he2022nearly}, we assume that the corruption budget $C$ is known to the learner. If $C$ is unknown, a natural extension is to replace it with a time-dependent upper estimate of the corruption level. Establishing fairness and regret guarantees for such an adaptive variant is left for future work.

Adversarial corruption enlarges the uncertainty in the reward comparisons used to justify unequal arm selection. Robustness must therefore address both reward estimation and the certification of preferential treatment. Our algorithms account for this additional uncertainty through corruption-aware sampling, estimation, and decision thresholds. Algorithm \ref{alg:robustfairols} treats linear rewards, while Algorithm \ref{alg:robustfairsmoothalg} treats smooth rewards.

The robust fair OLS algorithm modifies Algorithm \ref{alg:fairols} in two places. It extends the initial exploration by an additive \(O(C)\) term and adds an \(O(C/t)\) correction to the estimation threshold. Intuitively, the total amount of corruption is bounded by $C$, while the information available for reward estimation increases over time, so its effect on the estimator is progressively diluted at rate $C/t$. All other components remain unchanged.

\begin{algorithm}[!h]
\caption{Robust Fair OLS Bandit}\label{alg:robustfairols}
\label{alg:robust-linear}
\begin{algorithmic}[1]
\REQUIRE Parameters $C_a, C_b, h, T$, corruption budget $C$.
\STATE $\gamma_{\text{lin}} \gets \dfrac{64K^2dr^2}{h\lambda^*\tilde{p}}$ . 
\STATE $|\mathcal{T}_0^{RL}| \gets \lceil C_a\log T + \gamma_{\text{lin}} \cdot C\rceil$.
\FOR{$t \in\mathcal{T}_0^{RL}$}
    \STATE Random exploration (as in Steps 4-5 of Algorithm \ref{alg:fairols}).
\ENDFOR
\STATE Compute initial estimation $\hat{\beta}_{k,0}^{RL} \gets \hat{\beta}^{RL}(\mathcal{I}_{k,0})$ as in \eqref{eq:robustbetahat}.
\FOR{$t \in \{ |\mathcal{T}_0^{RL}| + 1,...,T\}$}
    \STATE Observe context $\boldsymbol{x}_t$ and compute $\hat{\mathcal{K}}^{RL}_{\boldsymbol{x}_t}$ (using $h/2$-chaining similar to Step 9 of Algorithm \ref{alg:fairols}).
    \IF{$\hat{\cK}^{RL}_{\bx_t}=\{k\}$}
    \STATE Pull arm $\pi_t=k$.
    \ELSE
        \STATE Let $\kappa \gets \dfrac{192K^2dr^2}{\lambda^*\tilde{p}}\vee \dfrac{32dr^2}{\lambda^*\tilde{p}^2}$. 
        \STATE Let $\epsilon_t^{RL} \gets C_b \sqrt{\dfrac{\log T}{t}} + \kappa \dfrac{C}{t}$.
        \STATE Compute the all-sample estimation $\hat{\beta}^{RL}_{k,t-1}=\hat{\beta}^{RL}(\mathcal{I}_{k,t-1})$ as in \eqref{eq:robustbetahat}.
        \STATE Compute $\mathcal{K}_c^{RL}(\boldsymbol{x}_t)$ via $\epsilon_t^{RL}$-chaining (similar to Step 14 of Algorithm \ref{alg:fairols}).
        \STATE Randomly pull arm $\pi_t \in \cK_c^{RL}(\bx_t)$ with equal probability, and receive reward $\tilde{y}_t$.
    \ENDIF
    \STATE Update index sets $\mathcal{I}_{\pi_t,t} \gets \mathcal{I}_{\pi_t,t-1} \cup \{t\}$; $\mathcal{I}_{k,t} \leftarrow\mathcal{I}_{k,t-1}$ for $k\in\mathcal{K}\setminus\{\pi_t\}$.
\ENDFOR
\end{algorithmic}
\end{algorithm}

Compared with Algorithm \ref{alg:fairsmoothalg}, the robust smooth algorithm makes three coordinated modifications: it enlarges epoch lengths to dilute the effect of a fixed corruption budget, uses a corruption-aware bandwidth, and inflates the comparison threshold used to certify preferential treatment. The bandwidth accounts for the additional local estimation bias induced by corruption, leading to the corruption-dependent scale in Algorithm~\ref{alg:robustfairsmoothalg}. This enlarged scale also changes the neighborhoods over which reliable estimation must be maintained, so the support-regularity argument from the uncorrupted setting does not apply directly; Proposition~\ref{prop:robustsqkregular} re-establishes the required regularity at the corruption-aware scale.

For the robust local-polynomial analysis, define
\[
B_{\mathrm{LP}}
:=
1\vee
\frac{\sqrt{M_\beta}}{\lambda_0}
\left(1+\frac{2Lv_dp_{\max}}{p^*}\right),
\]
where $\lambda_0$ is the local-polynomial constant specified in Lemma \ref{lem:A.6} in the supplement.

\begin{algorithm}[!h]
\caption{Robust Fair Smooth Bandit}
\label{alg:robustfairsmoothalg}
\begin{algorithmic}[1]
\REQUIRE Corruption budget $C$, grid representatives $G$, hypercubes $\mathcal C$, H\"older smoothness $\beta$, boundary-growth exponent $\beta'$, lower mass bound $p^*$, regularity constant $c_0$, constant $C_K$, context dimension $d$, and horizon $T$.
\STATE Initialize $\cK_{1,j}^{RS}=\cK$ for each $j\in\{1,\ldots,|\cC|\}$.
    \STATE For each $q\geq1$, set
    $\gamma_q^{RS} \gets (1\vee C)^{\frac{2\beta}{2\beta+d}}\vee 4^q$.
    \STATE For each $q\geq1$, compute the scheduled length\\
    $\displaystyle
    \begin{aligned}
    \bar n_q^{RS} \gets \Bigg\lceil &\frac{2K}{p^*}
    \left(\frac{\gamma_q^{RS}\log(T\delta_A^{-d})}{C_K}\right)^{\frac{2\beta+d}{2\beta}}
    (\log T)^{\frac{2\beta+d}{\beta'-1}-\frac{2\beta+d}{2\beta}}+\frac{K^2}{2p^{*2}}\log T\Bigg\rceil .
    \end{aligned}$
    \STATE Set
    $\displaystyle Q^{RS}\gets\min\left\{q\geq1:
    \sum_{\ell=1}^{q}\bar n_\ell^{RS}\geq T\right\}$.
\FOR{$q = 1$ \TO $Q^{RS}$}
    \STATE Set the realized epoch length as
    $\displaystyle |\cT_q^{RS}|\gets
    \begin{cases}
    \bar n_q^{RS},&q<Q^{RS},\\
    T-\sum_{\ell=1}^{q-1}\bar n_\ell^{RS},&q=Q^{RS}.
    \end{cases}$
    \STATE $\displaystyle
    \epsilon_q^{RS}\gets
    \left((\gamma_q^{RS})^{-1/2}
    (\log T)^{\frac{\beta'-1-2\beta}{2\beta'-2}}
    \vee T^{-\frac{\beta}{2\beta+d}}\right)
    +2B_{\mathrm{LP}}
    \left(\frac{4KC}{p^*|\cT_q^{RS}|}\right)^{\frac{\beta}{\beta+d}}$.
    \IF{$q = 1$}
        \STATE Random exploration (as in Lines 6--7 of Algorithm \ref{alg:fairsmoothalg}).
    \ELSE
    \STATE Set $N_{q-1,k}^{RS}=\left|\mathcal{T}_{q-1,k}^{RS}\right|$ and, for $k\in\cK$,
    $\displaystyle H_{q-1,k}^{RS}=(N_{q-1,k}^{RS})^{-\frac{1}{2\beta+d}}
    \vee\left(\frac{C}{N_{q-1,k}^{RS}}\right)^{\frac{1}{\beta+d}}$.
\STATE Compute index of exploration hypercubes $\mathcal{I}_q=\{j:j\in \{1, \ldots,|\mathcal{C}|\},|\cK^{RS}_{q-1,j}|>1\}$.
\STATE For each $j\in \mathcal{I}_q$ and $k\in \cK^{RS}_{q-1,j}$, construct the local polynomial estimate for $\bx\in \text{Cube}_j$
as in \eqref{eq:fhatRSqofx}:
\begin{align*}
    \hat{f}_{q-1,k}^{RS}(\bx)=\hat{f}^{\mathrm{LP}}\left(g(\bx) ; \mathcal{T}^{RS}_{q-1,k}, H^{RS}_{q-1,k}, \mathfrak{b}(\beta)\right).
\end{align*}
\vspace{-7mm}
\STATE Update active arm sets $\mathcal{K}_{q,j}^{RS}$ via $2\epsilon_{q-1}^{RS}$-chaining (similar to Step 13 of Algorithm \ref{alg:fairsmoothalg}).

\FOR{$t\in\mathcal{T}_q^{RS}$}
\STATE Observe context $\bx_t$, pull arm $\pi_t \in \cK^{RS}_{q,u(\bx_t)}$ with equal probability, and receive reward $\tilde{y}_t$. 
\ENDFOR
\ENDIF
\STATE Log the samples $\mathcal{T}_{q,k}^{RS}=\left\{\left(\bx_t,\tilde{y}_t\right): t \in \mathcal{T}_q^{RS}, \mathcal{\pi}_t=k\right\}$ for $k \in \mathcal{K}$.
\ENDFOR
\end{algorithmic}
\end{algorithm}

\subsection{Fairness Guarantee and Regret Analysis}

We first establish uniform error bounds for the corrupted estimators and then derive the corresponding fairness and regret guarantees. We subsequently compare these upper bounds with minimax lower bounds for the same problem classes.

\subsubsection{Linear Reward Function} 

We begin with the linear contextual bandit setting, where the reward model follows \eqref{eq:linearreward}. Under adversarial corruption, the observed reward at time $t$ becomes
\begin{align}\label{eq_linear_corrupt}
    \tilde{y}_t=\bz_t^\mathrm{T}\beta_k + c_t +\varepsilon_{k,t}, \quad \mbox{if }\pi_t = k.
\end{align}
The corresponding robust OLS estimator, trained on a corrupted sample set $\mathcal{J}$, is given by
\begin{align}\label{eq:robustbetahat}
\hat{\beta}^{RL}(\mathcal{J}) = \argmin_{\beta\in\RR^d} \frac{1}{|\cJ|} \|\tilde{Y} - Z\beta\|_2^2 = (Z^\mathrm{T} Z)^{-1} Z^\mathrm{T} \tilde{Y},
\end{align}
where $\tilde{Y} = \{\tilde{y}_s : s \in \mathcal{J}\} \in \mathbb{R}^{|\mathcal{J}|}$ is the vector of corrupted rewards. The following two propositions provide uniform error bounds for the robust OLS estimator under adversarial corruption.

\begin{proposition}\label{prop:robustrandomtailbound} 
Suppose Assumptions \ref{assum_parabound}-\ref{assum_excite} hold,
$T>2d$, and $C_a\geq C_{a,0}$. Then the following tail inequality holds:
\begin{align*}
    \PP\left(\max_{\bx\in \cX}\max_{k\in\mathcal{K}}|(\hat{\beta}^{RL}_k(\mathcal{I}_{k,0})-\beta_k)^{\mathrm{T}}\bz|\geq \frac{h}{4}\right) \leq 7KT^{-4}.
\end{align*}
\end{proposition}

\begin{proposition}\label{prop:robustallsampletailbound}
Suppose Assumptions \ref{assum_parabound}-\ref{assum_excite} hold,
$T>2d$, $C_a\geq C_{a,0}$, and $C_b\geq C_{b,0}$. Then the following
tail inequality holds for all $t>|\cT_0^{RL}|$:
\begin{align*}
    \PP\left(\max_{\bx\in \cX}\max_{k\in \cK}|(\hat{\beta}^{RL}_k(\mathcal{I}_{k,t})-\beta_k)^{\mathrm{T}}\bz|\geq \frac{\epsilon_{t+1}^{RL}}{2}\right)
    \leq 15KT^{-4}.
\end{align*}
\end{proposition}

Proposition \ref{prop:robustrandomtailbound} controls the initial estimator after the extended exploration phase, while Proposition \ref{prop:robustallsampletailbound} controls the corrupted all-sample estimator. Together, these bounds ensure that unequal selection probabilities are supported by a true reward difference. The following theorem gives the resulting fairness guarantee.

\begin{theorem}\label{thm:robustlinearfairness}
    Suppose Assumptions \ref{assum_parabound}-\ref{assum_excite} hold, $C_a\geq C_{a,0}$, $C_b\geq C_{b,0}$, and $T>2d+\sqrt{22K}$. Then the policy defined by Algorithm \ref{alg:robustfairols} satisfies $(1-\frac{1}{T})$-fairness.
\end{theorem}

Having established fairness under corruption, we next derive the cumulative regret upper bound for Algorithm \ref{alg:robustfairols}.

\begin{theorem}[Upper Bound for Robust Linear Contextual Bandits]\label{thm:robustminimaxoptimallinear}
Assume that the conditions in Theorem \ref{thm:robustlinearfairness} are satisfied. Then the cumulative regret of Algorithm \ref{alg:robustfairols} satisfies
\[
R_T=O(\log^2 T+C).
\]
\end{theorem}

Theorem \ref{thm:robustminimaxoptimallinear} gives an additive $O(C)$ corruption penalty. The lower bound below shows that both the $\log^2T$ term and the linear dependence on $C$ are unavoidable, establishing minimax optimality of Algorithm \ref{alg:robustfairols}.

\begin{theorem}[Lower Bound for Linear Contextual Bandits under Corruption]\label{thm:lowerboundlinear}
  Let $\Pi_C^{\mathrm{fair}}$ denote all admissible policies that satisfy $(1-\delta)$-fairness, with $\delta=\Theta(1/T)$, under any corruption strategy in $\cC_C$. For the problem class $\mathcal{P}$ satisfying Assumptions \ref{assum_parabound}-\ref{assum_excite}, we have
\begin{align*}
     \inf_{\pi\in\Pi_C^{\mathrm{fair}}}
     \sup_{\mathbb{P} \in \mathcal{P}}
     \sup_{c\in\cC_C}R_T(\pi;\mathbb P,c)
     =\Omega \left( \log^2 T + C\right).
\end{align*} 
\end{theorem}

\subsubsection{Smooth Reward Function}
For the smooth contextual bandit problem under adversarial corruption, the observed reward is given by \eqref{eq:yi_corrup}, with $c_t$ satisfying \eqref{eq_corrup_budget}. We use the corrupted local polynomial estimator
\begin{align}\label{eq:fhatRSqofx}
  \hat{f}^{RS}_{q-1,k}(\bx)=\hat{f}^{\mathrm{LP}}\left(g(\bx) ; \mathcal{T}_{q-1,k}^{RS}, H_{q-1,k}^{RS}, \mathfrak{b}(\beta)\right),  
\end{align}
where the corrupted samples collected at epoch $q$ for arm $k$ are given by $\mathcal{T}_{q,k}^{RS} = \left\{ \left( \bx_t, \tilde{y}_t \right) : t \in \mathcal{T}_q^{RS}, \, \pi_t = k \right\}$. 
For each epoch $q$, define the sample-size and estimation events
\begin{align*}
&\mathcal{M}^{RS}_q  =\left\{\min _{k \in \mathcal{K}} N_{q, k}^{RS} \geq\frac{p^*}{2K}|\cT_q^{RS}|\right\}, \\
&\mathcal{G}^{RS}_q = \left\{
\begin{aligned}
& \text{(i) } \text { for all }k\in\cK,  S_{q,k}^{RS} \text{ is weakly }  (\frac{c_0}{2^d},H_{q,k}^{RS})\text{-regular at all }\bx\in S_{q,k}^{RS}\cap G \\
& \text{(ii) } \left|\hat{f}_{q, k}^{RS}(\bx)-f^*_{ k}(\bx)\right| \leq \epsilon_q^{RS} / 2 \text { for all } \bx \in S_{q,k}^{RS} \text{ and } k\in\cK 
\end{aligned}
\right\},
\end{align*}
with $\overline{\mathcal{G}}_q^{RS}=\bigcap_{j=1}^q \mathcal{G}_j^{RS}$ and $\overline{\mathcal{M}}_q^{RS}=\bigcap_{j=1}^q \mathcal{M}_j^{RS}$. The event $\mathcal M_q^{RS}$ controls the realized sample sizes, while $\mathcal G_q^{RS}$ records support regularity and the uniform estimation bound used by the chaining rule.

\begin{proposition}\label{prop:robustsqkregular}
Suppose Assumptions \ref{assump:iiddensity}--\ref{assump:Qxliplower} hold with $\rho=T^{-\frac{\beta}{2\beta+d}}+\left(\frac{C}{T}\right)^{\frac{\beta}{\beta+d}}$. If $Q^{RS}>1$, then for sufficiently large $T$ and every $q=1,\ldots,Q^{RS}-1$, on $\overline{\mathcal{G}}_{q-1}^{RS}\cap\overline{\mathcal{M}}_q^{RS}$, the set $S_{q,k}^{RS}$ is weakly $(c_0/2^d,H_{q,k}^{RS})$-regular at every $\bx\in S_{q,k}^{RS}\cap G$ for all $k\in\cK$.
\end{proposition}

\begin{proposition}\label{thm:robusttailproboftwoevents}
Suppose Assumptions \ref{assump:iiddensity}--\ref{assump:Qxliplower} hold with $\rho=T^{-\frac{\beta}{2\beta+d}}+\left(\frac{C}{T}\right)^{\frac{\beta}{\beta+d}}$, $Q^{RS}>1$, and $T$ is sufficiently large. Then, for every $q=1,\ldots,Q^{RS}-1$,
\begin{align}\label{eq:thirdeqrobust}
    \mathbb{P}\left((\overline{\mathcal{G}}_q^{RS})^C \cup (\overline{\mathcal{M}}_q^{RS})^C\right) \leq \frac{K\left(5+2 M_\beta^2\right) q}{T}.
\end{align} 
\end{proposition}

Proposition \ref{prop:robustsqkregular} shows that the adaptively generated active supports remain weakly regular at the corruption-aware bandwidth. Proposition \ref{thm:robusttailproboftwoevents} combines this support property with the robust local-polynomial error bound and epoch-wise sample-size concentration to obtain uniform control across epochs.

These propositions yield the fairness and regret guarantees below. If $Q^{RS}=1$, the policy remains uniform throughout the horizon and is therefore fair.

\begin{theorem} \label{thm:robustsmoothbanditfairness}
   Suppose Assumptions \ref{assump:iiddensity}--\ref{assump:Qxliplower} hold with $\rho=T^{-\frac{\beta}{2\beta+d}}+\left(\frac{C}{T}\right)^{\frac{\beta}{\beta+d}}$. For sufficiently large $T$, the policy defined by Algorithm \ref{alg:robustfairsmoothalg} satisfies $(1-\delta)$-fairness with $\delta = \tilde{O}(1/T)$.
\end{theorem}

\begin{theorem}[Upper Bound for Robust Smooth Contextual Bandits]\label{thm:robustupperbound} 
    Suppose Assumptions \ref{assump:iiddensity}--\ref{assump:Qxliplower} hold with $\rho=T^{-\frac{\beta}{2\beta+d}}+\left(\frac{C}{T}\right)^{\frac{\beta}{\beta+d}}$. The expected cumulative regret satisfies
    \begin{align*}
        R_T=\widetilde O\left(
        T^{\frac{\beta+d-\alpha\beta}{2\beta+d}}+
        \begin{cases}
        C^{\frac{\beta(1+\alpha)}{\beta+d}}T^{\frac{d-\alpha\beta}{\beta+d}},&\alpha\beta<d,\\
        C,&\alpha\beta\geq d.
        \end{cases}\right).
    \end{align*}
\end{theorem}

Theorem \ref{thm:robustsmoothbanditfairness} establishes uniform fairness with probability at least \(1-\widetilde O(1/T)\), while Theorem \ref{thm:robustupperbound} separates the baseline estimation cost from the additional regret induced by corruption. The resulting phase transition reflects the interaction between corruption-induced local estimation error and the margin structure. When \(\alpha\beta<d\), corruption continues to affect increasingly refined decisions over the horizon, leading to the coupled \(C\)--\(T\) term. When \(\alpha\beta\ge d\), the margin structure sufficiently limits the contribution of these difficult regions, and the additional cost reduces to order \(C\). Setting \(C=0\) recovers the uncorrupted rate.

The following result gives minimax lower bounds for fair smooth contextual bandits. 

\begin{theorem}[Lower Bound for Smooth Contextual Bandits under Corruption]\label{thm:robustlowerbound} 
Fix positive parameters $\alpha$ and $\beta$ and suppose $1\leq C\leq T$. Let $\Pi_C^{\mathrm{fair}}$ denote all admissible policies that satisfy $(1-\widetilde O(1/T))$-fairness under any corruption strategy in $\cC_C$. For the class $\mathcal{P}$ of problem instances satisfying Assumptions \ref{assump:iiddensity}--\ref{assump:Qxliplower} with $\rho=T^{-\frac{\beta}{2\beta+d}}+\left(\frac{C}{T}\right)^{\frac{\beta}{\beta+d}}$,
\begin{align*}
     \inf_{\pi\in\Pi_C^{\mathrm{fair}}}
     \sup_{\mathbb{P} \in \mathcal{P}}
     \sup_{c\in\cC_C}R_T(\pi;\mathbb P,c)
     =\Omega\left(
     \begin{cases}
     T^{\frac{\beta+d-\alpha\beta}{2\beta+d}}
     +C^{\frac{\beta(1+\alpha)}{\beta+d}}T^{\frac{d-\alpha\beta}{\beta+d}},
     &\alpha\beta\leq d,\\
     C,&\alpha\beta>d
     \end{cases}\right).
\end{align*}
\end{theorem}

Together with Theorem~\ref{thm:robustupperbound}, the lower bound in Theorem \ref{thm:robustlowerbound} establishes minimax optimality up to logarithmic factors for the full regret rate when $\alpha\beta\le d$, and only for the corruption-dependent term when $\alpha\beta>d$. The corruption penalty exhibits a phase transition at $\alpha\beta=d$.
When $\alpha\beta<d$, the effect of corruption interacts with the probability
mass of contexts near the decision boundary, producing the coupled $C$--$T$
term. At and above the boundary, i.e., $\alpha\beta\ge d$, the corruption-dependent penalty reduces to
order $C$. Combined with the matching lower bounds, this characterizes the
sharp corruption dependence and shows how smoothness and margin structure
jointly determine the cost of robustness.

\section{Numerical Experiments}\label{sec_numerical}

We use numerical experiments to assess the fairness--performance trade-off and robustness to reward corruption. We first compare the proposed fair algorithms with standard contextual-bandit benchmarks under stochastic rewards, then evaluate their robust variants under controlled corruption, and finally consider a semi-synthetic wine brokerage application with real covariates. Detailed configurations and additional results are reported in the supplement.

\subsection{Verifying Fairness and Regret in the Stochastic Setting}\label{subsec:stochastic}

We first examine whether imposing uniform fairness  changes finite-horizon learning performance. We report cumulative regret together with the cumulative number of unfair decisions, where a round is counted as unfair if, at the realized context, some arm receives a higher selection probability than another arm with an equal or higher true reward.

For the linear setting, we compare the proposed Fair OLS algorithm with (1) a standard greedy algorithm \citep{bastani2021mostly}, (2) a UCB-style algorithm \citep{abbasi2011improved}, (3) the OLS bandit algorithm \citep{goldenshluger2013linear}, and (4) a random policy. The random policy provides a useful fairness benchmark because it treats all arms symmetrically but does not exploit reward information.

We consider a linear contextual bandit problem with $K=10$ arms and context dimension $d=10$. Contexts are drawn uniformly from $\mathcal{X} = [-1,1]^d$ at each round. The reward function is arm-specific and linear in the observed context $y_k(\bx) = f^*_k(\bx)+\varepsilon =w_k^\mathrm{T} \bx + b_k + \varepsilon,$
where $\varepsilon \sim \mathcal{N}(0, \sigma^2)$ with $\sigma=0.05$. The weight vectors $w_k \in \mathbb{R}^d$ are chosen to generate context-dependent relative advantages across arms. Specifically, the weight matrix \( W \in \mathbb{R}^{K \times d} \) is defined cyclically: $ W_{k,j} = 2 \cdot \mathbb{I}(j \equiv k) + 1 \cdot \mathbb{I}(j \equiv k+1) - 1 \cdot \mathbb{I}(j \equiv k-1),$ where the equivalence \( \equiv \) is taken modulo \( d \). The resulting construction gives different arms advantages in different regions of the context space, while the bias terms $b_k = 0.5\sin(2\pi k/K)$ introduce additional heterogeneity across arms.

\begin{figure}[ht]
\FIGURE{
\begin{tabular}{cc}
    \subfigure[Linear Setting]{\includegraphics[width=0.72\textwidth]{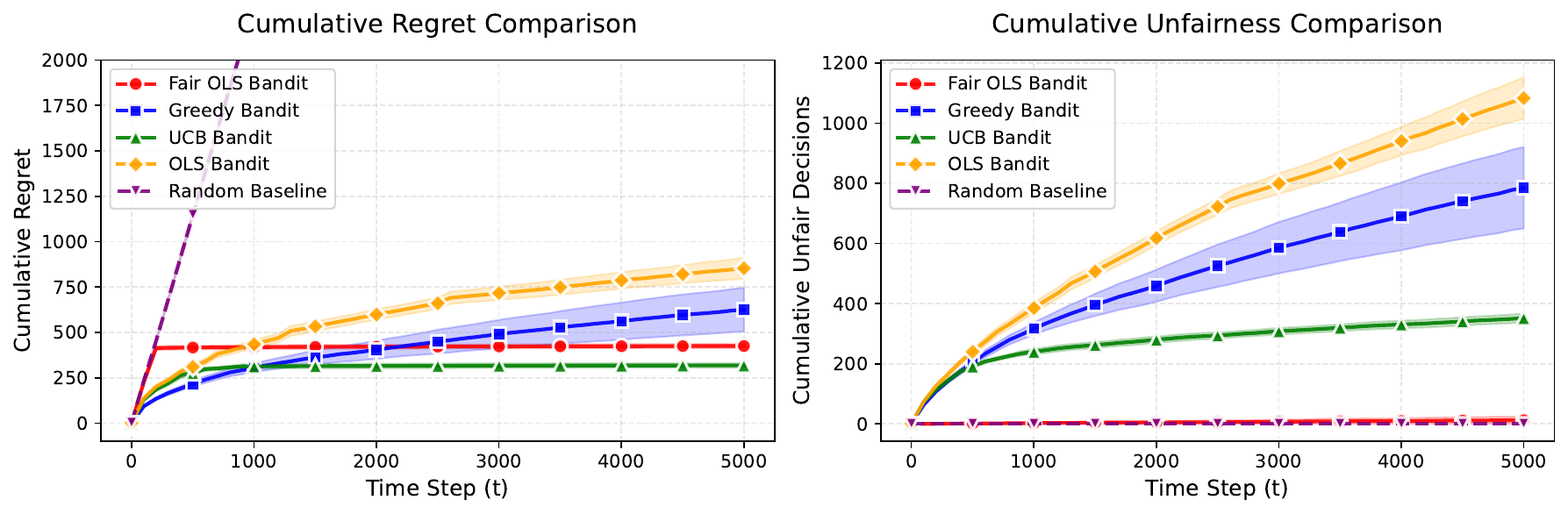}}\\
\subfigure[Smooth Setting]{\includegraphics[width=0.72\textwidth]{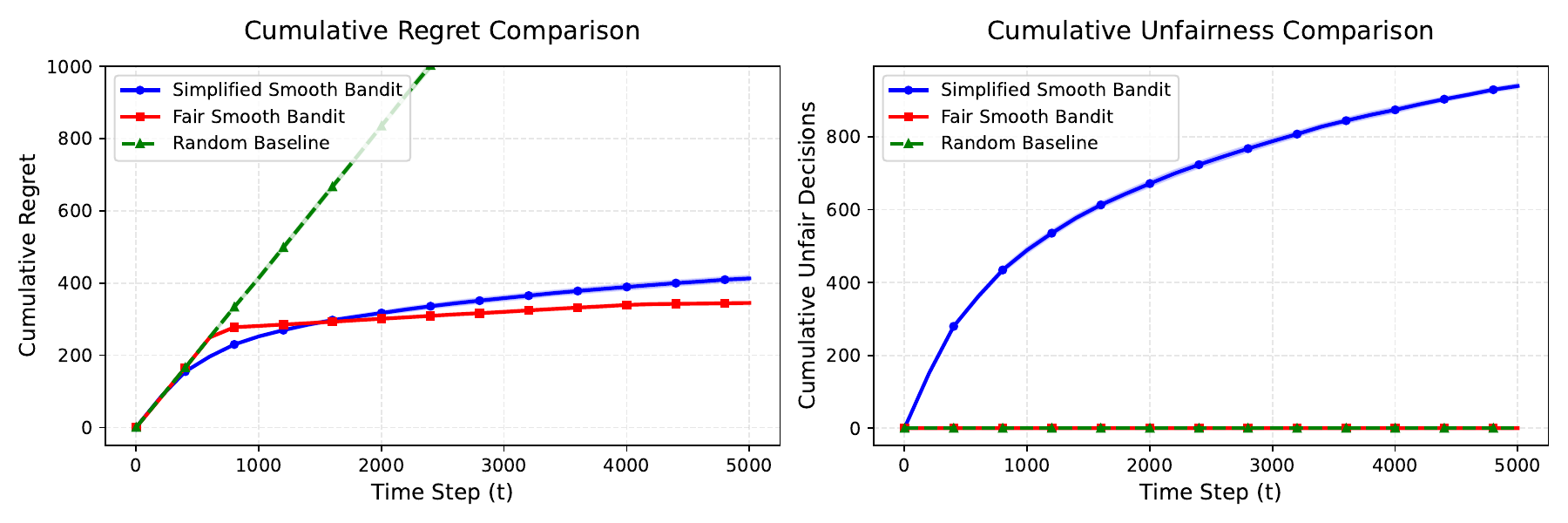}}
\end{tabular}
}
{Performance comparison of linear and smooth contextual bandit algorithms in the stochastic setting. Left: Cumulative regret over time. Right: Cumulative unfair decisions over time.  \label{fig:linear-algorithms-comparison} Lines show mean values from 10 independent runs, with shaded areas representing 95\% confidence intervals.} 
{
}
\end{figure}

We conduct 10 independent runs with a time horizon of \( T = 5000 \). Figure \ref{fig:linear-algorithms-comparison} (a) reports the mean trajectories and 95\% confidence intervals. Over this horizon, Fair OLS has regret of the same overall scale as the learning-based benchmarks while producing substantially fewer unfair decisions. The comparison therefore illustrates that the fairness constraint can greatly improve allocation behavior without an order-of-magnitude loss in learning performance in this instance.

We next consider the non-parametric setting. We compare our Fair Smooth Bandit with the smooth bandit baseline of \cite{hu2022smooth}. Following the numerical implementation in \cite{hu2022smooth}, we use its simplified version for empirical comparison rather than the full inestimable-region detection procedure, which is computationally demanding. We also include a random policy as a reference.

We consider a contextual bandit problem with \( K = 4 \) arms and context dimension \( d = 2 \). Contexts are sampled uniformly from the space \( \mathcal{X} = [-1, 1]^d \) in each round. The true reward function for arm \( k \) given context \( \bm{x} \in \mathcal{X} \) is defined as $y_k(\bm{x}) = f_k^*(\bm{x}) + \varepsilon =  \exp\left( -\| \bx - \bm{\mu}_k \|_2^2 \right) + \varepsilon,$
where \( \varepsilon \sim \mathcal{N}(0, \sigma^2) \) with \( \sigma = 0.05 \), and the set of arm centers is $\bm{\mu}_k \in \{(0.5, 0.5), (-0.5, 0.5), (-0.5, -0.5), (0.5, -0.5)\}$ for $k=0, 1, 2, 3$. This construction provides a smooth nonlinear reward landscape with context-dependent arm rankings. We set the smoothness parameter to \( \beta = 5 \). 

We again perform 10 independent runs with a time horizon of \( T = 5000 \). Figure \ref{fig:linear-algorithms-comparison} (b) shows that the Fair Smooth Bandit remains competitive in cumulative regret while sharply reducing the number of unfair decisions relative to the smooth bandit baseline. For finite-horizon implementation, we use the calibrated epoch schedule reported in the supplement. This calibration preserves the dependence of the schedule on the epoch index, horizon, and smoothness parameters while simplifying theory-specific constants and logarithmic factors. The same implementation is used in the subsequent smooth-bandit experiments.

\subsection{Robustness Under Adversarial Reward Perturbations}\label{subsec:adversarial}

We next introduce adversarial reward corruption into the same synthetic instances. In the linear experiment, five of the ten arms are designated as vulnerable to attack. The adversary has total budget $C = 200$ and, whenever a vulnerable arm is attacked, attempts to replace its expected reward by the misleading value $-4$, subject to the remaining budget. We conduct 10 independent runs with horizon $T=10,000$. Figure \ref{fig:linear-adversarial-comparison} (a) summarizes the resulting regret and unfair-decision trajectories.

\begin{figure}[ht]
\FIGURE{
\begin{tabular}{cc}
    \subfigure[Linear Setting]{\includegraphics[width=0.72\textwidth]{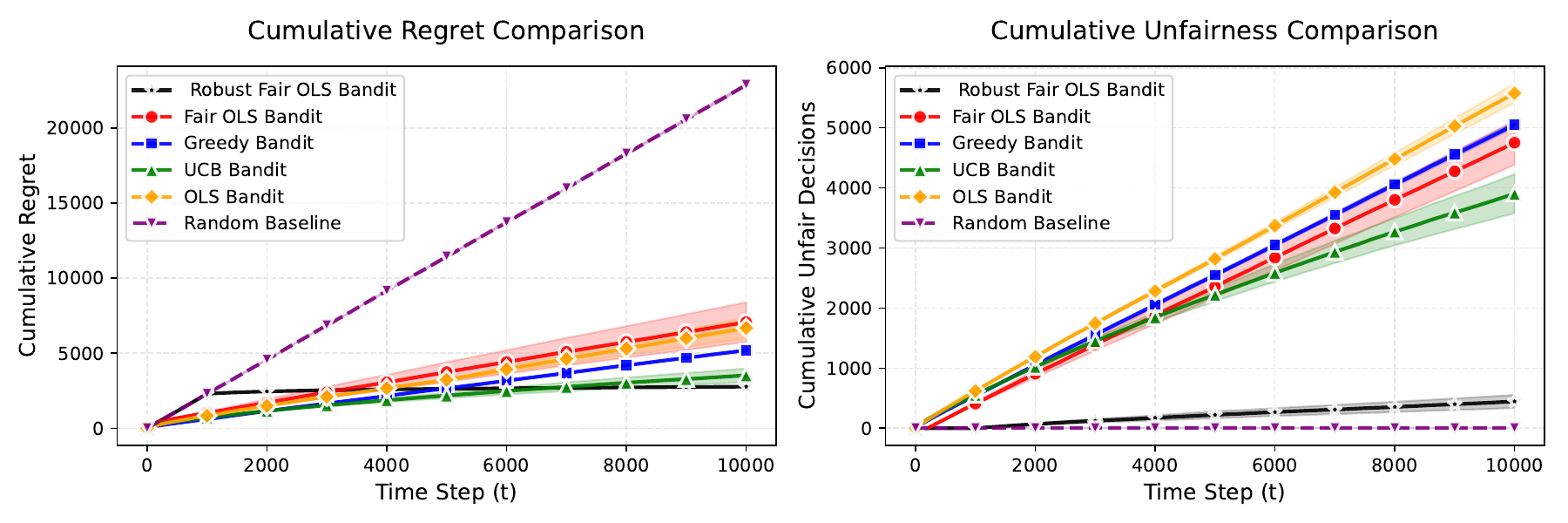}}\\
\subfigure[Smooth Setting]{\includegraphics[width=0.72\textwidth]{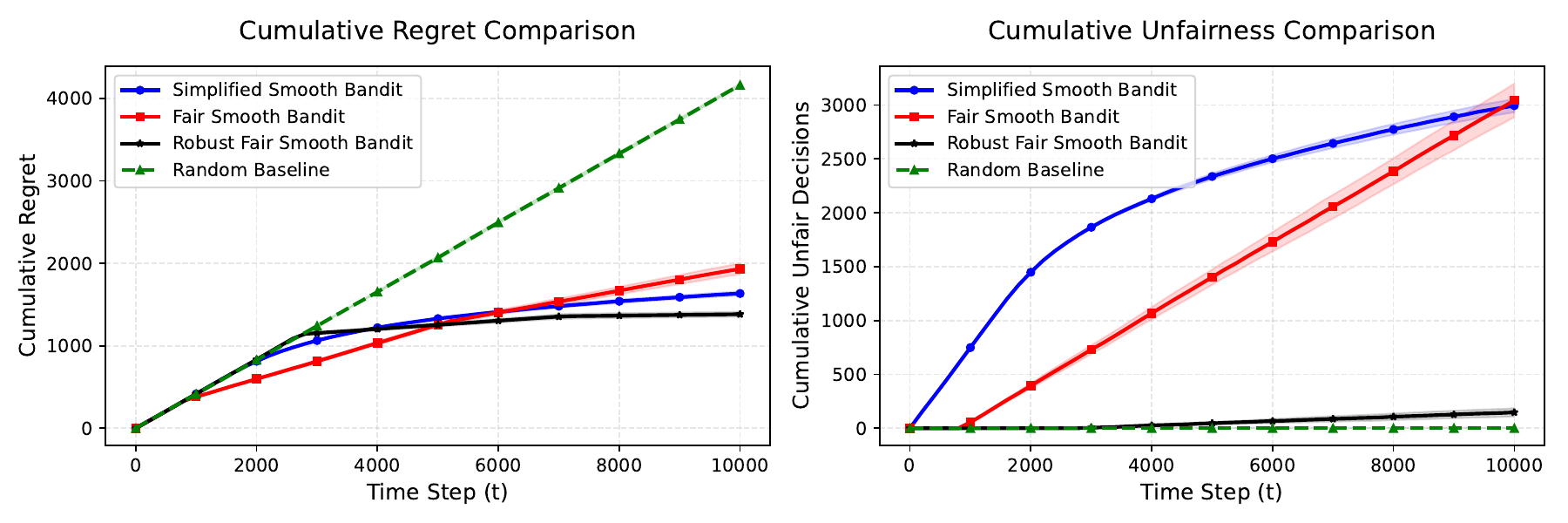}}
\end{tabular}
}
{Performance comparison of linear and smooth contextual bandit algorithms in the adversarial setting. Left: Cumulative regret over time. Right: Cumulative unfair decisions over time.  \label{fig:linear-adversarial-comparison}Lines show mean values from 10 independent runs, with shaded areas representing 95\% confidence intervals. } 
{ 
}
\end{figure}

Under corruption, the non-robust methods exhibit much steeper growth in both regret and unfair decisions over the reported horizon. Fair OLS is also substantially affected, showing that fairness alone does not protect the learning process from manipulated feedback. By contrast, Robust Fair OLS has markedly lower regret and fewer unfair decisions than the non-robust methods. This pattern is consistent with the role of the corruption-dependent exploration and threshold adjustments in the robust algorithm.

For the smooth setting, we use the same nonlinear instance as in the stochastic experiment. Two of the four arms are targeted for corruption, with the attack attempting to reduce their expected rewards to $-0.1$ under the same total budget $C=200$. We again conduct 10 independent runs with horizon $T=10,000$.

As shown in Figure \ref{fig:linear-adversarial-comparison} (b), the Fair Smooth Bandit exhibits approximately linear growth in regret and unfair decisions over the reported horizon, while the simplified smooth bandit baseline also deteriorates under corruption. The Robust Fair Smooth Bandit displays the lowest regret trajectory and substantially fewer unfair decisions. Together with the linear experiment, these results illustrate the empirical distinction between fairness without robustness and fairness mechanisms that explicitly account for corrupted rewards.

\subsection{Semi-Synthetic Application}\label{subsec:realworldexp}

To complement the synthetic experiments, we consider a semi-synthetic wine brokerage application using the Wine Quality Dataset from the UCI Machine Learning Repository \citep{wine_quality_186}. The dataset contains 6,497 wines described by 11 physicochemical features and a sensory quality score on a 0--10 scale. We interpret each wine as a context arriving to a brokerage platform and three wine agents as the available arms: an economy agent, a mid-range agent, and a premium agent. The platform sequentially assigns wines to agents, with the objective of obtaining high profit while avoiding exposure patterns that are inconsistent with the agents' realized merits. The context vector \(\bx_t\) consists of the normalized physicochemical features of the wine observed at round \(t\).

We construct agent-specific profits from the observed wine quality score. The economy agent is favored by lower-quality wines, the premium agent by higher-quality wines, while the mid-range agent's reward is concentrated around intermediate quality levels. Specifically, the reward functions are defined as follows:
\[
\begin{aligned}
y_1(q) = \frac{2}{1 + \exp((q - 6))}, 
y_3(q) = \frac{2}{1 + \exp(-(q - 6))},
y_2(q) = y_1(q)\times y_3(q).
\end{aligned}
\]
This construction creates heterogeneous, quality-dependent rewards and provides a simple setting in which the three agents compete for exposure.

For the linear algorithms, we use the original 11-dimensional feature vector
after standard normalization. For the smooth algorithms, local polynomial
estimation becomes costly in the original dimension, so we use a
three-dimensional representation produced by a four-layer fully connected
neural network with ReLU activations. The representation is fixed before the
bandit experiments, and all algorithms are evaluated over 10 random
permutations of the wine samples.

Because the semi-synthetic rewards are based on observed wine-quality scores,
small reward differences may reflect measurement variability rather than
meaningful differences in merit. We therefore ignore reward gaps below $0.01$
when reporting unfair decisions; this threshold is
small relative to the $[0,2]$ reward scale used in the experiment.

We first consider the marketplace without adversarial corruption. As
shown in Figure \ref{fig:realdatalinear}, the fair algorithms substantially
reduce the cumulative number of unfair decisions relative to their
learning-based benchmarks, while maintaining competitive regret over the
observed horizon. The results illustrate a favorable empirical trade-off
between reward performance and merit-aligned exposure in this semi-synthetic
setting.

\begin{figure}[ht]
\FIGURE{
\begin{tabular}{cc}
    \subfigure[Linear Setting]{\includegraphics[width=0.72\textwidth]{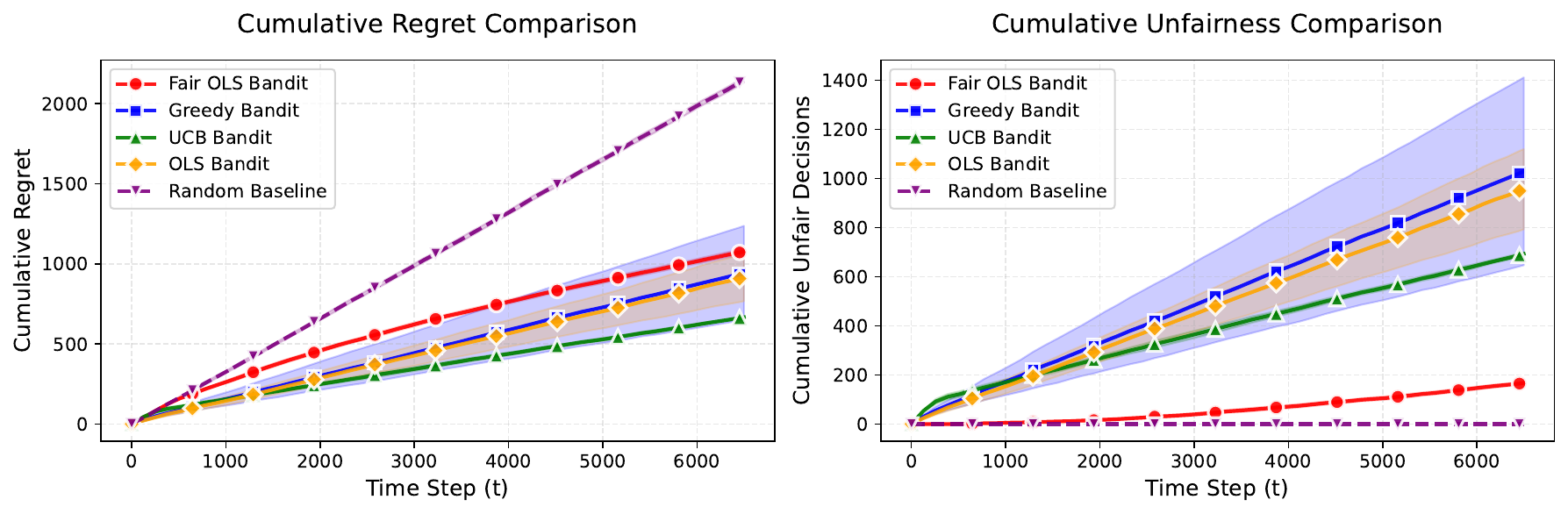}}\\
\subfigure[Smooth Setting]{\includegraphics[width=0.72\textwidth]{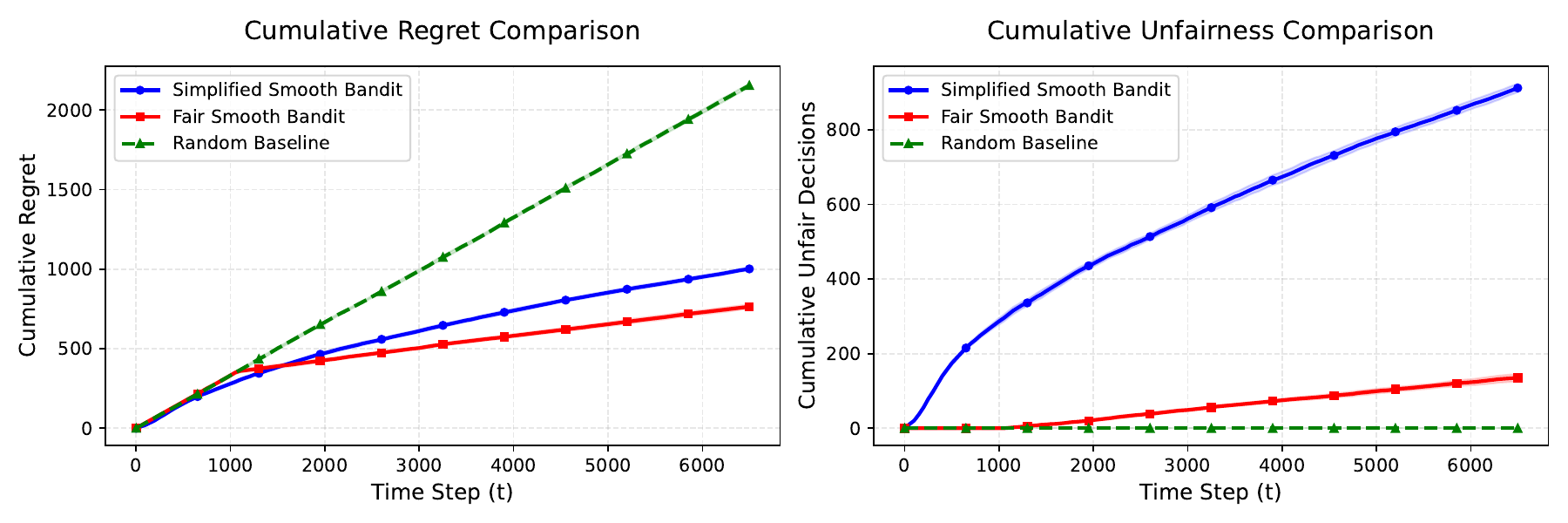}}
\end{tabular}
}
{Performance comparison of linear and smooth contextual bandit algorithms on the wine brokerage platform. Left: Cumulative regret over time. Right: Cumulative unfair decisions over time.  \label{fig:realdatalinear} The unfair-decision metric excludes pairwise comparisons with absolute reward gaps below 0.01. Lines show mean values from 10 independent runs, with shaded areas representing 95\% confidence intervals.} 
{
}
\end{figure}

We next introduce adversarial reward corruption in the same environment. The
attack targets the economy and premium agents by setting their corrupted
observed rewards to zero whenever the remaining budget permits, thereby making
the mid-range agent appear relatively more attractive during learning. We use
the same corruption budget across all competing methods. This setting allows
us to examine whether explicitly accounting for corrupted feedback improves
both reward performance and exposure fairness.

\begin{figure}[h]
\FIGURE{
\begin{tabular}{cc}
    \subfigure[Linear Setting]{\includegraphics[width=0.72\textwidth]{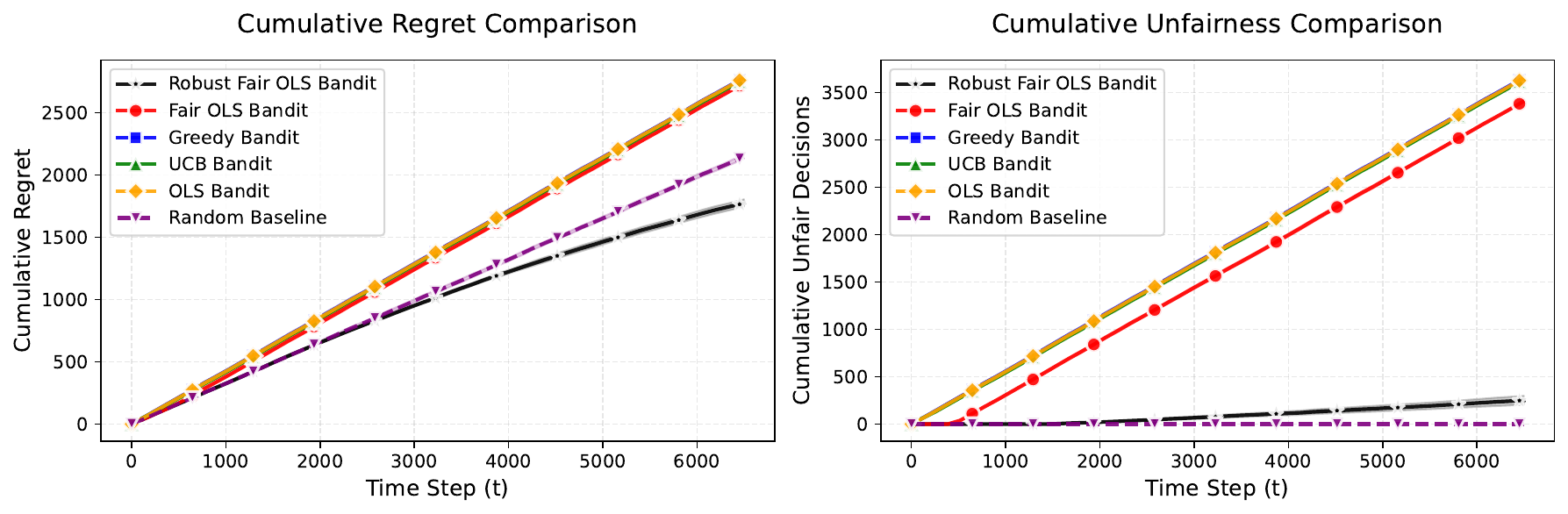}}\\
\subfigure[Smooth Setting]{\includegraphics[width=0.72\textwidth]{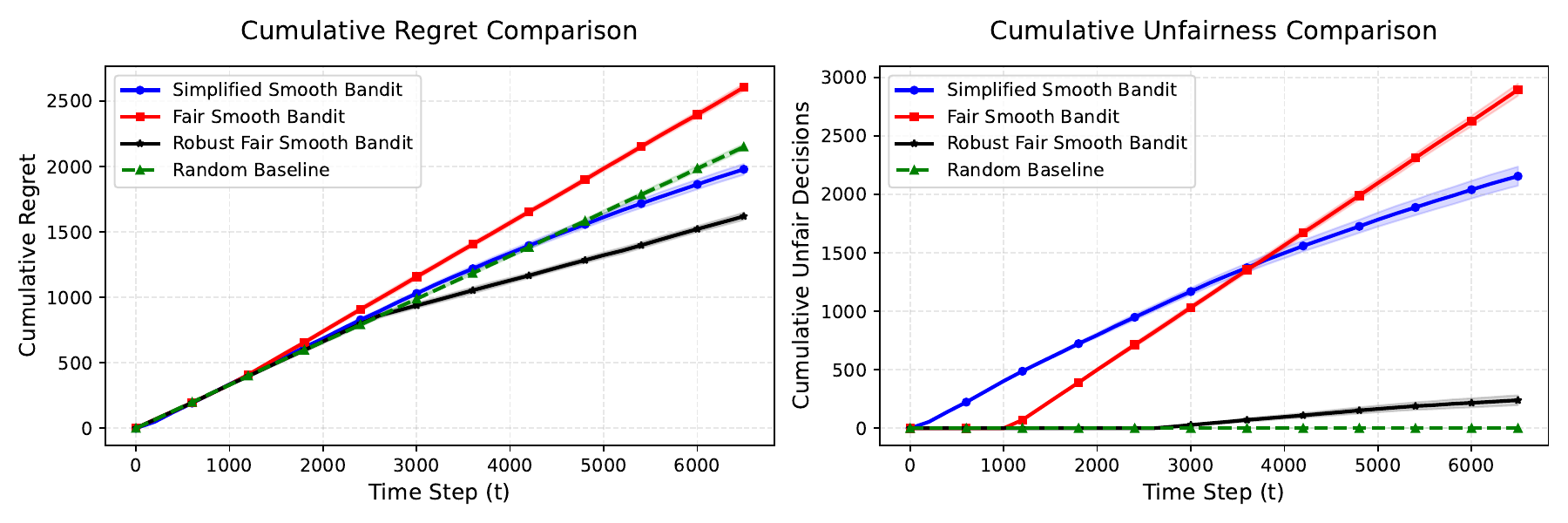}}
\end{tabular}
}
{Performance comparison of linear and smooth contextual bandit algorithms on the wine brokerage platform under attack. Left: Cumulative regret over time. Right: Cumulative unfair decisions over time.  \label{fig:realdatalinearcorrupt} The unfair-decision metric excludes pairwise comparisons with absolute reward gaps below 0.01. Lines show mean values from 10 independent runs, with shaded areas representing 95\% confidence intervals.} 
{
}
\end{figure}

Figure \ref{fig:realdatalinearcorrupt} shows that corruption substantially
degrades both regret and the empirical fairness metric for the non-robust
methods relative to the uncorrupted setting. In contrast, the robust fair
algorithms remain substantially more stable on both measures. These results
are consistent with the synthetic experiments and illustrate the benefit of
accounting explicitly for reward corruption when fairness is enforced through
estimated reward comparisons.

\section{Conclusions and Discussion}\label{sec_conclusion}

In this work, we study attack-resistant uniform fairness in contextual bandits with linear and smooth reward models. We develop fair algorithms, characterize their regret, and identify attacks that induce persistent unfairness under a $\widetilde O(1)$ corruption budget. We then construct robust variants that preserve the fairness guarantees under adversarial reward corruption. The regret bounds are minimax optimal for linear rewards and, up to logarithmic factors, for smooth rewards when $\alpha\beta\leq d$. Specifically, for linear rewards, we establish minimax-optimal regret bounds of $O(\log^2 T)$ without corruption and $O(\log^2 T+C)$ under a corruption budget $C$. For smooth rewards, the full regret rate is minimax optimal up to logarithmic factors when $\alpha\beta\le d$, both with and without corruption, while for all regimes the upper and lower bounds match in their dependence on $C$. 

Several directions remain open. When $\alpha\beta>d$, a full minimax
characterization of the regret for smooth rewards remains open: our results
identify the sharp $\Theta(C)$ corruption-dependent penalty in this regime,
but do not fully characterize the non-corruption term. Extending item-level guarantees to hierarchical fairness constraints, adapting to an unknown corruption budget, and reducing the computational cost of nonparametric estimation are also natural directions.

\bibliographystyle{apalike}
\bibliography{ref}

\appendix

\newpage
\setcounter{page}{1}

\numberwithin{figure}{section}
\numberwithin{table}{section}
\numberwithin{example}{section}

\section*{Supplemental Material}

This appendix provides the complete technical details supporting the main paper. It is organized as follows: Section \ref{ecsec:proof} contains proofs of the main theorems; Section \ref{ecsec:expdetail} presents details in experiments.

\section{Technical Proofs}\label{ecsec:proof}

This section provides the complete proofs of all main results, presented in the order they appear in the body of the paper.

\subsection{Proofs in Section \ref{sec:fairalglinear}}\label{EC_pf_fairbandit}

Define the fixed model-dependent constants
$D_1:=\frac{\lambda^{*2}\widetilde p^2}{32d^2r^4\sigma^2K^2}$,
$D_2:=\min\{\frac12,\frac{\lambda^*}{8r^2}\}$,
\[
C_{a,0}:=\frac{20K^2}{\widetilde pD_2}\vee\frac{8K^2}{\widetilde p^2}\vee\frac{640K}{h^2D_1}\vee\frac{40Kdr^2}{\widetilde p\lambda^*},
\]
and $C_{b,0}:=\frac{64\sqrt{2}K\sigma d r^2}{\widetilde p\lambda^*}$.
The first three terms in $C_{a,0}$ give the initial random-exploration estimator bounds used in Proposition~\ref{prop:randomtailbound}. The last term controls the matrix-Chernoff step in Lemma~\ref{lem:predictableols}: because $|\cT_0|=C_a\log T$, $C_a\geq C_{a,0}$, and $T>2d$,
$d\exp(-\frac{\widetilde p\lambda^*}{8Kdr^2}|\cT_0|)\leq T^{-4}$.
Moreover, for $c_{\mathrm G}=\widetilde p\lambda^*/(8K)$ and $c_{\mathrm S}=4\sigma r\sqrt d$, we have $C_{b,0}=2\sqrt{2dr^2}\,c_{\mathrm S}/c_{\mathrm G}$. 

We begin by stating two key propositions that support the subsequent analysis; their proofs are provided in Section \ref{subsec:ECproofinEC_pf_fairbandit}. Throughout this section, we define the event $\mathcal{A}=\{|\mathcal{I}_{k,0}|\geq \frac{1}{2K}|\mathcal{T}_{0}| \}$, event $\mathcal{B}=\{\max_{\bx\in\cX}\max_{k\in\mathcal{K}}|(\hat{\beta}_k(\mathcal{I}_{k,0})-\beta_k)^{\mathrm{T}}\bz|< \frac{h}{4}\}$ and event $\mathcal{W}_t=\{\max_{\bx\in\cX}\max_{k\in \cK}|(\hat{\beta}_k(\mathcal{I}_{k,t-1})-\beta_k)^{\mathrm{T}}\bz|< \frac{\epsilon_{t}}{2}\}$. For brevity, throughout the analysis, we say that fairness is fulfilled at time $t$, with the understanding that at $t$, \eqref{eq:defoffairness} holds for all context $\bx\in\cX$ and all pair of arms $i,j\in\mathcal{K}$.

\begin{proposition}\label{prop:randomtailbound} 
Suppose Assumptions \ref{assum_parabound}-\ref{assum_excite} hold,
$T>2d$, and $C_a\geq C_{a,0}$. Then the following tail inequality holds:
\begin{align*}
    \PP(\max_{\bx\in \cX}\max_{k\in\mathcal{K}}|(\hat{\beta}_k(\mathcal{I}_{k,0})-\beta_k)^{\mathrm{T}}\bz|\geq \frac{h}{4})\leq 4KT^{-4}.
\end{align*}
\end{proposition}

\begin{proposition}\label{prop:allsampletailbound}
Assume that the conditions in Proposition
\ref{prop:randomtailbound} are satisfied and $C_b\geq C_{b,0}$. Then the
following tail inequality holds for all $t>|\cT_0|$:
    \begin{align*}
\PP(\max_{\bx\in\cX}\max_{k\in \cK}|(\hat{\beta}_k(\mathcal{I}_{k,t})-\beta_k)^{\mathrm{T}}\bz|\geq \frac{\epsilon_{t+1}}{2})
    \leq 8KT^{-4}.
\end{align*}
\end{proposition}

\subsubsection{Proof of Theorem \ref{thm:linearfairness}} 
For $t\leq |\cT_0|$, Algorithm \ref{alg:fairols} pulls all arms equiprobably, so fairness holds deterministically. Fix $t>|\cT_0|$ and an arbitrary $\bx\in\cX$, and write $\bz=(1,\bx^{\mathrm T})^{\mathrm T}$.

If $\widehat{\cK}_{\bx}=\{k\}$, then the algorithm implies
\[
 \bz^{\mathrm T}\widehat\beta_{k,0}-\bz^{\mathrm T}\widehat\beta_{j,0}>h/2,
 \qquad j\in\cK\setminus\{k\}.
\]
On $\cB$, for each such $j$,
\[
 \bz^{\mathrm T}(\beta_k-\beta_j)>-h/4+h/2-h/4=0.
\]
Thus the deterministic choice of $k$ is fair.

Suppose next that $|\widehat{\cK}_{\bx}|>1$. The algorithm selects uniformly from $\cK_c(\bx)$, so it remains only to compare a selected arm $g\in\cK_c(\bx)$ with an excluded arm $j\in\cK\setminus\cK_c(\bx)$. There are two cases. If $j\in\widehat{\cK}_{\bx}\setminus\cK_c(\bx)$, then 
\[
 \bz^{\mathrm T}\widehat\beta_{g,t-1}-\bz^{\mathrm T}\widehat\beta_{j,t-1}>\epsilon_t,
\]
and on $\cW_t$,
\[
 \bz^{\mathrm T}(\beta_g-\beta_j)>-\epsilon_t/2+\epsilon_t-\epsilon_t/2=0.
\]
If instead $j\notin\widehat{\cK}_{\bx}$, then $g\in\widehat{\cK}_{\bx}$ and the two arms lie in different $h/2$-components decided by the initial estimation. Therefore
\[
 \bz^{\mathrm T}\widehat\beta_{g,0}-\bz^{\mathrm T}\widehat\beta_{j,0}>h/2,
\]
which on $\cB$ gives $\bz^{\mathrm T}(\beta_g-\beta_j)>0$ by the same calculation as above. Consequently, fairness holds at round $t$ on $\cB\cap\cW_t$.

Propositions \ref{prop:randomtailbound} and \ref{prop:allsampletailbound} give
\[
 \PP((\cB\cap\cW_t)^C)\leq 4KT^{-4}+8KT^{-4}=12KT^{-4}.
\]
A union bound over $t\in[T]$ shows that fairness holds simultaneously at all rounds, contexts, and arm pairs with probability at least $1-12KT^{-3}>1-1/T$ whenever $T^2>12K$. The stated condition on $T$ also ensures the auxiliary requirement $T>2d$. This proves the theorem. \hfill\Halmos

\subsubsection{Proof of Theorem \ref{thm:minimaxoptimallinear}}
For $t>|\cT_0|$, define the good event $\cG_t^{\mathrm{lin}}:=\cB\cap\cW_t$. By Propositions \ref{prop:randomtailbound} and \ref{prop:allsampletailbound},
\[
 \PP\bigl((\cG_t^{\mathrm{lin}})^C\bigr)\leq 12KT^{-4}.
\]
Let $\Delta_t:=\max_{k\in\cK}\bz_t^{\mathrm T}\beta_k
-\bz_t^{\mathrm T}\beta_{\pi_t}.$

We decompose the cumulative regret into three parts:
\begin{align*}
R_T
=&\sum_{t=1}^{|\mathcal T_0|}\EE[\Delta_t]
+\sum_{t=|\mathcal T_0|+1}^T
\EE\left[\Delta_t\mathbb I\{\cG_t^{\mathrm{lin}}\}\right]+\sum_{t=|\mathcal T_0|+1}^T
\EE\left[\Delta_t\mathbb I\{(\cG_t^{\mathrm{lin}})^c\}\right]\\
=:&R_1+R_2+R_3.
\end{align*}

Bounded rewards give
\[
 R_1=O(\log T),\qquad
 R_3\lesssim\sum_{t>|\cT_0|}\PP\bigl((\cG_t^{\mathrm{lin}})^C\bigr)=O(1).
\]

It remains to control $R_2$. If $|\widehat\cK_{\bx_t}|=1$, the proof of Theorem \ref{thm:linearfairness} shows that its unique arm is truly optimal, so the instantaneous regret is zero. Otherwise, let $o\in\argmax_{l\in\cK}\bz_t^{\mathrm T}\beta_l$ and $v\in\argmax_{l\in\widehat\cK_{\bx_t}}\bz_t^{\mathrm T}\widehat\beta_{l,t-1}.$
We first justify that $o\in\widehat\cK_{\bx_t}$. Let
\(
 r_0\in\argmax_{l\in\cK}\bz_t^{\mathrm T}\widehat\beta_{l,0}.
\)
Because $o$ is a true maximizer, $\bz_t^{\mathrm T}(\beta_{r_0}-\beta_o)\leq0$, and on $\cB$,
\[
 0\leq \bz_t^{\mathrm T}(\widehat\beta_{r_0,0}-\widehat\beta_{o,0})
 < h/4+0+h/4=h/2.
\]
Thus $o$ is directly $h/2$-linked to $r_0$ and hence belongs to $\widehat\cK_{\bx_t}$. By the definition of $v$,
\begin{align*}
0&\geq \bz_t^{\mathrm T}\widehat\beta_{o,t-1}-\bz_t^{\mathrm T}\widehat\beta_{v,t-1}\\
&=\bz_t^{\mathrm T}(\widehat\beta_{o,t-1}-\beta_o)
 +\bz_t^{\mathrm T}(\beta_o-\beta_v)
 +\bz_t^{\mathrm T}(\beta_v-\widehat\beta_{v,t-1})
\geq-\epsilon_t.
\end{align*}
Therefore $o$ is directly $\epsilon_t$-linked to $v$, so $o\in\cK_c(\bx_t)$. Each $k\in\cK_c(\bx_t)$ is connected to $v$ by a chain containing at most $K-1$ links; on $\cW_t$,
\[
 \bz_t^{\mathrm T}(\beta_o-\beta_k)
 \leq \epsilon_t/2+(K-1)\epsilon_t+\epsilon_t/2
 \leq K\epsilon_t.
\]
In particular, the instantaneous regret is at most $K\epsilon_t$. Since $\mathcal G_t^{\mathrm{lin}}$ is measurable with respect to the history before round $t$, while $X_t$ is an independent fresh context, conditioning on $\mathcal G_t^{\mathrm{lin}}$ does not change the distribution of $X_t$. Hence, the margin bound in Assumption~\ref{assum_margin1} continues to hold conditionally on $\mathcal G_t^{\mathrm{lin}}$. A union bound over arm pairs therefore yields
\[
 \EE\left[\left(\max_{k\in\cK}\bz_t^{\mathrm T}\beta_k-
 \bz_t^{\mathrm T}\beta_{\pi_t}\right)\mathbb I(\cG_t^{\mathrm{lin}})\right]
 \lesssim K\epsilon_t\,\PP\left(0<\max_{k\in\cK}\bz_t^{\mathrm T}\beta_k-
 \bz_t^{\mathrm T}\beta_{\pi_t}\leq K\epsilon_t\right)
 \lesssim\epsilon_t^2\lesssim\frac{\log T}{t}.
\]
Summing over the horizon from $t=|\mathcal{T}_0|+1$ to $T$  gives $R_2=O(\log^2T)$. Combining the three terms proves $R_T=O(\log^2T)$. 

\hfill\Halmos

\subsubsection{Proof of Theorem \ref{thm:price}}

\begin{lemma}\label{lem:lemmaec2.1}
Let $\Theta\sim\Pi$ and, conditional on $\Theta$, let $H\sim P_\Theta$ be an arbitrary random history. Given $H$, draw $\Theta'$ from the posterior distribution $\Pi(\cdot\mid H)$. Then
\[
 (\Theta,H)\overset{d}{=}(\Theta',H).
\]
\end{lemma}
\textit{Proof.}
For any measurable sets $A$ and $B$,
\begin{align*}
 \PP(\Theta'\in A,H\in B)
 &=\EE\left[\mathbb I\{H\in B\}\Pi(A\mid H)\right]\\
 &=\PP(\Theta\in A,H\in B),
\end{align*}
which proves the equality in distribution. \hfill\Halmos

\noindent\textit{Proof of Theorem \ref{thm:price}.} 
Let $\{x_t\}_{t \geq 1}$ be i.i.d. contexts at time $t$, each uniformly distributed on $\mathcal{X} = [-1, 1]$. Consider the observations of arm 1 and 2: $Y_t^{(1)} = x_t + \epsilon_t$ and $Y_t^{(2)} = \theta + \epsilon_t$, where $\theta \sim \mathrm{Unif}(-1/2, 1/2)$ and $\epsilon_t \sim \mathcal{N}(0, 1)$.

\noindent\textbf{Step 1. Verifying the Assumptions.}
We verify assumptions one by one. For Assumption \ref{assum_parabound}, the contexts $\{x_t : t=1,2,\ldots\}$ are i.i.d. with density $1/2$ with respect to Lebesgue measure, drawn from the fixed distribution $\mathbb{P}_X$ with support $\mathcal{X} = [-1, 1] \subseteq [-r, r]$ by taking $r=1 \geq 1$. Moreover, $\|\beta_1\|_2 = \|(0, 1)^\mathrm{T}\|_2 = 1$ and $\|\beta_2\|_2 = \|(\theta, 0)^\mathrm{T}\|_2 = |\theta| \leq 1/2 < 1$, so $\|\beta_k\|_2 \leq b=1$ holds for all $k \in \{1,2\}$. Thus, Assumption \ref{assum_parabound} is satisfied.

For Assumption \ref{assum_margin1}, consider the difference vectors: for $i=1, j=2$, $\beta_1 - \beta_2 = (-\theta, 1)^\mathrm{T}$, so $\bz_t^\mathrm{T} (\beta_1 - \beta_2) = x_t - \theta$; the case $i=2, j=1$ yields $\theta - x_t = -(x_t - \theta)$, with the same absolute value. The probability $\mathbb{P}(0 < |x_t - \theta| \leq \rho)$ is required to satisfy the bound in Assumption \ref{assum_margin1} for all $\rho > 0$. Since $x_t \sim \mathrm{Unif}[-1, 1]$ with density $1/2$ and $|\theta| \leq 1/2$, for $\rho \leq 1/2$ the interval $[\theta - \rho, \theta + \rho] \cap [-1, 1] = [\theta - \rho, \theta + \rho]$, with Lebesgue measure $2\rho$. We obtain $\mathbb{P}(0 < |x_t - \theta| \leq \rho) = (1/2) \cdot 2\rho = \rho$. For $\rho > 1/2$, the probability is bounded by $1 < 2\rho$. Thus, taking $L=2 > 0$ satisfies the condition for all $\rho > 0$ and all $i \neq j$. Assumption \ref{assum_margin1} therefore holds with $L=2$.

For Assumption \ref{assum_excite}, take $h = 1/4 > 0$. For arm 1, $Q_1 = \{\bz_t : \bz_t^\mathrm{T} \beta_1 > \bz_t^\mathrm{T} \beta_2 + h \} = \{x_t > \theta + h \}$; since $|\theta| \leq 1/2$, $\theta + h \leq 3/4 < 1$, so $\mathbb{P}(Q_1) = [1 - (\theta + h)] / 2 \geq  (1/4)/2 = 1/8$. For arm 2, $Q_2 = \{\bz_t : \bz_t^\mathrm{T} \beta_2 > \bz_t^\mathrm{T} \beta_1 + h \} = \{x_t < \theta - h \}$; $\theta - h \geq -3/4 > -1$, so $\mathbb{P}(Q_2) = [(\theta - h) - (-1)] / 2 \geq (1/4)/2 = 1/8$. Thus, $\mathbb{P}(\bz_t \in Q_k) \geq \tilde{p} = 1/8 $ for all $k \in \mathcal{K}$. For the minimum eigenvalue condition, note that $\bz_t \bz_t^\mathrm{T} = \begin{pmatrix} 1 & x_t \\ x_t & x_t^2 \end{pmatrix},$ so $\mathbb{E}\left[\mathbf{z}_t \mathbf{z}_t^{\mathrm{T}} \mid Q_k\right]=\left(\begin{array}{cc}1 & \mu_k \\ \mu_k & \mathbb{E}\left[x_t^2 \mid Q_k\right]\end{array}\right)$, where $\mu_k = \mathbb{E}[x_t \mid Q_k]$ and $\mathbb{E}[x_t^2 \mid Q_k] = \mathrm{Var}(x_t \mid Q_k) + \mu_k^2$. It can be verified that $\min_k \lambda_{\min}(\mathbb{E}[\bz_t \bz_t^\mathrm{T} \mid Q_k]) \geq 0.002 $, so taking $\lambda^* = 0.002$ satisfies the condition. Assumption \ref{assum_excite} therefore holds with $h=1/4$, $\tilde{p}=1/8$, and $\lambda^*=0.002$.

Next we establish the lemmas used in this proof.

\begin{lemma}\label{lem:adaptivegaussianscore}
Let $A_u\in\{0,1\}$ be measurable after the action at round $u$ is chosen and before a noise $\varepsilon_u\sim\cN(0,\sigma^2)$ is revealed. Set
\[
 V_t:=\sum_{u=1}^tA_u,\qquad Z_t:=\sum_{u=1}^tA_u\varepsilon_u.
\]
Then, for every $\eta\in(0,1)$,
\begin{equation}\label{eq:adaptive-gaussian-score}
 \PP\left(V_t\geq1,\ |Z_t|>
 \sigma\sqrt{2V_t\log\frac{2t}{\eta}}\right)\leq\eta.
\end{equation}
\end{lemma}

\textit{Proof.} For each round $u$, let $\mathcal G_u$ denote the information available after the action at round $u$ is selected and before $\varepsilon_u$ is observed. By construction, $A_u$ is $\mathcal G_u$-measurable, while $\varepsilon_u\sim\mathcal N(0,\sigma^2)$ is independent of $\mathcal G_u$. Hence, for any $\lambda\in\mathbb R$,
\begin{align*}
\mathbb E\left[\exp\left\{\lambda A_u\varepsilon_u-\frac{\lambda^2\sigma^2}{2}A_u\right\}\;\middle|\;\mathcal G_u\right]= \mathbb I\{A_u=0\} + \mathbb I\{A_u=1\}\exp\left\{-\frac{\lambda^2\sigma^2}{2}\right\}\mathbb E\left[e^{\lambda\varepsilon_u}\right]=1,
\end{align*}
where the last equality follows from the moment generating function of a centered Gaussian random variable, $\mathbb E[e^{\lambda\varepsilon_u}]=\exp\{\lambda^2\sigma^2/2\}$. Iterating conditional expectations therefore yields, for every $\lambda\in\mathbb R$,
\begin{equation}\label{eq:adaptive-gaussian-exp-martingale}
\mathbb E\exp\left\{\lambda Z_t-\frac{\lambda^2\sigma^2}{2}V_t\right\}=1.
\end{equation}
Let $L_\eta:=\log(2t/\eta)$. For each $n\in[t]$, define $\lambda_n:=\sqrt{2L_\eta/(\sigma^2 n)}$. On the event $\{V_t=n,\; Z_t>\sigma\sqrt{2nL_\eta}\}$, we have
\[
\lambda_n Z_t-\frac{\lambda_n^2\sigma^2}{2}V_t
> \lambda_n\sigma\sqrt{2nL_\eta}-\frac{\lambda_n^2\sigma^2}{2}n
= 2L_\eta-L_\eta = L_\eta.
\]
Therefore, by Markov's inequality and \eqref{eq:adaptive-gaussian-exp-martingale},
\begin{align*}
\mathbb P\left(V_t=n,\; Z_t>\sigma\sqrt{2nL_\eta}\right)
&\leq e^{-L_\eta}\mathbb E\exp\left\{\lambda_n Z_t-\frac{\lambda_n^2\sigma^2}{2}V_t\right\}
\leq \frac{\eta}{2t}.
\end{align*}
Applying the same argument to $-Z_t$ gives $\mathbb P(V_t=n,\; Z_t<-\sigma\sqrt{2nL_\eta})\leq \eta/(2t)$. Finally, summing the two bounds over $n=1,\ldots,t$ yields
\begin{align*}
\mathbb P\left(V_t\geq1,\; |Z_t|>\sigma\sqrt{2V_t\log(2t/\eta)}\right)
&\leq \sum_{n=1}^t \mathbb P\left(V_t=n,\; |Z_t|>\sigma\sqrt{2n\log(2t/\eta)}\right)
\leq \eta.
\end{align*}
This proves \eqref{eq:adaptive-gaussian-score}. \hfill\Halmos

\begin{lemma}\label{lem:posteriorconfusion}
In the two-arm instance used in Theorem~\ref{thm:price}, let
\[
s_t:=\sum_{u=1}^t\mathbb I\{\pi_u=2\},\qquad
Z_t:=\sum_{u=1}^t\mathbb I\{\pi_u=2\}\varepsilon_{2,u},
\]
and let $\delta'=\sqrt\delta$. There are absolute constants $c_0,c_1,c_2,c_3>0$ such that, whenever $\delta=\Theta(1/T)$, $|\theta_0|\leq1/4$, and $c_0\log(1/\delta')\leq t\leq T$, the event
\[
\mathcal H_t:=\left\{s_t\geq t/16,\; |Z_t|\leq\sqrt{2s_t\log(2t/\delta')}\right\}
\]
has the following property. On $\mathcal H_t$ there is an $\mathcal F_t$-measurable interval $\mathcal J_t\subset[-1/2,1/2]$ such that, with $r_t:=\sqrt{\log(1/\delta')/t}$,
\begin{align}\label{eq:posterior-confusion-lemma}
&\operatorname{Leb}(\mathcal J_t)\geq c_1 r_t,\nonumber\\
&\mathbb P(\theta'\leq x\mid\mathcal F_t)\geq\delta',\qquad
 \mathbb P(\theta'\geq x\mid\mathcal F_t)\geq\delta',\qquad x\in\mathcal J_t,\nonumber\\
&c_2 r_t\leq\mathbb E(|x-\theta'|\mid\mathcal F_t)\leq c_3 r_t,\qquad x\in\mathcal J_t,
\end{align}
where $\theta'\sim\mathbb P_{\theta\mid\mathcal F_t}$.
\end{lemma}

\textit{Proof.} We first characterize the posterior distribution of $\theta$ given the adaptive history. For a fixed realized history $\mathcal F_t$, the likelihood factors corresponding to the contexts and the policy's action probabilities do not depend on $\theta$. Moreover, the reward distribution of arm $1$ is also independent of $\theta$. Hence, the only $\theta$-dependent likelihood terms are the rewards observed when arm $2$ is selected. Therefore, under the uniform prior on $[-1/2,1/2]$,
\begin{align*}
\pi(\theta\mid\mathcal F_t)
&\propto \mathbb I\{|\theta|\leq 1/2\}\prod_{u:\pi_u=2}\exp\left\{-\frac{1}{2}(Y_u^{(2)}-\theta)^2\right\} \\
&\propto \mathbb I\{|\theta|\leq 1/2\}\exp\left\{-\frac{s_t}{2}(\theta-\bar Y_t)^2\right\},
\end{align*}
where $\bar Y_t:=s_t^{-1}\sum_{u:\pi_u=2}Y_u^{(2)}=\theta_0+Z_t/s_t$. Consequently, whenever $s_t\geq1$,
\begin{equation}\label{eq:posterior-truncated-normal}
\theta'\mid\mathcal F_t\sim\mathcal N(\bar Y_t,1/s_t)\quad\text{truncated to }[-1/2,1/2].
\end{equation}

Let $q:=\sqrt{\log(1/\delta')}$. Since $\delta=\Theta(1/T)$ and $\delta'=\sqrt{\delta}$, we have $q^2=\Theta(\log T)$. Moreover, because $t\leq T$, $\log(2t/\delta')\leq C_1\log(1/\delta')=C_1q^2$ for some absolute constant $C_1$ and all sufficiently large $T$. On $\mathcal H_t$, it follows that
\begin{equation}\label{eq:posterior-center-control}
|\bar Y_t-\theta_0|=|Z_t|/s_t\leq \sqrt{2\log(2t/\delta')/s_t}\leq C_2 q/\sqrt{s_t}
\end{equation}
for some absolute constant $C_2>0$.

Since $s_t\geq t/16$ and $t\geq c_0q^2$, we have $q/\sqrt{s_t}\leq 4q/\sqrt t\leq 4/\sqrt{c_0}$. Thus, by taking $c_0$ sufficiently large and using $|\theta_0|\leq1/4$, both $\bar Y_t$ and the interval
\begin{equation}\label{eq:def-posterior-confusion-interval}
\mathcal J_t:=[\bar Y_t+q/(4\sqrt{s_t}),\; \bar Y_t+q/(2\sqrt{s_t})]
\end{equation}
lie strictly inside $[-1/2,1/2]$.

We next establish the posterior-tail bounds on $\mathcal J_t$. Let $X_t\sim\mathcal N(\bar Y_t,1/s_t)$ denote the corresponding untruncated Gaussian variable. For any $x\in\mathcal J_t$,
\begin{equation}\label{eq:standardized-confusion-distance}
z_x:=\sqrt{s_t}(x-\bar Y_t)\in[q/4,q/2].
\end{equation}
Hence, by the Gaussian Mills lower bound,
\begin{align}
\mathbb P(X_t\geq x)
&=\overline\Phi(z_x)\geq \overline\Phi(q/2) \nonumber\\
&\geq \frac{c}{1+q}\exp(-q^2/8)=\frac{c}{1+q}(\delta')^{1/8},
\label{eq:posterior-right-tail-untruncated}
\end{align}
for an absolute constant $c>0$. Since $(\delta')^{1/8}/(1+q)\gg\delta'$ as $\delta'\to0$, the RHS of \eqref{eq:posterior-right-tail-untruncated} is at least $4\delta'$ for all sufficiently large $T$. Also, because $x\geq\bar Y_t$, $\mathbb P(X_t\leq x)\geq1/2$.

It remains to verify that truncation does not essentially affect these bounds. By the fixed separation of $\bar Y_t$ from $\{-1/2,1/2\}$, there exists $c'>0$ such that $\mathbb P(X_t<-1/2)+\mathbb P(X_t>1/2)\leq 2e^{-c's_t}$. Since $s_t\geq t/16\geq c_0q^2/16$, increasing $c_0$ if necessary gives
\begin{equation}\label{eq:posterior-truncation-tail}
\mathbb P(X_t<-1/2)+\mathbb P(X_t>1/2)\leq\delta'
\end{equation}
for all sufficiently large $T$. Combining \eqref{eq:posterior-right-tail-untruncated}--\eqref{eq:posterior-truncation-tail} with the truncated-Gaussian representation \eqref{eq:posterior-truncated-normal}, we obtain, uniformly over $x\in\mathcal J_t$,
\[
\mathbb P(\theta'\geq x\mid\mathcal F_t)
=\frac{\mathbb P(x\leq X_t\leq1/2)}{\mathbb P(-1/2\leq X_t\leq1/2)}
\geq \mathbb P(X_t\geq x)-\mathbb P(X_t>1/2)\geq\delta',
\]
and similarly,
\[
\mathbb P(\theta'\leq x\mid\mathcal F_t)
=\frac{\mathbb P(-1/2\leq X_t\leq x)}{\mathbb P(-1/2\leq X_t\leq1/2)}
\geq \mathbb P(X_t\leq x)-\mathbb P(X_t<-1/2)\geq\delta'.
\]
This proves the second statement in \eqref{eq:posterior-confusion-lemma}.

The length of the confusion interval satisfies $\operatorname{Leb}(\mathcal J_t)=q/(4\sqrt{s_t})\geq q/(4\sqrt t)=r_t/4$, so the first statement holds with $c_1=1/4$.

Finally, we establish the posterior first-moment bound. By \eqref{eq:posterior-truncation-tail}, the truncation normalizing constant is bounded below by an absolute positive constant. Therefore,
\begin{equation}\label{eq:posterior-centered-first-moment}
\mathbb E(|\theta'-\bar Y_t|\mid\mathcal F_t)\leq C_3/\sqrt{s_t}
\end{equation}
for some $C_3>0$. For $x\in\mathcal J_t$, $q/(4\sqrt{s_t})\leq |x-\bar Y_t|\leq q/(2\sqrt{s_t})$. Hence, by triangle inequality and \eqref{eq:posterior-centered-first-moment},
\[
\mathbb E(|x-\theta'|\mid\mathcal F_t)\leq |x-\bar Y_t|+\mathbb E(|\theta'-\bar Y_t|\mid\mathcal F_t)\leq (q/2+C_3)/\sqrt{s_t},
\]
while the reverse triangle inequality gives
\[
\mathbb E(|x-\theta'|\mid\mathcal F_t)\geq |x-\bar Y_t|-\mathbb E(|\theta'-\bar Y_t|\mid\mathcal F_t)\geq (q/4-C_3)/\sqrt{s_t}.
\]
Since $q\to\infty$, for all sufficiently large $T$ the last two bounds imply $c_2 q/\sqrt{s_t}\leq \mathbb E(|x-\theta'|\mid\mathcal F_t)\leq c_3' q/\sqrt{s_t}$ for some absolute constants $c_2,c_3'>0$. Using $t/16\leq s_t\leq t$ and $r_t=q/\sqrt t$ therefore yields $c_2' r_t\leq \mathbb E(|x-\theta'|\mid\mathcal F_t)\leq c_3 r_t$ uniformly for $x\in\mathcal J_t$, after redefining the absolute constants if necessary. This proves \eqref{eq:posterior-confusion-lemma}.

\hfill\Halmos

\noindent\textbf{Step 2. Adaptive sampling and posterior confusion.}We say that \(\mathcal{F}_t\) is unfair for \((\theta, x)\) if the round-\((t + 1)\) policy strictly prefers arm 1 at \(x\)
while \(\theta \geq x\), or strictly prefers arm 2 while \(\theta \leq x\). Let \(\mathcal{E}\) denote the event that fairness
holds at all rounds and all contexts. Fix $\theta_0\in[-1/2,1/2]$ and define $s_t:=\sum_{u=1}^t\mathbb I\{\pi_u=2\}$, $p_u:=\mathbb P(\pi_u=2\mid\mathcal F_{u-1})$. Here $p_u$ integrates over the fresh context $x_u$ and policy randomization at round $u$. On the global fairness event $\mathcal E$, the fairness inequalities hold at all rounds and contexts. In particular, whenever $x<\theta_0$, arm $2$ has the larger mean reward, and hence $\mathbb P(\pi_u=2\mid\mathcal F_{u-1},x_u=x)\geq1/2$. Since $x_u\sim\operatorname{Unif}[-1,1]$ independently of $\mathcal F_{u-1}$, $\mathbb P(x_u<\theta_0)=(\theta_0+1)/2\geq1/4$ for $\theta_0\in[-1/2,1/2]$. Therefore, along any history prefix on which fairness holds,
\[
p_u=\mathbb E[\mathbb P(\pi_u=2\mid\mathcal F_{u-1},x_u)\mid\mathcal F_{u-1}]\geq \frac12\mathbb P(x_u<\theta_0\mid\mathcal F_{u-1})\geq\frac18.
\]
Let $M_u:=\mathbb I\{\pi_u=2\}-p_u$. Then $(M_u)_{u\geq1}$ is a martingale-difference sequence with $\mathbb E(M_u\mid\mathcal F_{u-1})=0$. On $\mathcal E$, $\sum_{u=1}^t p_u\geq t/8$, hence $\{s_t<t/16\}\cap\mathcal E\subseteq\{\sum_{u=1}^t M_u<-t/16\}$. Since $\mathbb P(\mathcal E^c)\leq\delta$, martingale Hoeffding gives
\begin{equation}\label{eq:price-adaptive-count}
\mathbb P_{\mathcal F_t\mid\theta_0}(s_t<t/16)\leq\delta+\mathbb P\left(\sum_{u=1}^tM_u<-t/16\right)\leq\delta+e^{-c_A t}
\end{equation}
for an absolute constant $c_A>0$.

For the noise accumulated from arm $2$, apply Lemma~\ref{lem:adaptivegaussianscore} with $A_u=\mathbb I\{\pi_u=2\}$, $\sigma=1$, $\eta=\delta'$, where $\delta':=\sqrt{\delta}$. Thus
\begin{equation}\label{eq:price-adaptive-score}
\mathbb P_{\mathcal F_t\mid\theta_0}\left(s_t\geq1,\; |Z_t|>\sqrt{2s_t\log(2t/\delta')}\right)\leq\delta'.
\end{equation}
Combining \eqref{eq:price-adaptive-count} and \eqref{eq:price-adaptive-score}, for $|\theta_0|\leq1/4$ and $t\geq c_0\log(1/\delta')$,
\[
\mathbb P_{\mathcal F_t\mid\theta_0}(\mathcal H_t^c)\leq\delta+e^{-c_A t}+\delta',
\]
where $\mathcal H_t=\{s_t\geq t/16,\; |Z_t|\leq\sqrt{2s_t\log(2t/\delta')}\}$.

By Lemma~\ref{lem:posteriorconfusion}, on $\mathcal H_t$ there exists an $\mathcal F_t$-measurable interval $\mathcal J_t(\mathcal F_t)\subset[-1/2,1/2]$ such that, with $r_t:=\sqrt{\log(1/\delta')/t}$,
\begin{align}
&\operatorname{Leb}(\mathcal J_t(\mathcal F_t))\geq c_1r_t,\label{eq:price-confusion-length}\\
&\mathbb P(\theta'\leq x\mid\mathcal F_t)\geq\delta',\quad \mathbb P(\theta'\geq x\mid\mathcal F_t)\geq\delta',\quad x\in\mathcal J_t(\mathcal F_t),\label{eq:price-confusion-tails}\\
&c_2r_t\leq\mathbb E(|x-\theta'|\mid\mathcal F_t)\leq c_3r_t,\quad x\in\mathcal J_t(\mathcal F_t),\label{eq:price-confusion-gap}
\end{align}
where $\theta'\sim\mathbb P_{\theta\mid\mathcal F_t}$.

Let $\mathcal G_t$ denote the $\mathcal F_t$-measurable set of histories for which $\mathcal J_t(\mathcal F_t)$ satisfies \eqref{eq:price-confusion-length}--\eqref{eq:price-confusion-gap}; set $\mathcal J_t(\mathcal F_t)=\varnothing$ outside $\mathcal G_t$. Since $\mathcal H_t\subseteq\mathcal G_t$ for all $|\theta_0|\leq1/4$,
\begin{equation}\label{eq:price-good-history-fixed-theta}
\mathbb P_{\mathcal F_t\mid\theta_0}(\mathcal G_t^c)\leq\delta+e^{-c_A t}+\delta',\qquad |\theta_0|\leq1/4.
\end{equation}

Now place the uniform prior $\Theta\sim\operatorname{Unif}[-1/2,1/2]$, and let $\mathbb P_\Pi$ and $\mathbb E_\Pi$ denote probability and expectation under the joint distribution of $(\Theta,\mathcal F_t)$. Since the prior assigns mass $1/2$ to $[-1/4,1/4]$, averaging \eqref{eq:price-good-history-fixed-theta} over this interval gives
\[
\mathbb P_\Pi(\mathcal G_t)\geq\int_{-1/4}^{1/4}\mathbb P_{\mathcal F_t\mid\theta}(\mathcal G_t)\,d\theta
\geq \frac12(1-\delta-e^{-c_A t}-\delta').
\]
In particular, for all sufficiently large $T$ and $c_0\log(1/\delta')\leq t\leq T$,
\begin{equation}\label{eq:price-good-history-constant}
\mathbb P_\Pi(\mathcal G_t)\geq c_G
\end{equation}
for an absolute constant $c_G>0$.

\noindent\textbf{Step 3. Integrated fairness constraint.} For a realized history $\mathcal F_t$ and context $x\in[-1,1]$, define $u_t(x,\mathcal F_t):=\mathbb P_{\theta'\mid\mathcal F_t}\{\mathcal F_t\text{ is unfair for }(\theta',x)\}$. Fix $\theta$ and $x$. Uniform $(1-\delta)$-fairness implies $\mathbb P_{\mathcal F_t\mid\theta}\{\mathcal F_t\text{ is unfair for }(\theta,x)\}\leq\delta$. Averaging over $\Theta\sim\operatorname{Unif}[-1/2,1/2]$ yields $\mathbb P_\Pi\{\mathcal F_t\text{ is unfair for }(\Theta,x)\}\leq\delta$.

Apply Lemma~\ref{lem:lemmaec2.1} with $H=\mathcal F_t$.  Hence $(\Theta,\mathcal F_t)\overset{d}{=}(\theta',\mathcal F_t)$, $\theta'\sim\mathbb P_{\theta\mid\mathcal F_t}$. Therefore, for fixed $x$,
\[
\mathbb E_\Pi[u_t(x,\mathcal F_t)]=\mathbb E_{\mathcal F_t}[\mathbb P_{\theta'\mid\mathcal F_t}\{\mathcal F_t\text{ is unfair for }(\theta',x)\}]
=\mathbb P_\Pi\{\mathcal F_t\text{ is unfair for }(\Theta,x)\}\leq\delta.
\]
Integrating over $x\in[-1,1]$ and applying Tonelli's theorem gives
\begin{equation}\label{eq:price-integrated-unfair-mass}
\mathbb E_\Pi\int_{-1}^1u_t(x,\mathcal F_t)\,dx\leq2\delta.
\end{equation}

Let $\mathcal A_t(\mathcal F_t)$ denote the set of contexts at which the policy used at round $t+1$ strictly prefers one arm over the other. For $x\in\mathcal J_t(\mathcal F_t)$, if arm $1$ is strictly preferred, then whenever $\theta'>x$, arm $2$ has larger mean and the policy is unfair. Hence, on $\mathcal G_t$, $u_t(x,\mathcal F_t)\geq\mathbb P(\theta'>x\mid\mathcal F_t)\geq\delta'$. Similarly, if arm $2$ is preferred, $u_t(x,\mathcal F_t)\geq\mathbb P(\theta'<x\mid\mathcal F_t)\geq\delta'$. Consequently, pathwise on $\mathcal G_t$,
\[
\delta'\operatorname{Leb}(\mathcal A_t(\mathcal F_t)\cap\mathcal J_t(\mathcal F_t))
\leq\int_{\mathcal A_t(\mathcal F_t)\cap\mathcal J_t(\mathcal F_t)}u_t(x,\mathcal F_t)\,dx
\leq\int_{-1}^1u_t(x,\mathcal F_t)\,dx.
\]
Since $\mathcal J_t=\varnothing$ outside $\mathcal G_t$, taking expectations and using \eqref{eq:price-integrated-unfair-mass} gives
\begin{equation}\label{eq:price-strict-set}
\mathbb E_\Pi\operatorname{Leb}(\mathcal A_t(\mathcal F_t)\cap\mathcal J_t(\mathcal F_t))\leq2\sqrt{\delta}.
\end{equation}

\noindent\textbf{Step 4. Regret lower bound.} For $x\notin\mathcal A_t(\mathcal F_t)$, the policy does not strictly prefer either arm, so it assigns probability $1/2$ to each arm. For parameter $\theta$, the difference between mean rewards at context $x$ is $|x-\theta|$. Therefore, conditional on $\mathcal F_t$ and $x$, posterior averaging gives the Bayes instantaneous regret under equal randomization as
\begin{equation}\label{eq:price-equal-randomization-regret}
\mathbb E[\operatorname{reg}_{t+1}\mid\mathcal F_t,x,\; x\notin\mathcal A_t(\mathcal F_t)]=\frac12\mathbb E(|x-\theta'|\mid\mathcal F_t).
\end{equation}
The next context $x_{t+1}$ is independent of $\mathcal F_t$ and uniform on $[-1,1]$, with density $1/2$. Hence, using \eqref{eq:price-equal-randomization-regret},
\[
\mathbb E_\Pi[\operatorname{reg}_{t+1}]
\geq\frac14\mathbb E_\Pi\left[\int_{\mathcal J_t(\mathcal F_t)\setminus\mathcal A_t(\mathcal F_t)}\mathbb E(|x-\theta'|\mid\mathcal F_t)\,dx\right].
\]
On $\mathcal G_t$, \eqref{eq:price-confusion-gap} gives $\mathbb E(|x-\theta'|\mid\mathcal F_t)\geq c_2r_t$ for $x\in\mathcal J_t$. Thus
\[
\mathbb E_\Pi[\operatorname{reg}_{t+1}]
\geq\frac{c_2r_t}{4}\mathbb E_\Pi[\operatorname{Leb}(\mathcal J_t(\mathcal F_t)\setminus\mathcal A_t(\mathcal F_t))]
=\frac{c_2r_t}{4}\{\mathbb E_\Pi\operatorname{Leb}(\mathcal J_t(\mathcal F_t))-\mathbb E_\Pi\operatorname{Leb}(\mathcal A_t(\mathcal F_t)\cap\mathcal J_t(\mathcal F_t))\}.
\]
By \eqref{eq:price-confusion-length} and \eqref{eq:price-good-history-constant}, $\mathbb E_\Pi\operatorname{Leb}(\mathcal J_t(\mathcal F_t))\geq c_1r_t\mathbb P_\Pi(\mathcal G_t)\geq c_1c_G r_t$. Combining with \eqref{eq:price-strict-set} yields
\begin{equation}\label{eq:price-one-round-lower}
\mathbb E_\Pi[\operatorname{reg}_{t+1}]\geq c_4r_t^2-c_5r_t\sqrt{\delta}
\end{equation}
for absolute constants $c_4,c_5>0$.

Recall $r_t=\sqrt{\log(1/\delta')/t}$. Since $\delta=\Theta(1/T)$ and $\delta'=\sqrt{\delta}$, $\log(1/\delta')=\Theta(\log T)$. Moreover, uniformly over $t\leq T$,
\[
\frac{\sqrt{\delta}}{r_t}=\sqrt{\frac{\delta t}{\log(1/\delta')}}\leq\sqrt{\frac{\delta T}{\log(1/\delta')}}=O(1/\sqrt{\log T})=o(1).
\]
Therefore, after increasing $T$ if necessary, \eqref{eq:price-one-round-lower} implies
\begin{equation}\label{eq:price-round-final}
\mathbb E_\Pi[\operatorname{reg}_{t+1}]\geq c_6\frac{\log(1/\delta')}{t}=\Omega\left(\frac{\log T}{t}\right)
\end{equation}
uniformly for $c_0\log(1/\delta')\leq t\leq T-1$. Summing \eqref{eq:price-round-final} over this range gives
\[
\mathbb E_\Pi[R_T]\geq c_6\log T\sum_{t=c_0\log T}^{T-1}\frac1t=\Omega(\log^2T).
\]
Finally, the Bayes risk under the uniform prior is bounded above by the worst-case risk: $\mathbb E_\Pi[R_T]\leq\sup_{\theta}\mathbb E_\theta[R_T]$. Hence $\sup_{\theta}\mathbb E_\theta[R_T]=\Omega(\log^2T)$, proving Theorem~\ref{thm:price}. \hfill\Halmos

\subsection{Proofs in Section \ref{sec:fairalgsmooth}}\label{ec_pf_nonpara}

\subsubsection{Supporting Lemmas and Propositions}

We first present several lemmas and propositions that are used in this section, whose proofs can be found in Section \ref{ec_pf_lemmasinfair_nonpara}.

\begin{lemma}\label{lem:gapawayfromoptimal}
Under Assumptions \ref{assump:iiddensity}--\ref{assump:c0r0regular},
for each $r>0$ there exists a constant $c_r>0$ such that, for all
$k\in\cK$ and $\bx\in\cX$, $\operatorname{dist}(\bx,\cR_k)\geq r$ implies $\Delta_k(\bx)\geq c_r.$
\end{lemma}

\begin{lemma}\label{lem:RksubsetSqk}
    Under event $\overline{\mathcal{G}}_{q-1}$, $\mathcal{R}_k\subseteq S_{q,k}$ for all $k\in\cK$.
\end{lemma}

\begin{lemma}\label{lem:epsilongreatthan1/2delta}   
    For sufficiently large $T$, $Q \leq\lceil \frac{\beta}{(2\beta+d)\log2}\log(T\left(\log T\right)^{-\frac{2\beta+d}{\beta^{\prime}-1}+\frac{2\beta+d}{2\beta}})\rceil$, and, for all $1\leq q\leq Q$,
    $\epsilon_q \geq \frac{1}{2}T^{-\frac{\beta}{2\beta+d}}\geq (1+2L_1\sqrt d)\delta_A$.
\end{lemma}

\begin{lemma}\label{lem:toosmallallsupport}
For sufficiently large $T$ and each $q=1,\ldots,Q$, on
$\overline{\mathcal G}_{q-1}$, $S_{q,k}=\cX$ for every arm $k\in\cK$
satisfying $\Delta_{k,\max}\leq T^{-\beta/(2\beta+d)}$.
\end{lemma}

\begin{lemma}\label{lem:A.6}
Fix $1\leq q\leq Q-1$, an arm $k\in\cK$, and an integer $n$ satisfying
\[
n\geq
\left(
\frac{6\sqrt{M_\beta}Lv_dp_{\max}}
{p^*\lambda_0\epsilon_q}
\right)^{\frac{2\beta+d}{\beta}}.
\]
Under Assumptions \ref{assump:iiddensity}--\ref{assump:c0r0regular},
suppose that $S_{q,k}$ is weakly
$(c_0/2^d,H_{q,k})$-regular at every $\bx\in S_{q,k}\cap G$.
Then, for sufficiently large $T$,
\begin{align*}
&\mathbb P\left(
\sup_{\bx\in S_{q,k}}
\left|\widehat f_{q,k}(\bx)-f_k^*(\bx)\right|
\geq\frac{\epsilon_q}{2}
\ \middle|\
\overline{\mathcal G}_{q-1},\overline{\mathcal M}_{q-1},
N_{q,k}=n
\right)\\
&\qquad\leq
\delta_A^{-d}(4+2M_\beta^2)
\exp\left(-C_Kn^{\frac{2\beta}{2\beta+d}}\epsilon_q^2\right).
\end{align*}
Here
\begin{align*}
\lambda_0
&=\frac14p_{\min}
\inf_{\substack{W\in\mathbb R^{M_\beta},\ \|W\|=1\\
S\subseteq\mathcal B(0,1)\text{ compact},\
\operatorname{Leb}(S)=c_0v_d/2^d}}
\int_S\left(\sum_{|s|\leq\mathfrak b(\beta)}W_su^s\right)^2du,\\
C_K
&=\frac{3p^*\lambda_0^2}{4(1+2L_1\sqrt d)^2}
\min\left\{
\frac{1}{6KM_\beta^4p_{\max}v_d+2p^*\lambda_0M_\beta^2},
\frac{1}{54KM_\beta v_dp_{\max}+6\sqrt{M_\beta}p^*\lambda_0},
\right.\\[-1mm]
&\hspace{49mm}\left.
\frac{1}{54KM_\beta L^2v_dp_{\max}
+6\sqrt{M_\beta}L(Kv_dp_{\max}+p^*)\lambda_0}
\right\}.
\end{align*}
\end{lemma}

\begin{lemma}\label{lem:chaining-regret}
Fix an epoch $q\geq2$, and let $\mathcal F_{q-1}^{\mathrm{ep}}$ denote the information collected before the beginning of epoch $q$. For each context $\bx$, let $\mathcal K_q(\bx)$ denote the active arm set from which the policy samples during epoch $q$, and let $\{\widehat f_{q-1,k}(\bx):k\in\mathcal K_q(\bx)\}$ denote the estimates used to construct this set. Let $\eta>0$.
Suppose that, for every $\bx$,
\begin{itemize}
    \item[(i)] at least one optimal arm belongs to $\mathcal K_q(\bx)$;
    \item[(ii)] $|\widehat f_{q-1,k}(\bx)-f_k^*(\bx)|\leq \eta/2$ for all $k\in\mathcal K_q(\bx)$;
    \item[(iii)] for each $k\in\mathcal K_q(\bx)$, $\widehat f_{q-1,k}(\bx)$ and $\max_{l\in\mathcal K_q(\bx)}\widehat f_{q-1,l}(\bx)$ are $(2\eta)$-chained in $\{\widehat f_{q-1,l}(\bx):l\in\mathcal K_q(\bx)\}$.
\end{itemize}
Then, conditional on any realization of $\mathcal F_{q-1}^{\mathrm{ep}}$ for which the above conditions hold, if the policy samples uniformly from $\mathcal K_q(\bx_t)$, Assumption~\ref{assum:margin} implies that, for every $t\in\mathcal T_q$,
\[
\mathbb E\left[\max_{k\in\cK}f_k^*(\bx_t)-f_{\pi_t}^*(\bx_t)\mid\mathcal F_{q-1}^{\mathrm{ep}}\right]\lesssim \eta^{1+\alpha}.
\]
\end{lemma}

\textit{Proof.}
Fix $t\in\mathcal T_q$. Let $i$ be an optimal arm at $\bx_t$ belonging to $\mathcal K_q(\bx_t)$, whose existence follows from condition (i), and let $k=\pi_t\in\mathcal K_q(\bx_t)$. By condition (iii), both $\widehat f_{q-1,i}(\bx_t)$ and $\widehat f_{q-1,k}(\bx_t)$ are in the same $(2\eta)$-chained component as $\max_{l\in\mathcal K_q(\bx_t)}\widehat f_{q-1,l}(\bx_t)$. Hence they belong to the same $(2\eta)$-chained component. There therefore exists a simple chain $a_0,a_1,\ldots,a_r$, with $a_0=i$, $a_r=k$, and $r\leq K-1$ such that
\[
|\widehat f_{q-1,a_{s-1}}(\bx_t)-\widehat f_{q-1,a_s}(\bx_t)|\leq2\eta,\qquad s=1,\ldots,r.
\]
Together with condition (ii), this gives
\begin{align}
0&\leq f_i^*(\bx_t)-f_k^*(\bx_t)\nonumber\\
&\leq |f_i^*(\bx_t)-\widehat f_{q-1,i}(\bx_t)|
+|\widehat f_{q-1,i}(\bx_t)-\widehat f_{q-1,k}(\bx_t)|
+|\widehat f_{q-1,k}(\bx_t)-f_k^*(\bx_t)|\nonumber\\
&\leq \eta/2+2(K-1)\eta+\eta/2\leq2K\eta.
\label{eq:fistarbt-fpitbtleq}
\end{align}

Let $\Delta_t:=\max_{k\in\cK}f_k^*(\bx_t)-f_{\pi_t}^*(\bx_t)$. If $\Delta_t>0$, then the optimal arm $i$ and the selected arm $k=\pi_t$ satisfy $0<|f_i^*(\bx_t)-f_k^*(\bx_t)|\leq2K\eta$. Consequently,
\[
\Delta_t\leq 2K\eta\,\mathbb I\{\exists l\neq j:0<|f_l^*(\bx_t)-f_j^*(\bx_t)|\leq2K\eta\}.
\]
Since $\bx_t$ is independent of $\mathcal F_{q-1}^{\mathrm{ep}}$, Assumption~\ref{assum:margin} and a union bound over arm pairs yield
\begin{align*}
\mathbb E[\Delta_t\mid\mathcal F_{q-1}^{\mathrm{ep}}]
&\leq 2K\eta\sum_{1\leq l<j\leq K}\PP(0<|f_l^*(\bx_t)-f_j^*(\bx_t)|\leq2K\eta)\\
&\leq 2K\eta\binom K2 C_0(2K\eta)^\alpha\lesssim \eta^{1+\alpha}.
\end{align*}
\hfill\Halmos

\begin{proposition}\label{thm:thm3inhusmooth}
Fix any positive parameters $\alpha, \beta, d, L, L_1$ satisfying $\alpha \beta \leq d$. For any admissible policy $\pi$ and $T$, there exists a contextual bandit instance satisfying Assumptions \ref{assump:iiddensity}--\ref{assump:Qxliplower} with $\rho=T^{-\frac{\beta}{2\beta+d}}$, and with the provided parameters such that
$$
\sup _{\mathbb{P} \in \mathcal{P}} R_T(\pi)=\Omega\left(T^{\frac{\beta+d-\alpha \beta}{2 \beta+d}}\right),
$$
where the $\Omega(\cdot)$ term only depends on the parameters of the class $\mathcal{P}$ and not on $\pi$.
\end{proposition}

\subsubsection{Proof of Proposition \ref{thm:smoothbanditfair}}
First we need to clarify the notations used in \cite{hu2022smooth} which is different from ours. In this proof, the two arms are denoted as arm 1 and arm -1 respectively and the epochs are from 1 to $K$. Moreover, $\mathcal{E}_{+1, j}$ refers to the exploitation region of arm 1 at epoch $j$, and  similarly $\mathcal{E}_{-1, j}$ refers to the exploitation region of arm -1 at epoch $j$. The exploitation region of arm $i$ means that only arm $i$ would be chosen when context falls on this region.

When $t\in\mathcal{T}_1$, the algorithm pulls two arms with equal probability, which implies that $\PP(\pi_t=1)=\PP(\pi_t=2)=1/2.$ Thus, fairness condition is satisfied trivially. After that, at epoch $k>1$, the algorithm pulls arm $a$ when $\bx_t\in \bigcup_{j=1}^k \mathcal{E}_{a, j}$ and randomly pulls arm with equal probability if $\bx_t\in \cX\setminus\left(\bigcup_{a\in \{-1,+1\}}\bigcup_{j=1}^k \mathcal{E}_{a, j}\right)$. The fairness condition is obviously satisfied in the latter case. 

Then we focus on the scenario that $\bx_t\in \bigcup_{j=1}^k \mathcal{E}_{a, j}$. By statement(ii) of Lemma 5 in \cite{hu2022smooth}, under certain event $\overline{\mathcal{G}}_{k-1} \cap \overline{\mathcal{M}}_{k-1}$, $\left(\bigcup_{j=1}^k \mathcal{E}_{a, j}\right) \cap \cX \subseteq\{\bx \in \cX: f_a^*(\bx)>f_{-a}^*(\bx)\}$, which validates the fairness condition. Thus, under certain event $\overline{\mathcal{G}}_{k-1} \cap \overline{\mathcal{M}}_{k-1}$, the fairness condition is satisfied for epoch $k$ by combining both cases. According to Theorem 1 in \cite{hu2022smooth}, if $T$ is large enough, $\mathbb{P}\left(\overline{\mathcal{G}}_{k-1} \cap \overline{\mathcal{M}}_{k-1}\right) \geq 1-\frac{c(k-1)}{T}$ where $c$ is some positive constant. By taking union bound over all epochs from 2 to $K$, the fairness condition is satisfied with probability $\mathbb{P}\left( \bigcap_{k=2}^{K} (\overline{\mathcal{G}}_{k-1} \cap \overline{\mathcal{M}}_{k-1})\right) =\mathbb{P}\left( \overline{\mathcal{G}}_{K-1} \cap \overline{\mathcal{M}}_{K-1})\right)\geq 1-\frac{c(K-1)}{T}=1-\tilde{O}(1/T)$ since Lemma 2 in \cite{hu2022smooth} gives $K=O(\log T)$.\hfill\Halmos

\subsubsection{Proof of Proposition \ref{prop:explainassump3.8}}

Fix $k\in\cK$ and
$\bx\in\cX\setminus\cR_k$ with
$\operatorname{dist}(\bx,\cR_k)\leq r$. Choose $\bx_0$ and $i$ as in the
proposition, and write
\[
v:=\frac{\bx-\bx_0}{\|\bx-\bx_0\|_2},
\qquad
g:=f_i^*-f_k^*.
\]
Because $f_i^*(\bx_0)=f_k^*(\bx_0)$, we have $g(\bx_0)=0$. The finitely
many gradients $\nabla(f_i^*-f_k^*)$ are uniformly continuous on a compact
neighborhood of $\cX$. Hence, after decreasing $r$ if necessary, the segment
$\{\bx_0+sv:0\leq s\leq\|\bx-\bx_0\|_2\}$ lies in the neighborhood on
which the extensions are defined and
\[
\nabla g(\bx_0+sv)^{\mathrm T}v\geq a/2,
\qquad
0\leq s\leq\|\bx-\bx_0\|_2.
\]
It follows that
\begin{align*}
\Delta_k(\bx)
&\geq f_i^*(\bx)-f_k^*(\bx)\\
&=\int_0^{\|\bx-\bx_0\|_2}
\nabla g(\bx_0+sv)^{\mathrm T}v\,ds\\
&\geq \frac a2\|\bx-\bx_0\|_2
=\frac a2\operatorname{dist}(\bx,\cR_k).
\end{align*}
Therefore,
\[
\operatorname{dist}(\bx,\cR_k)
\leq\frac{2}{a}\Delta_k(\bx).
\]
This verifies the boundary-growth alternative in
Assumption~\ref{assump:Qxliplower} with
$\beta'=\beta$, $\widetilde c=2/a$, and the chosen radius $r$. Since the
argument is uniform over $k$ and does not restrict the number of arms tied at
$\bx_0$, the proposition follows. \hfill\Halmos

\subsubsection{Proof of Proposition \ref{prop:sqkregular}}

The case $q=1$ follows because $S_{1,k}=\cX$ for all arms. On
$\mathcal M_1$, the sample-size lower bound gives $H_{1,k}\leq r_0$ for
sufficiently large $T$. Any $\bx\in\cX$ belongs to the optimal region
of at least one arm, and that regular optimal region is contained in $\cX$.

Suppose $q\geq2$ and fix $k\in\cK$. On $\mathcal M_q$,
$H_{q,k}\leq r_0$ for sufficiently large $T$.
If $\Delta_{k,\max}\leq T^{-\beta/(2\beta+d)}$,
Lemma~\ref{lem:toosmallallsupport} gives $S_{q,k}=\cX$, and the
same regularity argument applies. Otherwise,
Assumption~\ref{assump:Qxliplower} supplies boundary growth for this arm.

Fix $\widetilde\bx\in S_{q,k}\cap G$.
Lemma~\ref{lem:RksubsetSqk} gives $\cR_k\subseteq S_{q,k}$.
If $\widetilde\bx\in\cR_k$, its regularity proves the claim.
It remains to consider $\widetilde\bx\notin\cR_k$.

Choose an optimal arm
$i\in\arg\max_{j\in\cK}f_j^*(\widetilde\bx)$. Both $i$ and $k$ belong to
the active set at $\widetilde\bx$: arm $k$ does so by the definition of
$S_{q,k}$, while arm $i$ does so by Lemma~\ref{lem:RksubsetSqk}. On
$\overline{\mathcal G}_{q-1}$, the active-set update and the uniform
estimation bound therefore yield
\begin{align}\label{eq:active-gap-boundary-growth}
\Delta_k(\widetilde\bx)
&=f_i^*(\widetilde\bx)-f_k^*(\widetilde\bx)\nonumber\\
&\leq
|f_i^*(\widetilde\bx)-\widehat f_{q-1,i}(\widetilde\bx)|
+|\widehat f_{q-1,i}(\widetilde\bx)
  -\widehat f_{q-1,k}(\widetilde\bx)|
+|\widehat f_{q-1,k}(\widetilde\bx)-f_k^*(\widetilde\bx)|\nonumber\\
&\leq
\frac{\epsilon_{q-1}}2+2(K-1)\epsilon_{q-1}
+\frac{\epsilon_{q-1}}2
\leq 2K\epsilon_{q-1}.
\end{align}

For this arm, let
$c_{\mathfrak r,k}:=\inf\{\Delta_k(\bx):\bx\in\cX,
\operatorname{dist}(\bx,\cR_k)\geq\mathfrak r\}$,
with $c_{\mathfrak r,k}=1$ if the set is empty.
Lemma~\ref{lem:gapawayfromoptimal} gives $c_{\mathfrak r,k}>0$.
Since $\epsilon_q$ decreases and $\epsilon_1\to0$, for sufficiently
large $T$, $2K\epsilon_{q-1}<c_{\mathfrak r,k}$.
Equation~\eqref{eq:active-gap-boundary-growth} then gives
$\operatorname{dist}(\widetilde\bx,\cR_k)<\mathfrak r$.
Choose a nearest point $\bx_0\in\cR_k$ and write
$d_q:=\|\widetilde\bx-\bx_0\|_2$.
Assumption~\ref{assump:Qxliplower} gives
\begin{align}\label{eq:boundary-distance-clean}
d_q
\leq
\widetilde c\,\Delta_k(\widetilde\bx)^{\beta'/\beta}
\leq
\widetilde c(2K\epsilon_{q-1})^{\beta'/\beta}.
\end{align}

On $\mathcal M_q$, $H_{q,k}=N_{q,k}^{-1/(2\beta+d)}$ and
$N_{q,k}\leq|\mathcal T_q|\leq\bar n_q$. Hence, using the definitions of
$\bar n_q$ and $\epsilon_{q-1}$,
\begin{align}\label{eq:boundary-bandwidth-clean}
\frac{d_q}{H_{q,k}}
&\lesssim
\epsilon_{q-1}^{\beta'/\beta}
\bar n_q^{1/(2\beta+d)}\nonumber\\
&\lesssim
(\log T)^{\frac{\beta'-2\beta}{2\beta}}
2^{q(1-\beta')/\beta}
=o(1),
\end{align}
where the last relation uses $1<\beta'\leq\beta$. Thus, for sufficiently
large $T$, $d_q\leq H_{q,k}/2$.

The ball
$\mathcal B(\bx_0,H_{q,k}-d_q)$ is contained in
$\mathcal B(\widetilde\bx,H_{q,k})$. Since
$\bx_0\in\cR_k$, $\cR_k\subseteq S_{q,k}$, and
$H_{q,k}-d_q\leq r_0$, Assumption~\ref{assump:c0r0regular} gives
\begin{align*}
\operatorname{Leb}
\bigl[S_{q,k}\cap\mathcal B(\widetilde\bx,H_{q,k})\bigr]
&\geq
\operatorname{Leb}
\bigl[\cR_k\cap\mathcal B(\bx_0,H_{q,k}-d_q)\bigr]\\
&\geq
c_0\operatorname{Leb}
\bigl[\mathcal B(\bx_0,H_{q,k}-d_q)\bigr]\\
&\geq
\frac{c_0}{2^d}
\operatorname{Leb}
\bigl[\mathcal B(\widetilde\bx,H_{q,k})\bigr].
\end{align*}
Therefore, $S_{q,k}$ is weakly
$(c_0/2^d,H_{q,k})$-regular at $\widetilde\bx$. Since $k$ and
$\widetilde\bx$ were arbitrary, the proposition follows.

\hfill\Halmos

\subsubsection{Proof of Proposition \ref{thm:tailproboftwoevents}}

Recall that 
\begin{align*}
&\mathcal{M}_q  =\left\{\min _{k \in \mathcal{K}} N_{q, k} \geq\left(\frac{4^q\log(T\delta_A^{-d})}{C_{K}}\right)^{\frac{2 \beta+d}{2 \beta}}\left(\log T\right)^{(2\beta+d)\frac{2\beta-\beta^{\prime}+1}{2\beta(\beta^{\prime}-1)}}\right\}, \\
&\mathcal{G}_q = \left\{
\begin{aligned}
& \text{(i) } \text { for all }k\in\cK, S_{q,k} \text{ is weakly }  (\frac{c_0}{2^d},H_{q,k})\text{-regular at all }\bx\in S_{q,k}\cap G \\
& \text{(ii) } \left|\hat{f}_{q, k}(\bx)-f^*_{ k}(\bx)\right| \leq \epsilon_q / 2 \text { for all } \bx \in S_{q,k} \text{ and } k\in\cK 
\end{aligned}
\right\}.
\end{align*}
We first show the following two inequalities:
\begin{align}
     &\mathbb{P}\left(\mathcal{G}_q^C \mid \overline{\mathcal{G}}_{q-1}, \overline{\mathcal{M}}_q\right) \leq \frac{\left(4+2 M_\beta^2\right)K}{T},\label{eq:firsteq}\\ 
     &\mathbb{P}\left(\mathcal{M}_q^C \mid \overline{\mathcal{G}}_{q-1}, \overline{\mathcal{M}}_{q-1}\right) \leq \frac{K}{T}.\label{eq:secondeq}
\end{align}

\noindent\textit{Proof of \eqref{eq:firsteq}}.
On $\mathcal M_q$, $N_{q,k}\geq
\left(\frac{4^q\log(T\delta_A^{-d})}{C_K}\right)^{\frac{2\beta+d}{2\beta}}
(\log T)^{(2\beta+d)\frac{2\beta-\beta'+1}{2\beta(\beta'-1)}}$.
for each $k\in\cK$. For sufficiently large $T$, this lower bound also
satisfies the sample-size condition in Lemma \ref{lem:A.6}.
Proposition \ref{prop:sqkregular} establishes part (i) of
$\mathcal G_q$ on
$\overline{\mathcal G}_{q-1}\cap\overline{\mathcal M}_q$.
For a fixed arm $k$, condition further on $N_{q,k}=n$ and apply Lemma
\ref{lem:A.6}. Since any admissible $n$ is at least the displayed
lower bound, we have that
\begin{align*}
\delta_A^{-d}(4+2M_\beta^2)
\exp\left(-C_Kn^{\frac{2\beta}{2\beta+d}}\epsilon_q^2\right)\leq
\delta_A^{-d}(4+2M_\beta^2)
\exp\left(-\log(T\delta_A^{-d})\right)
=\frac{4+2M_\beta^2}{T}.
\end{align*}
Averaging over $N_{q,k}$ and taking a union bound over $k\in\cK$ proves
\eqref{eq:firsteq}.

\noindent\textit{Proof of \eqref{eq:secondeq}}. Lemma \ref{lem:RksubsetSqk} implies that $\mathcal{R}_k\subseteq S_{q,k}$ for all $k\in\cK$. Together with Assumptions \ref{assump:iiddensity} and \ref{assump:c0r0regular}, we have
\begin{align*}
\mathbb{E}\left(N_{q, k} \mid \overline{\mathcal{G}}_{q-1}, \overline{\mathcal{M}}_{q-1}\right) & =\mathbb{E}\left(\sum_{t \in \mathcal{T}_q} \mathbb{I}\left\{\pi_t =k \right\}\mid \overline{\mathcal{G}}_{q-1}, \overline{\mathcal{M}}_{q-1}\right) \\
& =\sum_{t \in \mathcal{T}_q} \mathbb{P}\left(\pi_t =k  \mid \overline{\mathcal{G}}_{q-1}, \overline{\mathcal{M}}_{q-1}\right) \nonumber\\
& \geq \sum_{t \in \mathcal{T}_q} \mathbb{P}\left(\pi_t =k, \bx_t\in S_{q,k} \mid \overline{\mathcal{G}}_{q-1}, \overline{\mathcal{M}}_{q-1}\right) \nonumber\\
& \geq \frac{1}{K} \sum_{t \in \mathcal{T}_q} \mathbb{P}\left(\bx_t \in \mathcal{R}_k \mid \overline{\mathcal{G}}_{q-1}, \overline{\mathcal{M}}_{q-1}\right) \geq \frac{p^*}{K} |\mathcal{T}_q| .
\end{align*}

Conditional on the history before epoch $q$, the active sets are fixed and the contexts within the epoch are independent. Hoeffding's inequality therefore gives the following bound.
\begin{align*}
& \mathbb{P}\left(N_{q, k}<\left(\frac{4^q\log \left(T \delta_A^{-d}\right)}{C_K }\right)^{\frac{2 \beta+d}{2 \beta}}\left(\log T\right)^{(2\beta+d)\frac{2\beta-\beta^{\prime}+1}{2\beta(\beta^{\prime}-1)}}\bigg|~ \overline{\mathcal{G}}_{q-1}, \overline{\mathcal{M}}_{q-1},\{\bigcup_{k\in\cK}\cT_{h,k}\}_{h=1}^{q-1}\right) \nonumber\\
\leq & 
\mathbb{P}\bigg(\mathbb{E}\left(N_{q, k} \mid \overline{\mathcal{G}}_{k-1}, \overline{\mathcal{M}}_{k-1}\right)-N_{q, k}>\frac{p^*}{K} |\mathcal{T}_q|\\
&\quad -
\left(\frac{4^q\log \left(T \delta_A^{-d}\right)}{C_K }\right)^{\frac{2 \beta+d}{2 \beta}}\left(\log T\right)^{(2\beta+d)\frac{2\beta-\beta^{\prime}+1}{2\beta(\beta^{\prime}-1)}} \bigg|~ \overline{\mathcal{G}}_{q-1}, \overline{\mathcal{M}}_{q-1},\{\bigcup_{k\in\cK}\cT_{h,k}\}_{h=1}^{q-1}\bigg) \nonumber\\
\leq & \exp \left(-\frac{2}{|\mathcal{T}_q|}\left[\frac{p^*}{K} |\mathcal{T}_q|-\left(\frac{4^q\log \left(T \delta_A^{-d}\right)}{C_K}\right)^{\frac{2 \beta+d}{2 \beta}}\left(\log T\right)^{(2\beta+d)\frac{2\beta-\beta^{\prime}+1}{2\beta(\beta^{\prime}-1)}}\right]^2\right)  \nonumber\\
\leq &  \exp \left(-\frac{2}{|\mathcal{T}_q|}\left(\frac{p^{*2}}{K^2}|\mathcal{T}_q|^2-\frac{2p^*}{K} |\mathcal{T}_q|\left(\frac{4^q\log \left(T \delta_A^{-d}\right)}{C_K}\right)^{\frac{2 \beta+d}{2 \beta}}\left(\log T\right)^{(2\beta+d)\frac{2\beta-\beta^{\prime}+1}{2\beta(\beta^{\prime}-1)}}\right)\right)\nonumber\\
=& \exp \left(-\frac{2p^{*2}}{K^2}|\mathcal{T}_q|+\frac{4p^*}{K} \left(\frac{4^q\log \left(T \delta_A^{-d}\right)}{C_K}\right)^{\frac{2 \beta+d}{2 \beta}}\left(\log T\right)^{(2\beta+d)\frac{2\beta-\beta^{\prime}+1}{2\beta(\beta^{\prime}-1)}}\right).
\end{align*}
Since 
\begin{align*}
   |\cT_q|&=\left\lceil\frac{2K}{p^*}\left(\frac{4^q\log \left(T \delta_A^{-d}\right)}{C_{K}}\right)^{\frac{2 \beta+d}{2 \beta}}\left(\log T\right)^{\frac{2\beta+d}{\beta^{\prime}-1}-\frac{2\beta+d}{2\beta}}+\frac{K^2}{2 p^{*2}} \log T\right\rceil\nonumber\\
   &>\frac{2K}{p^*}\left(\frac{4^q\log \left(T \delta_A^{-d}\right)}{C_{K}}\right)^{\frac{2 \beta+d}{2 \beta}}\left(\log T\right)^{(2\beta+d)\frac{2\beta-\beta^{\prime}+1}{2\beta(\beta^{\prime}-1)}}+\frac{K^2}{2 p^{*2}} \log T ,
\end{align*}
we have
\begin{align*}
    \frac{2p^{*2}}{K^2}|\mathcal{T}_q|-\frac{4p^*}{K} \left(\frac{4^q\log \left(T \delta_A^{-d}\right)}{C_K}\right)^{\frac{2 \beta+d}{2 \beta}}\left(\log T\right)^{(2\beta+d)\frac{2\beta-\beta^{\prime}+1}{2\beta(\beta^{\prime}-1)}}\geq \log T,
\end{align*}
which implies
\begin{align*}
    \mathbb{P}\left(N_{q, k}<\left(\frac{4^q\log \left(T \delta_A^{-d}\right)}{C_K }\right)^{\frac{2 \beta+d}{2 \beta}}\left(\log T\right)^{(2\beta+d)\frac{2\beta-\beta^{\prime}+1}{2\beta(\beta^{\prime}-1)}}\bigg|~ \overline{\mathcal{G}}_{q-1}, \overline{\mathcal{M}}_{q-1},\{\bigcup_{k\in\cK}\cT_{h,k}\}_{h=1}^{q-1}\right)\leq \frac{1}{T}.
\end{align*}
Averaging over the pre-epoch history and taking a union bound over the arms gives
\begin{align*}
    \mathbb{P}\left(\min_{k\in \cK}N_{q, k}<\left(\frac{4^q\log \left(T \delta_A^{-d}\right)}{C_K }\right)^{\frac{2 \beta+d}{2 \beta}}\left(\log T\right)^{(2\beta+d)\frac{2\beta-\beta^{\prime}+1}{2\beta(\beta^{\prime}-1)}}\bigg|~ \overline{\mathcal{G}}_{q-1}, \overline{\mathcal{M}}_{q-1}\right)\leq \frac{K}{T}.
\end{align*}

\noindent\textit{Proof of \eqref{eq:thirdeq} in Proposition \ref{thm:tailproboftwoevents}}. It can be verified that
\begin{align*}
    \overline{\mathcal{G}}_q^C=\cup_{j=1}^q \mathcal{G}_j^C&\subseteq \bigcup_{j=0}^{q-1} (\mathcal{G}_{j+1}^C\cap \overline{\mathcal{G}}_j)\subseteq\bigcup_{j=0}^{q-1} (\mathcal{G}_{j+1}^C\cap\overline{\mathcal{M}}_{j+1} \cap \overline{\mathcal{G}}_j) \cup (\mathcal{G}_{j+1}^C\cap\overline{\mathcal{M}}_{j+1}^C \cap \overline{\mathcal{G}}_j)\nonumber\\
    &\subseteq\bigcup_{j=0}^{q-1} (\mathcal{G}_{j+1}^C\cap\overline{\mathcal{M}}_{j+1} \cap \overline{\mathcal{G}}_j)\cup (\overline{\mathcal{M}}_{j+1}^C \cap \overline{\mathcal{G}}_j)\nonumber\\
    &\subseteq \bigcup_{j=0}^{q-1} \left(\mathcal{G}_{j+1}^C\cap\overline{\mathcal{M}}_{j+1} \cap \overline{\mathcal{G}}_j)\cup (\bigcup_{i=0}^{j}{\mathcal{M}}_{i+1}^C \cap\overline{\mathcal{M}}_{i}\cap \overline{\mathcal{G}}_j)\right)\nonumber\\
    &=\left(\bigcup_{j=0}^{q-1} \mathcal{G}_{j+1}^C \cap \overline{\mathcal{M}}_{j+1} \cap \overline{\mathcal{G}}_j\right) \cup\left(\bigcup_{j=0}^{q-1} \mathcal{M}_{j+1}^C \cap \overline{\mathcal{G}}_j \cap \overline{\mathcal{M}}_j\right) .   
\end{align*}
Similarly,
\begin{align*}
    \overline{\mathcal{M}}_q^C=\cup_{j=1}^q \mathcal{M}_j^C&\subseteq \bigcup_{j=0}^{q-1} (\mathcal{M}_{j+1}^C\cap \overline{\mathcal{M}}_j)\subseteq\bigcup_{j=0}^{q-1} (\mathcal{M}_{j+1}^C\cap\overline{\mathcal{G}}_{j} \cap \overline{\mathcal{M}}_j) \cup (\mathcal{M}_{j+1}^C\cap\overline{\mathcal{G}}_{j}^C \cap \overline{\mathcal{M}}_j)\nonumber\\
    &\subseteq\bigcup_{j=0}^{q-1} (\mathcal{M}_{j+1}^C\cap\overline{\mathcal{G}}_{j} \cap \overline{\mathcal{M}}_j)\cup (\overline{\mathcal{G}}_{j}^C \cap \overline{\mathcal{M}}_j)\nonumber\\
    &\subseteq \bigcup_{j=0}^{q-1}\left( (\mathcal{M}_{j+1}^C\cap\overline{\mathcal{G}}_{j} \cap \overline{\mathcal{M}}_j)\cup (\bigcup_{i=0}^{j-1}{\mathcal{G}}_{i+1}^C \cap\overline{\mathcal{M}}_{j+1}\cap \overline{\mathcal{G}}_{i })\right)\nonumber\\
    &\subseteq \bigcup_{j=0}^{q-1} (\mathcal{M}_{j+1}^C\cap\overline{\mathcal{G}}_{j} \cap \overline{\mathcal{M}}_j)\cup\left(\bigcup_{j=0}^{q-1} \mathcal{G}_{j+1}^C \cap \overline{\mathcal{M}}_{j+1} \cap \overline{\mathcal{G}}_{j}\right) .   
\end{align*}
Therefore, we can obtain that
\begin{align*}
    \overline{\mathcal{G}}_k^C \cup \overline{\mathcal{M}}_k^C \subseteq\left(\bigcup_{j=0}^{k-1} \mathcal{G}_{j+1}^C \cap \overline{\mathcal{M}}_{j+1} \cap \overline{\mathcal{G}}_j\right) \cup\left(\bigcup_{j=0}^{k-1} \mathcal{M}_{j+1}^C \cap \overline{\mathcal{G}}_j \cap \overline{\mathcal{M}}_j\right) .
\end{align*}

With \eqref{eq:firsteq} and \eqref{eq:secondeq}, we can conclude that
\begin{align*}
\mathbb{P}\left(\overline{\mathcal{G}}_q^C \cup \overline{\mathcal{M}}_q^C\right) & \leq \sum_{j=0}^{q-1} \mathbb{P}\left(\mathcal{G}_{j+1}^C \cap \overline{\mathcal{M}}_{j+1} \cap \overline{\mathcal{G}}_j\right)+\sum_{j=0}^{q-1} \mathbb{P}\left(\mathcal{M}_{j+1}^C \cap \overline{\mathcal{G}}_j \cap \overline{\mathcal{M}}_j\right) \nonumber\\
& \leq \sum_{j=0}^{q-1} \mathbb{P}\left(\mathcal{G}_{j+1}^C \mid \overline{\mathcal{G}}_j, \overline{\mathcal{M}}_{j+1}\right)+\sum_{j=0}^{q-1} \mathbb{P}\left(\mathcal{M}_{j+1}^C \mid \overline{\mathcal{G}}_j, \overline{\mathcal{M}}_j\right)\nonumber \\
& \leq \sum_{j=0}^{q-1} \frac{\left(4+2 M_\beta^2\right)K}{T}+\frac{K}{T}=\frac{\left(5+2 M_\beta^2\right) qK}{T} .
\end{align*}
This finishes the proof.\hfill\Halmos

\subsubsection{Proof of Theorem \ref{thm:smoothbanditfairness}}
When $t\in\mathcal{T}_1$, the algorithm pulls all arms with equal probability, which implies that for all $\bx\in\cX$, $\PP(\pi_t=i|\cF_{t-1},\bx_t=\bx)=1/K$ for all $i\in\cK$. Thus, fairness condition is satisfied trivially at $t$. 

At epoch $q>1$, when $\bx_t=\bx$, the algorithm randomly pulls arms from set $\cK_{q,u(\bx)}$ with equal probability. For the pair of arms $i,j\in \cK_{q,u(\bx)}$, $\PP(\pi_t=i|\cF_{t-1},\bx_t=\bx)=\PP(\pi_t=j|\cF_{t-1},\bx_t=\bx)\geq 1/K$. On the other hand, when $i,j\in \cK\setminus\cK_{q,u(\bx)}$, $\PP(\pi_t=i|\cF_{t-1},\bx_t=\bx)=\PP(\pi_t=j|\cF_{t-1},\bx_t=\bx)=0$. The former two cases automatically satisfy fairness constraint \eqref{eq:defoffairness}. Therefore, it remains to consider the arm pairs with $i\in \cK_{q,u(\bx)}$ and $j\in \cK\setminus\cK_{q,u(\bx)}$, which satisfies $\PP(\pi_t=i|\cF_{t-1},\bx_t=\bx)\geq \frac{1}{K}>\PP(\pi_t=j|\cF_{t-1},\bx_t=\bx)=0$. For $j\in \cK\setminus\cK_{q,u(\bx)}$, there exists an epoch $1\leq l\leq q-1$ such that $j\in \cK_{l,u(\bx)}$ and $j\in \cK\setminus\cK_{l+1,u(\bx)}$. Moreover, the selection rule implies that $i\in \cK_{q,u(\bx)}\subseteq \cK_{l,u(\bx)}$ when $l<q$. Thus, $\hat{f}_{l,j} (\bx)$ is not $(2\epsilon_l)$-chained to $\hat{f}_{l,i} (\bx)$ in $\{\hat{f}_{l,k} (\bx):k\in \cK_{l,u(\bx)}\}$, which gives
\begin{align*}
   \hat{f}_{l,i}\left(\bx\right)-\hat{f}_{l,j}\left(\bx\right)>2\epsilon_{l}.
\end{align*}

Under event $\overline{\mathcal{G}}_{q-1}\cap\overline{\mathcal{M}}_{q-1}$,
\begin{align*}
    f_i^*\left(\bx\right)-f_j^*\left(\bx\right)&= f_i^*\left(\bx\right)-\hat{f}_{l,i}\left(\bx\right)+\hat{f}_{l,i}\left(\bx\right)-\hat{f}_{l,j}\left(\bx\right)+\hat{f}_{l,j}\left(\bx\right)-f_j^*\left(\bx\right)\nonumber\\
    &>   -\frac{1}{2}\epsilon_l+2\epsilon_l - \frac{1}{2}\epsilon_l\nonumber\\
    &>0,
\end{align*} which validates the fairness condition \eqref{eq:defoffairness}. Thus, under event $\overline{\mathcal{G}}_{q-1} \cap \overline{\mathcal{M}}_{q-1}$, the fairness condition is satisfied for epoch $q$ by combining all three cases. 

Since the events are nested (i.e., $\overline{\mathcal{G}}_{q-1} \cap \overline{\mathcal{M}}_{q-1}$ implies $\overline{\mathcal{G}}_{o} \cap \overline{\mathcal{M}}_{o}$ for all $o\leq q-1$), it follows that the fairness condition holds for epochs 1 to $q$ by induction on earlier epochs. Then under event $\overline{\mathcal{G}}_{Q-1} \cap \overline{\mathcal{M}}_{Q-1}$, the fairness condition is satisfied for all epochs from $1$ to $Q$.

According to Proposition \ref{thm:tailproboftwoevents},
\[
\mathbb{P}\left(\overline{\mathcal{G}}_{Q-1} \cap \overline{\mathcal{M}}_{Q-1}\right)
\geq 1-\frac{K(5+2M_{\beta}^2)(Q-1)}{T}.
\]
Hence, the fairness condition is satisfied with probability at least $1-\frac{K(5+2M_{\beta}^2)(Q-1)}{T}\geq 1-\tilde{O}(1/T)$ since Lemma \ref{lem:epsilongreatthan1/2delta} gives $Q=O(\log T)$.

\hfill\Halmos

\subsubsection{Proof of Theorem \ref{thm:minimax}}

Without loss of generality, we assume $\max_{k\in\cK}\|f_k^*\|_{\infty}\leq 1 $. We bound the first epoch directly and decompose the remaining regret according to the good event:
\begin{align}\label{eq_regret_fairsmoothdecomp}
    R_T\leq 2|\cT_1|
    &+\sum_{q=2}^Q\sum_{t\in\mathcal{T}_q}
    \mathbb{E}\left(\max_{k\in\cK}f_k^*\left(\bx_t\right)
    -f_{\pi_t}^*\left(\bx_t\right)
    \mid \overline{\mathcal{G}}_{q-1}\cap\overline{\mathcal{M}}_{q-1}\right)\nonumber\\
    &+2\sum_{q=2}^Q\sum_{t\in\mathcal{T}_q}
    \PP(\overline{\mathcal{G}}_{q-1}^C
    \cup\overline{\mathcal{M}}_{q-1}^C),
\end{align}
For each $q\geq2$, on
$\overline{\mathcal G}_{q-1}\cap\overline{\mathcal M}_{q-1}$,
Lemma \ref{lem:RksubsetSqk} and part (ii) of $\mathcal G_{q-1}$ verify
the conditions of Lemma \ref{lem:chaining-regret} with
$\eta=\epsilon_{q-1}$. Hence
\begin{align}\label{eq_regret_fairsmoothcomp1}
&\sum_{q=2}^Q\sum_{t\in\mathcal T_q}
\mathbb E\left(
\max_{k\in\cK}f_k^*(\bx_t)-f_{\pi_t}^*(\bx_t)
\ \middle|\
\overline{\mathcal G}_{q-1}\cap\overline{\mathcal M}_{q-1}
\right)\\
&\qquad\lesssim
\sum_{q=2}^Q|\mathcal T_q|\epsilon_{q-1}^{1+\alpha}
=\widetilde O\left(
T^{\frac{\beta+d-\alpha\beta}{2\beta+d}}
\right).
\end{align}

By Proposition \ref{thm:tailproboftwoevents}, the second component can be upper bounded by
\begin{align}\label{eq_regret_fairsmoothcomp2}
2\sum_{q=2}^Q\sum_{t\in\mathcal{T}_q}\PP(\overline{\mathcal{G}}_{q-1}^C\cup\overline{\mathcal{M}}_{q-1}^C)\lesssim \sum_{q=2}^Q\sum_{t\in\mathcal{T}_q} \frac{(q-1)}{T}=\tilde{O}(1).
\end{align}
Since $|\cT_1|=\widetilde O(1)$, plugging \eqref{eq_regret_fairsmoothcomp1} and \eqref{eq_regret_fairsmoothcomp2} into \eqref{eq_regret_fairsmoothdecomp} gives
\begin{align*}
    R_T = \tilde{O}(T^{\frac{\beta+d-\alpha \beta}{2\beta+d}}).
\end{align*}
When $\alpha\beta\leq d$, this rate matches the lower bound in Proposition \ref{thm:thm3inhusmooth}.
Thus, Algorithm \ref{alg:fairsmoothalg} achieves the minimax-optimal regret rate when $\alpha\beta\leq d$. \hfill\Halmos

\subsection{Proofs in Section \ref{sec:effectiveattack}}
\subsubsection{Proof of Theorem \ref{thm:covert}}

Recall that $p^*$ can be chosen as any fixed lower bound satisfying
$p^*<\min_{k\in\cK}\PP_X(\cR_k)$. Choose nested balls
$D''\Subset D'\Subset D\Subset E\Subset B(\bx_0,r_0)\subseteq\cR_i\cap\cR_j$
with positive radii, where $E$ is chosen sufficiently small that
$\PP_X(\cR_k\setminus E)>p^*$ for $k\in\{i,j\}$. Let
$\rho_D:=\operatorname{dist}(D,E^c)>0$ and
$\rho_{D''}:=\operatorname{dist}(D'',(D')^c)>0$.
Because $D$ is open, there exists a smooth bump
$\psi:\cX\to[0,a]$, for a fixed constant $a>0$, such that $\psi=a$
on $D'$ and $\psi=0$ on $D^c$.

During the first epoch, the adversary uses
$c_t=-\psi(\bx_t)\mathbb I\{\pi_t=i\}$ for $t\in\cT_1$, and sets
$c_t=0$ thereafter. This is a valid pathwise-budget attack, since
\[
\sum_{t=1}^T|c_t|
\leq a|\cT_1|
=
O\bigl((\log T)^{(2\beta+d)/(\beta'-1)}\bigr)
=
\widetilde O(1).
\]

In epoch $1$, the observations of arm $i$ are centered at
$g_i:=f_i^*-\psi$, while those of arm $j$ remain centered at $f_j^*$.
Since $\psi$ is fixed and smooth, $g_i$ is $\beta$-H\"older with a
constant independent of $T$. Hence there exists a fixed constant
$c_\psi\geq1$ such that the first-epoch local-polynomial bias and
concentration argument gives an event $\mathcal E_1$ with
$\PP(\mathcal E_1)\geq1-K(5+2M_\beta^2)/T$ on which
\[
\sup_{\bx\in D'}|\widehat f_{1,i}(\bx)-g_i(\bx)|
\leq c_\psi\epsilon_1,
\qquad
\sup_{\bx\in D'}|\widehat f_{1,j}(\bx)-f_j^*(\bx)|
\leq c_\psi\epsilon_1.
\]
Since $f_i^*=f_j^*$ on $B(\bx_0,r_0)$ and $g_i=f_i^*-a$ on $D'$,
we have
$\widehat f_{1,j}(\bx)-\widehat f_{1,i}(\bx)
\geq a-2c_\psi\epsilon_1>2\epsilon_1$ on $D'$ for sufficiently
large $T$.

For sufficiently large $T$, $\sqrt d\,\delta_A<\rho_{D''}$, so every
$\bx\in D''$ has its grid representative $G_{u(\bx)}\in D'$. The
elimination rule therefore removes arm $i$ from all cubes meeting
$D''$ at the end of epoch $1$, and the nested candidate sets prevent
it from returning in later epochs. Since $f_i^*=f_j^*$ on $D''$, on
$\mathcal E_1$ and for each subsequent round with $\bx_t\in D''$,
\[
\PP(\pi_t=i\mid\mathcal F_{t-1},\bx_t)=0,
\qquad
\PP(\pi_t=j\mid\mathcal F_{t-1},\bx_t)=1,
\]
so the fairness condition is violated.

By Assumption~\ref{assump:iiddensity}, $p_D:=\PP_X(D'')>0$. The future
contexts are independent of $\mathcal E_1$, and hence
$\PP(\mathcal U_t)\geq p_D\PP(\mathcal E_1)\geq p_D/2$ for all
$t>|\cT_1|$ and sufficiently large $T$. Thus the algorithm is persistently unfair.

Moreover, for sufficiently large $T$, $|\cT_1|<T/2$. Conditional on
$\mathcal E_1$, the random variable
$\sum_{t=\lceil T/2\rceil}^{T}\mathbb I\{\bx_t\in D''\}$ is a sum of
independent Bernoulli variables with success probability $p_D$, and
each such context produces an unfair decision. A Chernoff bound gives
\[
\PP\left(
\sum_{t=\lceil T/2\rceil}^{T}\mathbb I\{\mathcal U_t\}<p_DT/4
\;\middle|\;
\mathcal E_1
\right)
\leq e^{-c_DT}
\]
for some constant $c_D>0$. Consequently, for sufficiently large $T$,
\[
\PP\left(
\sum_{t=1}^T\mathbb I\{\mathcal U_t\}\geq p_DT/4
\right)
\geq
1-\frac{K(5+2M_\beta^2)+1}{T}.
\]
Thus the attack produces persistent unfairness and linearly many unfair rounds.

It remains to verify that the attack is covert in regret. Define
$\cX_E:=\{\bx\in\cX:G_{u(\bx)}\in E^c\}$. Since
$\|\bx-G_{u(\bx)}\|_2\leq\sqrt d\,\delta_A$ and
$\overline E\Subset B(\bx_0,r_0)$, for sufficiently large $T$,
$\cX\setminus\cX_E\subset B(\bx_0,r_0)\subseteq\cR_i\cap\cR_j$.
Thus decisions outside $\cX_E$ incur zero regret. Moreover, by the
choice of $E$ and $\delta_A\to0$,
$\PP_X(\cR_k\cap\cX_E)\geq p^*$ for $k\in\{i,j\}$ and all sufficiently
large $T$.

We next localize the clean analysis to $\cX_E$. Because
$\overline E\Subset B(\bx_0,r_0)\subseteq\cR_i\cap\cR_j$, there exist
fixed constants $\bar c_0,\bar r_0>0$ such that
$\cR_k\cap E^c$ is $(\bar c_0,\bar r_0)$-regular for
$k\in\{i,j\}$. In addition, once $H_{q,k}<\rho_D/2$, every observation
entering a local-polynomial fit centered at a representative in $E^c$
has context outside $D$ and is therefore uncorrupted.

For an arm with $\Delta_{k,\max}>T^{-\beta/(2\beta+d)}$, the
boundary-growth argument of Proposition~\ref{prop:sqkregular} applies
on this localized region. Indeed, if
$\widetilde\bx\in E^c\setminus\cR_k$ and $\bx_0$ is a nearest point
in $\cR_k$, then $\bx_0\in\partial\cR_k\subset E^c$ because
$\overline E$ lies strictly inside $\cR_k$. The same ball argument
therefore yields the required weak regularity after decreasing
$\bar c_0$ and $\bar r_0$ if necessary.

For an arm with $\Delta_{k,\max}\leq T^{-\beta/(2\beta+d)}$, fix a
representative $G_j\in E^c$ where $k$ is active, and let
$b\in\arg\max_{l\in\cK_{q-1,j}}\widehat f_{q-1,l}(G_j)$. On the
preceding localized estimation event,
\[
0
\leq
\widehat f_{q-1,b}(G_j)-\widehat f_{q-1,k}(G_j)
\leq
\Delta_{k,\max}+\epsilon_{q-1}
\leq
2\epsilon_{q-1},
\qquad q\geq2.
\]
Thus arm $k$ remains active at every representative in $E^c$ at which
it was previously active. Since $E$ is a ball strictly inside the
common optimal region and $\delta_A=o(H_{q,k})$, the union of cells
represented in $E^c$ retains a fixed positive fraction of every
sufficiently small ball centered at such a representative. Hence the
same weak $(\bar c_0/2^d,H_{q,k})$-regularity holds in the small-gap
case as well.

Let
\[
m_q:=
\left(
\frac{4^q\log(T\delta_A^{-d})}{C_K}
\right)^{\frac{2\beta+d}{2\beta}}
(\log T)^{
(2\beta+d)\frac{2\beta-\beta'+1}{2\beta(\beta'-1)}
}
\]
denote the sample-size threshold defining $\mathcal M_q$. We claim
that there is an event $\mathcal L$ with
$\PP(\mathcal L^c)=\widetilde O(1/T)$ such that, simultaneously for
all completed epochs $q\leq Q-1$ and $k\in\{i,j\}$,
$\cR_k\cap\cX_E\subseteq S_{q,k}$, $N_{q,k}\geq m_q$,
$S_{q,k}$ is weakly $(\bar c_0/2^d,H_{q,k})$-regular at every
$\bx\in S_{q,k}\cap E^c\cap G$, and
\[
\sup_{\bx\in S_{q,k}\cap\cX_E}
|\widehat f_{q,k}(\bx)-f_k^*(\bx)|
\leq
\epsilon_q/2.
\]

We prove the claim by induction. In epoch $1$, both arms are sampled
with probability $1/2$. Since $K=2$, the epoch-length schedule and the
same Chernoff calculation as in
Proposition~\ref{thm:tailproboftwoevents} give $N_{1,k}\geq m_1$ for
$k\in\{i,j\}$ outside an event of probability at most $O(1/T)$.
Hence $H_{1,k}\to0$, and therefore $H_{1,k}<\rho_D/2$ for sufficiently
large $T$. Thus every local-polynomial fit centered at a representative
in $E^c$ uses only uncorrupted observations. The localized
support-regularity argument above and Lemma~\ref{lem:A.6} then give
\[
\sup_{\bx\in S_{1,k}\cap\cX_E}
|\widehat f_{1,k}(\bx)-f_k^*(\bx)|
\leq
\epsilon_1/2
\]
outside an event of probability $O(1/T)$. The localized chaining
argument also retains $\cR_k\cap\cX_E$.

Suppose now that the claim holds through epoch $q-1$, where
$2\leq q\leq Q-1$. The localized chaining argument gives
$\cR_k\cap\cX_E\subseteq S_{q,k}$ for $k\in\{i,j\}$. Conditional on
the history before epoch $q$, whenever $\bx_t\in\cR_k\cap\cX_E$,
arm $k$ is active and is selected with probability at least $1/2$.
Hence
\[
\EE(N_{q,k}\mid\mathcal F_{q-1})
\geq
\frac12\PP_X(\cR_k\cap\cX_E)|\cT_q|
\geq
\frac{p^*}{2}|\cT_q|.
\]
Since $K=2$, the definition of $|\cT_q|$ and the same Chernoff
calculation as in Proposition~\ref{thm:tailproboftwoevents} yield
$\PP(N_{q,k}<m_q\mid\mathcal F_{q-1})\leq1/T$. Thus
$N_{q,k}\geq m_q$ outside this event. Consequently,
$H_{q,k}<\rho_D/2$ for sufficiently large $T$, so all local
neighborhoods used by estimators on $\cX_E$ remain disjoint from $D$.
The localized support-regularity argument and Lemma~\ref{lem:A.6}
therefore give
\[
\sup_{\bx\in S_{q,k}\cap\cX_E}
|\widehat f_{q,k}(\bx)-f_k^*(\bx)|
\leq
\epsilon_q/2.
\]
This completes the induction through epoch $Q-1$. Summing the
conditional failure probabilities over the $O(\log T)$ completed
epochs gives $\PP(\mathcal L^c)=\widetilde O(1/T)$. The final,
possibly truncated epoch uses the estimates from epoch $Q-1$, so the
same decision guarantees hold without any sample-size or estimation
requirement on epoch $Q$.

Finally, let
$\Delta_t:=\max_{k\in\cK}f_k^*(\bx_t)-f_{\pi_t}^*(\bx_t)$. On
$\mathcal L$, the same chaining calculation as in
\eqref{eq:fistarbt-fpitbtleq} shows that, during epoch $q\geq2$ and
for $\bx_t\in\cX_E$, $\Delta_t\leq2K\epsilon_{q-1}$. Whenever
$\Delta_t>0$, this implies that there exist two arms $k\neq j$ such
that $0<|f_k^*(\bx_t)-f_j^*(\bx_t)|\leq2K\epsilon_{q-1}$. Therefore,
\[
\Delta_t\mathbb I\{\mathcal L,\bx_t\in\cX_E\}
\leq
2K\epsilon_{q-1}
\mathbb I
\left\{
\exists k\neq j:
0<
|f_k^*(\bx_t)-f_j^*(\bx_t)|
\leq
2K\epsilon_{q-1}
\right\}.
\]
Taking expectations and applying the margin condition gives
$\EE[\Delta_t\mathbb I\{\mathcal L,\bx_t\in\cX_E\}]
\lesssim\epsilon_{q-1}^{1+\alpha}$. Since $\Delta_t=0$ on
$\cX\setminus\cX_E$ and $\Delta_t\leq1$ everywhere,
$\EE[\Delta_t\mathbb I\{\mathcal L^c\}]
\leq\PP(\mathcal L^c)$. Hence
\begin{align*}
R_T^C
&\lesssim
|\cT_1|
+
\sum_{q=2}^Q
|\cT_q|\epsilon_{q-1}^{1+\alpha}
+
T\PP(\mathcal L^c)\\
&=
\widetilde O
\left(
T^{\frac{\beta+d-\alpha\beta}{2\beta+d}}
\right).
\end{align*}
Thus the attack induces persistent unfairness while preserving the
same-order regret bound.
\hfill\Halmos

\subsubsection{Proof of Theorem \ref{thm:constantbudgetunfairness}}\label{ec:lowbudgetunfairness}

Let
\(
A:=2\max_{k\in\cK}\|f_k^*\|_\infty+4
\)
and choose an arm $i$. During the initial exploration phase, whenever arm $i$ is pulled the adversary sets $c_t=-A$, and it sets $c_t=0$ otherwise. Its cost is at most $A|\cT_0|=O(\log T)$ for Algorithm \ref{alg:fairols}, and at most $A|\cT_1|=\widetilde O(1)$ for Algorithm \ref{alg:fairsmoothalg}. The observations of arm $i$ during initial exploration are centered at the masked function
\(
\widetilde f_i^*(\bx):=f_i^*(\bx)-A
\), and, uniformly over $\bx\in\cX$ and $l\neq i$,
\begin{equation}\label{eq:catastrophic-uniform-gap}
 f_l^*(\bx)-\widetilde f_i^*(\bx)\geq4.
\end{equation}

\noindent\textit{Linear case.}
Choose $i$ as in Assumption \ref{assum_excite} and, without loss of generality, take $h<1$. On the initial-estimation event $\cB$, whose failure probability is at most $4KT^{-4}$, the uniform inequalities are
\[
 \sup_{\bx\in\cX}
 \left|\bz^\mathrm{T}(\widehat\beta_{i,0}-\widetilde\beta_i)\right|<\frac h4,
 \qquad
 \sup_{\bx\in\cX}
 \left|\bz^\mathrm{T}(\widehat\beta_{l,0}-\beta_l)\right|<\frac h4,\qquad l\neq i.
\]
Combining these bounds with \eqref{eq:catastrophic-uniform-gap}, for all $\bx\in\cX$ and all $l\neq i$,
\[
 \bz^\mathrm{T}\widehat\beta_{l,0}
 -\bz^\mathrm{T}\widehat\beta_{i,0}
 \geq4-\frac h2>\frac h2.
\]
Thus arm $i$ is absent from $\widehat\cK_{\bx}$, and hence cannot be selected, at any round after initial exploration. On the excitation region $Q_i$, however, arm $i$ is better than all other arms by at least $h$. Since $\bz_t$ is independent of $\cB$ and $\PP(\bz_t\in Q_i)\geq\widetilde p$, each $t>|\cT_0|$ satisfies
\[
 \EE\!\left[\max_{k\in\cK}f_k^*(\bx_t)-f_{\pi_t}^*(\bx_t)\right]
 \geq h\,\PP(\cB,\bz_t\in Q_i)
 \geq h\widetilde p(1-4KT^{-4}).
\]
Summing gives $R_T=\Omega(T)$. On the same event, each occurrence of $\bz_t\in Q_i$ is also an unfair round because the strictly best arm receives probability zero. A Chernoff bound over the last $T/2$ fresh contexts therefore gives, for all sufficiently large $T$,
\[
 \PP\left(\sum_{t=1}^T\mathbb I\{\cU_t\}
 \geq\frac{\widetilde p}{4}T\right)
 \geq1-4KT^{-4}-e^{-c\widetilde pT}
 \geq1-\frac{4K+1}{T^4}.
\]

\noindent\textit{Smooth case.}
Choose $i\in\cK$ and $\bx^\star\in\cX$ so that
\(
f_i^*(\bx^\star)-\max_{l\neq i}f_l^*(\bx^\star)>c''
\).
Uniform concentration of the first-epoch local-polynomial estimators, together with \eqref{eq:catastrophic-uniform-gap}, implies on $\mathcal G_1$ that
\[
 \widehat f_{1,l}(\bx)-\widehat f_{1,i}(\bx)
 \geq4-\epsilon_1>2\epsilon_1,\qquad
 \bx\in\cX,\quad l\neq i,
\]
for all sufficiently large $T$. Hence the elimination rule removes arm $i$ from all cubes for all later epochs. By continuity, the set
\[
 Q_i':=\left\{\bx\in\cX:
 f_i^*(\bx)-\max_{l\neq i}f_l^*(\bx)>\frac{c''}{2}\right\}
\]
contains a nonempty open neighborhood. Assumption \ref{assump:iiddensity} therefore yields $\PP_X(Q_i')=:p_i>0$. For all $t>|\cT_1|$,
\[
 \EE\!\left[\max_{k\in\cK}f_k^*(\bx_t)-f_{\pi_t}^*(\bx_t)\right]
 \geq\frac{c''}{2}\PP(\mathcal G_1,\bx_t\in Q_i')
 \geq\frac{c''p_i}{2}\left(1-\frac{K(5+2M_\beta^2)}{T}\right).
\]
Consequently, $R_T=\Omega(T)$. Each such context is also an unfair round. Applying a Chernoff bound to the contexts in the last half of the horizon and combining it with the bound on $\mathcal G_1^c$ gives
\[
 \PP\left(\sum_{t=1}^T\mathbb I\{\cU_t\}
 \geq\frac{p_i}{4}T\right)
 \geq1-\frac{K(5+2M_\beta^2)+1}{T}.
\]
This proves both claims. \hfill\Halmos

\subsection{Proofs in Section \ref{sec:fairrobustalg}}\label{ec_pf_robust}

\subsubsection{Supporting Lemmas and Theorems}

For the robust smooth analysis, define
\[
 \epsilon_q':=(\gamma_q^{RS})^{-1/2}
       (\log T)^{\frac{\beta'-1-2\beta}{2\beta'-2}}
       \vee T^{-\frac{\beta}{2\beta+d}},
 \qquad
 n_q:=\frac{p^*}{2K}|\cT_q^{RS}|.
\]

\begin{lemma}\label{lem:robustepsilongreatthan1/2delta}
For sufficiently large $T$ and every $q=1,\ldots,Q^{RS}$,
$\epsilon'_q\geq(1+2L_1\sqrt d)\delta_A$.
\end{lemma}

\begin{lemma}\label{lem:robusttoosmallallsupport}
For each $q=1,\ldots,Q^{RS}$, on
$\overline{\mathcal G}_{q-1}^{RS}$, $S_{q,k}^{RS}=\cX$ for every
arm $k\in\cK$ satisfying
$\Delta_{k,\max}\leq T^{-\beta/(2\beta+d)}+(C/T)^{\beta/(\beta+d)}$.
\end{lemma}

\begin{lemma}\label{lem:robustA.6}
Fix $1\leq q\leq Q^{RS}-1$, an arm $k\in\cK$, and an integer $n\geq n_q$.
Under Assumptions \ref{assump:iiddensity}--\ref{assump:c0r0regular}, suppose that
$S_{q,k}^{RS}$ is weakly $(c_0/2^d,H_{q,k}^{RS})$-regular at every
$\bx\in S_{q,k}^{RS}\cap G$.
\begin{align*}
&\PP\left(
 \sup_{\bx\in S_{q,k}^{RS}}
 \left|\widehat f_{q,k}^{RS}(\bx)-f_k^*(\bx)\right|
 >\frac{\epsilon'_q}{2}
 +B_{\mathrm{LP}}\left(\frac Cn\right)^{\frac{\beta}{\beta+d}}
 \ \bigg|\ \overline{\mathcal G}_{q-1}^{RS},\overline{\mathcal M}_{q-1}^{RS},
 N_{q,k}^{RS}=n\right)\\
&\hspace{18mm}\leq
 \delta_A^{-d}(4+2M_\beta^2)
 \exp\left(-C_K n(H_{q,k}^{RS})^d {\epsilon'_q}^2\right).
\end{align*}
\end{lemma}

\subsubsection{Proof of Proposition \ref{prop:robustrandomtailbound}}

Since
\[
|\cT_0^{RL}|
=\left\lceil C_a\log T+\gamma_{\text{lin}}C\right\rceil
\geq C_a\log T,
\]
the initial-estimator and matrix-Chernoff bounds from the clean analysis remain valid under $C_a\geq C_{a,0}$.

We decompose $(\widehat\beta^{RL}_k(\mathcal I_{k,0})-\beta_k)^{\mathrm T}\bz_t$ as
\begin{align}\label{eq:tildebetakminusbetak}
   \left|(\hat{\beta}^{RL}_k(\mathcal{I}_{k,0})-\beta_k)^{\mathrm{T}}\bz_t\right|&=\left|\left((Z^\mathrm{T}Z)^{-1}Z^\mathrm{T}\tilde{Y}-\beta_k\right)^{\mathrm{T}}\bz_t\right|\nonumber\\
   &=\left|\left((Z^\mathrm{T}Z)^{-1}Z^\mathrm{T}(Z\beta_k+\bc+\bvarepsilon)-\beta_k\right)^{\mathrm{T}}\bz_t\right|\nonumber\\
   &\leq\left|\left((Z^\mathrm{T}Z)^{-1}Z^\mathrm{T}(Z\beta_k+\bvarepsilon)-\beta_k\right)^{\mathrm{T}}\bz_t\right|+\left|\left((Z^\mathrm{T}Z)^{-1}Z^\mathrm{T}\bc\right)^{\mathrm{T}}\bz_t\right|
\end{align}
where $\bc=(c_j)_{j\in\mathcal I_{k,0}}$ and $\bvarepsilon=(\varepsilon_{k,j})_{j\in\mathcal I_{k,0}}$.

The first term is exactly
$\left|(\widehat\beta_k(\mathcal I_{k,0})-\beta_k)^{\mathrm T}\bz_t\right|$.
Applying the proof of Proposition \ref{prop:randomtailbound} with
$\chi=h/8$ and $C_a\geq C_{a,0}$ gives
\begin{align}\label{eq:probmaxmaxbetahatzth8}
    \PP(\max_{\bx\in \cX}\max_{k\in\mathcal{K}}\left|(\hat{\beta}_k(\mathcal{I}_{k,0})-\beta_k)^{\mathrm{T}}\bz\right|\geq \frac{h}{8})\leq 4KT^{-4}.
\end{align}

We next bound the corruption term in \eqref{eq:tildebetakminusbetak}. By the Cauchy--Schwarz inequality,
\begin{align}\label{eq:ztzinverztcinverzt}
\max_{\bx\in \cX}\left|\left((Z^\mathrm{T}Z)^{-1}Z^\mathrm{T}\bc\right)^{\mathrm{T}}\bz\right|&\leq \|(Z^\mathrm{T}Z)^{-1}Z^\mathrm{T}\bc\|_2\max_{\bx\in \cX}\|\bz\|_2\nonumber\\
&\leq \|(Z^\mathrm{T}Z)^{-1}Z^\mathrm{T}\bc\|_2\sqrt{dr^2}\nonumber\\
&= |\mathcal{I}_{k,0}|^{-1} \|\hat{\Sigma}(\mathcal{I}_{k,0})^{-1}Z^\mathrm{T}\bc\|_2\sqrt{dr^2}\nonumber\\
&\leq |\mathcal{I}_{k,0}|^{-1} \lambda_{\min }^{-1}(\hat{\Sigma}(\mathcal{I}_{k,0})) \| \sum_{j\in \mathcal{J}}\bz_{j}c_j\|_2\sqrt{dr^2}\nonumber\\
&\leq  |\mathcal{I}_{k,0}|^{-1} \lambda_{\min }^{-1}(\hat{\Sigma}(\mathcal{I}_{k,0}))\sqrt{dr^2}\sum_{j\in \mathcal{J}}|c_j|\|\bz_{j}\|_2\nonumber\\
&\leq  |\mathcal{I}_{k,0}|^{-1} \lambda_{\min }^{-1}(\hat{\Sigma}(\mathcal{I}_{k,0})){dr^2}C,
\end{align}
where the last inequality is because $\|\bz_{j}\|_2\leq \sqrt{dr^2}$ and $\sum_{j\in \mathcal{J}}|c_j|\leq C$. Define the event $\mathcal{A}=\{|\mathcal{I}_{k,0}|\geq \frac{1}{2K}|\mathcal{T}_{0}^{RL}| \}$. By \eqref{eq:sigmahateigen}, under event $\mathcal{A}$,
    \begin{align}\label{eq:lambdaminsigmahatik0leq}
    \PP\left[\lambda_{\min }(\hat{\Sigma}(\mathcal{I}_{k,0})) \leq \frac{\lambda^{*} \tilde{p}}{4K}\right]\leq 2T^{-4},
    \end{align}
and \eqref{eq:probeventA} implies that
\begin{align*}
 \PP\left[\mathcal{A}\right]\geq 1-\frac{1}{T^4}.
\end{align*}
Conditional on $\mathcal A$, \eqref{eq:ztzinverztcinverzt} and \eqref{eq:lambdaminsigmahatik0leq} imply
\begin{align*}
        & \PP\left(\max_{\bx\in \cX}\left|\left((Z^\mathrm{T}Z)^{-1}Z^\mathrm{T}\bc\right)^{\mathrm{T}}\bz\right|\leq \frac{2K}{|\cT_0^{RL}|}\frac{4K}{\lambda^{*} \tilde{p}}dr^2C\right)\nonumber\\
   \geq   &  \PP\left(|\mathcal{I}_{k,0}|^{-1} \lambda_{\min }^{-1}(\hat{\Sigma}(\mathcal{I}_{k,0})){dr^2}C\leq \frac{2K}{|\cT_0^{RL}|}\frac{4K}{\lambda^{*} \tilde{p}}dr^2C\right)\nonumber\\
\geq  &  1-2T^{-4}.
\end{align*}
Combining this bound with $\PP(\mathcal A^c)\leq T^{-4}$ gives
\begin{align}\label{eq:ztranszinverztranscgreaterh8}
    \PP\left(\max_{\bx\in \cX}\left|\left((Z^\mathrm{T}Z)^{-1}Z^\mathrm{T}\bc\right)^{\mathrm{T}}\bz\right|\leq \frac{2K}{|\cT_0^{RL}|}\frac{4K}{\lambda^{*} \tilde{p}}dr^2C\right) \geq 1-3T^{-4},
\end{align}
which further implies
\begin{align*}
\PP\left(\max_{\bx\in \cX}\left|\left((Z^\mathrm{T}Z)^{-1}Z^\mathrm{T}\bc\right)^{\mathrm{T}}\bz\right|\geq \frac{h}{8}\right)\leq 3T^{-4},
\end{align*}
since $|\cT_0^{RL}|>C_a\log T+\frac{64K^2dr^2}{h\lambda^*\tilde{p}}C$ gives $\frac{2K}{|\cT_0^{RL}|}\frac{4K}{\lambda^{*} \tilde{p}}dr^2C\leq \frac{h}{8}$.

Taking union bounds over $k\in\cK$, we have
\begin{align}\label{eq:probmaxmaxczth8}
    \PP\left(\max_{\bx\in \cX}\max_{k\in\mathcal{K}}\left|\left((Z^\mathrm{T}Z)^{-1}Z^\mathrm{T}\bc\right)^{\mathrm{T}}\bz\right|\geq \frac{h}{8}\right) \leq 3KT^{-4}.
\end{align}
Combining \eqref{eq:tildebetakminusbetak}, \eqref{eq:probmaxmaxbetahatzth8} and \eqref{eq:probmaxmaxczth8}, we have
\begin{align*}
    & \PP\left(\max_{\bx\in \cX}\max_{k\in\mathcal{K}}\left|(\hat{\beta}^{RL}_k(\mathcal{I}_{k,0})-\beta_k)^{\mathrm{T}}\bz\right|\geq \frac{h}{4}\right)\nonumber\\
     \leq &\PP\left(\max_{\bx\in \cX}\max_{k\in\mathcal{K}}\left|(\hat{\beta}_k(\mathcal{I}_{k,0})-\beta_k)^{\mathrm{T}}\bz\right|\geq \frac{h}{8}\right)+\PP\left(\max_{\bx\in \cX}\max_{k\in\mathcal{K}}\left|\left((Z^\mathrm{T}Z)^{-1}Z^\mathrm{T}\bc\right)^{\mathrm{T}}\bz\right|\geq \frac{h}{8}\right)\leq 7KT^{-4},
\end{align*}
which finishes the proof. \hfill\Halmos

\begin{lemma}
\label{lem:robustpredictableols}
For $k\in\cK$ and $t\geq1$, let
\[
G_{k,t}:=\sum_{s=1}^t\mathbb I\{\pi_s=k\}\bz_s\bz_s^{\mathrm T},
\qquad
S_{k,t}:=\sum_{s=1}^t\mathbb I\{\pi_s=k\}\bz_s\varepsilon_{k,s},
\]
and define
\[
\mathcal B^{RL}
:=\left\{
\max_{\bx\in\cX}\max_{k\in\cK}
\left|\bz^{\mathrm T}(\widehat\beta_{k,0}^{RL}-\beta_k)\right|<\frac h4
\right\}.
\]
With $c_{\mathrm G}=\widetilde p\lambda^*/(8K)$ and
$c_{\mathrm S}=4\sigma r\sqrt d$, let
\[
\mathcal E_t^{RL}
:=\mathcal B^{RL}
\cap\left\{\min_{k\in\cK}\lambda_{\min}(G_{k,t})\geq c_{\mathrm G}t\right\}
\cap\left\{\max_{k\in\cK}\|S_{k,t}\|_2\leq c_{\mathrm S}\sqrt{t\log T}\right\}.
\]
If $C_a\geq C_{a,0}$, then for all $t>|\cT_0^{RL}|$,
\[
\PP\bigl((\mathcal E_t^{RL})^c\bigr)\leq11KT^{-4}.
\]
\end{lemma}

\textit{Proof.}
The proof of Lemma \ref{lem:predictableols} applies after replacing
$\cT_0$ by $\cT_0^{RL}$ and $\cB$ by $\mathcal B^{RL}$.
Indeed, $|\cT_0^{RL}|\geq C_a\log T$, and on $\mathcal B^{RL}$ any
context in $Q_k$ forces the selection of arm $k$. The two matrix-Chernoff
calculations in that proof therefore give
\[
\PP\left(
\mathcal B^{RL}\cap
\left\{\min_{k\in\cK}\lambda_{\min}(G_{k,t})<c_{\mathrm G}t\right\}
\right)\leq3KT^{-4},
\]
while the coordinatewise martingale bound gives
\[
\PP\left(
\max_{k\in\cK}\|S_{k,t}\|_2>c_{\mathrm S}\sqrt{t\log T}
\right)\leq KT^{-4}.
\]
Together with Proposition \ref{prop:robustrandomtailbound}, which gives
$\PP((\mathcal B^{RL})^c)\leq7KT^{-4}$, a union bound yields
$\PP((\mathcal E_t^{RL})^c)\leq11KT^{-4}$.

\hfill\Halmos

\subsubsection{Proof of Proposition \ref{prop:robustallsampletailbound}}

For $k\in\cK$, use $G_{k,t}$, $S_{k,t}$, and
$\mathcal E_t^{RL}$ from Lemma \ref{lem:robustpredictableols}, and define
\[
D_{k,t}:=\sum_{s=1}^t\mathbb I\{\pi_s=k\}\bz_sc_s.
\]
The corrupted all-sample estimator satisfies
\begin{equation}\label{eq:robust-ols-score-decomposition}
\widehat\beta_k^{RL}(\mathcal I_{k,t})-\beta_k
=G_{k,t}^{-1}(S_{k,t}+D_{k,t}).
\end{equation}

The corruption contribution is controlled pathwise. Since
$\|\bz_s\|_2\leq\sqrt{dr^2}$ and
$\sum_{s=1}^T|c_s|\leq C$,
\begin{equation}\label{eq:robust-corruption-score}
\|D_{k,t}\|_2
\leq\sum_{s=1}^t\mathbb I\{\pi_s=k\}|c_s|\,\|\bz_s\|_2
\leq\sqrt{dr^2}\,C.
\end{equation}
On $\mathcal E_t^{RL}$, combining
\eqref{eq:robust-ols-score-decomposition} and
\eqref{eq:robust-corruption-score} gives, uniformly over $\bx\in\cX$,
\begin{align*}
\left|\bz^{\mathrm T}
\bigl(\widehat\beta_k^{RL}(\mathcal I_{k,t})-\beta_k\bigr)\right|
&\leq\|\bz\|_2\|G_{k,t}^{-1}\|_{\mathrm{op}}
\bigl(\|S_{k,t}\|_2+\|D_{k,t}\|_2\bigr)\\
&\leq
\frac{c_{\mathrm S}\sqrt{dr^2}}{c_{\mathrm G}}
\sqrt{\frac{\log T}{t}}
+\frac{dr^2}{c_{\mathrm G}}\frac Ct.
\end{align*}

Since $t+1\leq2t$ and
$C_b\geq C_{b,0}=2\sqrt{2dr^2}\,c_{\mathrm S}/c_{\mathrm G}$,
\[
\frac{c_{\mathrm S}\sqrt{dr^2}}{c_{\mathrm G}}
\sqrt{\frac{\log T}{t}}
\leq
\frac{C_b}{2}\sqrt{\frac{\log T}{t+1}}.
\]
Moreover, $1/t\leq2/(t+1)$ and
\[
\kappa
=\frac{192K^2dr^2}{\lambda^*\widetilde p}
\vee\frac{32dr^2}{\lambda^*\widetilde p^2}
\geq\frac{4dr^2}{c_{\mathrm G}},
\]
so
\[
\frac{dr^2}{c_{\mathrm G}}\frac Ct
\leq\frac{\kappa C}{2(t+1)}.
\]
Thus, on $\mathcal E_t^{RL}$,
\[
\max_{\bx\in\cX}\max_{k\in\cK}
\left|\bz^{\mathrm T}
\bigl(\widehat\beta_k^{RL}(\mathcal I_{k,t})-\beta_k\bigr)\right|
\leq\frac{\epsilon_{t+1}^{RL}}2.
\]
Lemma \ref{lem:robustpredictableols} therefore yields
\[
\PP\left(
\max_{\bx\in\cX}\max_{k\in\cK}
\left|\bz^{\mathrm T}
\bigl(\widehat\beta_k^{RL}(\mathcal I_{k,t})-\beta_k\bigr)\right|
\geq\frac{\epsilon_{t+1}^{RL}}2
\right)
\leq11KT^{-4}\leq15KT^{-4}.
\]
This proves the proposition. \hfill\Halmos

\subsubsection{Proof of Theorem \ref{thm:robustlinearfairness}}
If $|\mathcal T_0^{RL}|=T$, the policy is uniform throughout and is fair. Otherwise the initial exploration is untruncated.
Let $\cB^{RL}$ and $\cW_t^{RL}$ denote, respectively, the robust initial- and all-sample error events supplied by Propositions \ref{prop:robustrandomtailbound} and \ref{prop:robustallsampletailbound}. Their joint failure probability is at most
\[
 \PP\bigl((\cB^{RL}\cap\cW_t^{RL})^C\bigr)
 \leq7KT^{-4}+15KT^{-4}=22KT^{-4}.
\]
The proof of Theorem \ref{thm:linearfairness} applies with these robust events and $\epsilon_t^{RL}$. If an excluded arm $j$ belongs to $\widehat\cK_{\bx}^{RL}\setminus\cK_c^{RL}(\bx)$, the $\epsilon_t^{RL}$-component separation and $\cW_t^{RL}$ imply that each selected arm is truly better than $j$. If $j\notin\widehat\cK_{\bx}^{RL}$, the $h/2$-component separation and $\cB^{RL}$ give the same conclusion. Thus fairness holds on $\cB^{RL}\cap\cW_t^{RL}$. A union bound over $t\in[T]$ gives simultaneous fairness with probability at least $1-22KT^{-3}>1-1/T$ when $T^2>22K$. \hfill\Halmos

\subsubsection{Proof of Theorem \ref{thm:robustminimaxoptimallinear}}
If $|\mathcal T_0^{RL}|=T$, then $T\leq\lceil C_a\log T+\gamma_{\mathrm{lin}}C\rceil$ and $R_T=O(\log T+C)$, which is within the stated bound. Otherwise the initial exploration is untruncated.
Define $\cG_t^{RL}:=\cB^{RL}\cap\cW_t^{RL}$. The preceding proof gives $\PP((\cG_t^{RL})^C)\leq22KT^{-4}$. The exploration phase contributes
\[
 R_1=O(|\cT_0^{RL}|)=O(\log T+C),
\]
and bounded rewards give $R_3=O(1)$ on the complements of the good events.

On $\cG_t^{RL}$, the $h/2$-link argument from the proof of Theorem \ref{thm:minimaxoptimallinear}, applied with $\cB^{RL}$, shows that the true best arm belongs to $\widehat\cK_{\bx_t}^{RL}$. It follows that the true best arm belongs to $\cK_c^{RL}(\bx_t)$ and that each selected arm has instantaneous regret at most $K\epsilon_t^{RL}$. The margin condition then yields
\[
 \EE\left[\left(\max_{k\in\cK}\bz_t^{\mathrm T}\beta_k-
 \bz_t^{\mathrm T}\beta_{\pi_t}\right)\mathbb I(\cG_t^{RL})\right]
 \lesssim(\epsilon_t^{RL})^2
 \lesssim\frac{\log T}{t}+\frac{C^2}{t^2},
\]
where the last inequality uses $(a+b)^2\leq2a^2+2b^2$ and the additive threshold
\(
\epsilon_t^{RL}=C_b\sqrt{\log T/t}+\kappa C/t.
\)
Since $|\cT_0^{RL}|\gtrsim\log T+C$,
\[
 \sum_{t>|\cT_0^{RL}|}\left(\frac{\log T}{t}+\frac{C^2}{t^2}\right)
 =O\left(\log^2T+\frac{C^2}{\log T+C}\right)
 =O(\log^2T+C).
\]
Combining the three components proves the theorem. \hfill\Halmos

\subsubsection{Proof of Theorem \ref{thm:lowerboundlinear}}
The lower bound on the cumulative regret incurred by any admissible policy for the problem class defined by Assumptions \ref{assum_parabound} through \ref{assum_excite} under corrupted observation model with budget $C$ is established as follows.

It suffices to consider a noiseless two-arm instance without an intercept, with $f_k(x)=\beta_kx$ and $X\sim\operatorname{Unif}[-1,1]$. This instance satisfies Assumptions \ref{assum_parabound}--\ref{assum_excite}. Let $\beta_1=0$ and $\beta_2=a$, where $a$ is a random variable satisfying $\PP(a=1)=\frac{1}{2}$ and $\PP(a=-1)=\frac{1}{2}$. Set $T_C=\lfloor C\rfloor\wedge T$. During the first $T_C$ rounds, the adversary uses $c_t=-a x_t$ if $\pi_t=2$ and $c_t=0$ if $\pi_t=1$. Since $|x_t|\leq1$, this attack satisfies $\sum_{t=1}^T|c_t|\leq T_C\leq C$ on any realization. Both arms then produce the observation zero during these rounds, so the observed history, and hence the policy's action, is independent of $a$. Consider the problem set as $\cP = \{(\beta_1 = 0, \beta_2 = 1), (\beta_1 = 0, \beta_2 = -1)\}$. Hence, for any admissible policy, we have
\begin{align}\label{eq_pf_lbrobustlinear1}
    & \sup_{(\beta_1,\beta_2)\in \cP}\EE\left(\sum_{t=1}^T \left(\max_{k=1,2} \beta_kx_t  -  \beta_{\pi_t}x_t\right)\right)\nonumber\\
    \geq & \EE_a\EE\left(\sum_{t=1}^T \left(\max_{k=1,2} \beta_kx_t  -  \beta_{\pi_t}x_t\right)\right)\nonumber\\
    \geq & \sum_{t=1}^{T_C} \EE\left(\left(\max_{k=1,2} \beta_kx_t  -  \beta_{\pi_t}x_t\right)\bigg|a=1,\pi_t = 1,x_t>\frac{1}{2}\right)\PP(a=1,\pi_t = 1,x_t>1/2)\nonumber\\
    & + \sum_{t=1}^{T_C}\EE\left(\left(\max_{k=1,2} \beta_kx_t  -  \beta_{\pi_t}x_t\right)\bigg|a=-1,\pi_t = 2,x_t>\frac{1}{2}\right)\PP(a=-1,\pi_t = 2,x_t>1/2).
\end{align}
For $t\leq T_C$, if $a=1,\pi_t = 1,x_t>\frac{1}{2}$, the optimal arm should be 2, while the player chooses arm 1, with the regret $\beta_2\bx_t-0=a\bx_t>1/2$. Similarly, for $t\leq T_C$, if $a=-1,\pi_t = 2,x_t>1/2$, the optimal arm should be 1, while the player chooses arm 2, with $0-\beta_2\bx_t=-a\bx_t>1/2$. Therefore, \eqref{eq_pf_lbrobustlinear1} implies that
\begin{align}\label{eq_pf_lbrobustlinear2}
    & \sup_{(\beta_1,\beta_2)\in \cP}\EE\left(\sum_{t=1}^T \left(\max_{k=1,2} \beta_kx_t  -  \beta_{\pi_t}x_t\right)\right)\nonumber\\
    \geq & \frac{1}{2}\sum_{t=1}^{T_C} \left(\PP(a=1,\pi_t = 1,x_t>1/2) +\PP(a=-1,\pi_t = 2,x_t>1/2)\right)\nonumber\\
    = & \frac{1}{2}\sum_{t=1}^{T_C} \left(\PP(a=1)\PP(\pi_t = 1,x_t>1/2) +\PP(a=-1)\PP(\pi_t = 2,x_t>1/2)\right)\nonumber\\
    = & \frac{1}{4}\sum_{t=1}^{T_C} \PP(x_t>1/2) = \frac{1}{16}T_C = \Omega(C),
\end{align}
where the first equality is by the independence of $a$ and $\pi_t$. It can be seen from \eqref{eq_pf_lbrobustlinear2} that $R_T=\Omega(C)$ (for the nontrivial regime $1\leq C\leq T$). Moreover, by Theorem \ref{thm:price}, we have that any admissible policy will incur a cumulative regret of at least $R_T=\Omega(\log^2 T)$. Since $\max\{u,v\}\geq(u+v)/2$ for nonnegative $u,v$, combining the two instance-wise lower bounds gives $R_T=\Omega(\log^2 T+C)$. This finishes the proof. \hfill\Halmos

\subsubsection{Proof of Proposition \ref{prop:robustsqkregular}}

The case $q=1$ follows because $S_{1,k}^{RS}=\cX$ for all arms. When
$Q^{RS}>1$, the first completed epoch satisfies
$H_{1,k}^{RS}\leq r_0$ on $\mathcal M_1^{RS}$ for sufficiently large
$T$. 
Suppose $q\geq2$ and fix $k\in\cK$. On $\mathcal M_q^{RS}$,
$H_{q,k}^{RS}\leq r_0$ for sufficiently large $T$.
If $\Delta_{k,\max}\leq\rho$, where
$\rho=T^{-\beta/(2\beta+d)}+(C/T)^{\beta/(\beta+d)}$,
Lemma~\ref{lem:robusttoosmallallsupport} gives $S_{q,k}^{RS}=\cX$,
and the same regularity argument applies. Otherwise,
Assumption~\ref{assump:Qxliplower} supplies boundary growth for this arm.

Fix $\widetilde\bx\in S_{q,k}^{RS}\cap G$.
The chaining argument of Lemma~\ref{lem:RksubsetSqk} gives
$\cR_k\subseteq S_{q,k}^{RS}$.
If $\widetilde\bx\in\cR_k$, its regularity proves the claim.
Otherwise, choose an optimal arm at $\widetilde\bx$.
On $\overline{\mathcal G}_{q-1}^{RS}$, the same calculation as in
\eqref{eq:active-gap-boundary-growth} yields
\begin{align}\label{eq:active-gap-boundary-growth-robust}
\Delta_k(\widetilde\bx)
\leq
2K\epsilon_{q-1}^{RS}.
\end{align}

We first record the relevant scales. On $\mathcal M_q^{RS}$,
\[
\frac{p^*}{2K}|\cT_q^{RS}|
\leq N_{q,k}^{RS}\leq|\cT_q^{RS}|,
\]
so $N_{q,k}^{RS}\asymp|\cT_q^{RS}|$. If the $C$-dependent component
dominates the first-epoch schedule, then
$|\cT_1^{RS}|\asymp
C(\log T)^{(2\beta+d)/(\beta'-1)}$ up to constants; otherwise $C=O(1)$
while $|\cT_1^{RS}|\to\infty$. Consequently, whenever $Q^{RS}>1$ and $C>0$,
\begin{align}\label{eq:robust-sample-corruption-ratio}
\frac{N_{q,k}^{RS}}{C}\longrightarrow\infty,
\qquad q\leq Q^{RS}-1.
\end{align}
When $C=0$, the corruption component below is identically zero. Moreover,
\[
\epsilon_{q-1}^{RS}
=
\epsilon'_{q-1}
+2B_{\mathrm{LP}}
\left(\frac{4KC}{p^*|\cT_{q-1}^{RS}|}\right)^{\frac{\beta}{\beta+d}},
\]
and both terms vanish uniformly over the completed epochs.
Let $c_{\mathfrak r,k}>0$ be the arm-specific constant defined
in the clean support-regularity proof. For sufficiently large $T$,
$2K\epsilon_{q-1}^{RS}<c_{\mathfrak r,k}$, so
\eqref{eq:active-gap-boundary-growth-robust} implies
$\operatorname{dist}(\widetilde\bx,\cR_k)<\mathfrak r$.

Choose a projection $\bx_0\in\cR_k$ and set
$d_q^{RS}:=\|\widetilde\bx-\bx_0\|_2$. By
Assumption~\ref{assump:Qxliplower},
\begin{align}\label{eq:boundary-distance-robust}
d_q^{RS}
\leq
\widetilde c(2K\epsilon_{q-1}^{RS})^{\beta'/\beta}.
\end{align}
It remains to compare this distance with the corruption-aware bandwidth.
Recall that adjacent scheduled scales satisfy
$\gamma_q^{RS}\leq4\gamma_{q-1}^{RS}$. Also, $N_{q,k}^{RS}\leq\bar n_q^{RS}$.
On the completed epochs, the schedule gives
$\bar n_q^{RS}\asymp\allowbreak
(\gamma_q^{RS})^{(2\beta+d)/(2\beta)}
(\log T)^{(2\beta+d)/(\beta'-1)}$.
Using $H_{q,k}^{RS}\geq(N_{q,k}^{RS})^{-1/(2\beta+d)}$
and the definition of $\epsilon'_{q-1}$, we obtain
\begin{align*}
\frac{(\epsilon'_{q-1})^{\beta'/\beta}}{H_{q,k}^{RS}}
&\lesssim
(\gamma_{q-1}^{RS})^{-\beta'/(2\beta)}
(\log T)^{\frac{\beta'}{\beta}
\frac{\beta'-1-2\beta}{2\beta'-2}}
(\bar n_q^{RS})^{1/(2\beta+d)}
+T^{-\beta'/(2\beta+d)}(N_{q,k}^{RS})^{1/(2\beta+d)}\\
&\lesssim
(\gamma_{q-1}^{RS})^{(1-\beta')/(2\beta)}
(\log T)^{(\beta'-2\beta)/(2\beta)}
+T^{-(\beta'-1)/(2\beta+d)}=o(1).
\end{align*}
The second inequality uses $\gamma_q^{RS}\leq4\gamma_{q-1}^{RS}$
and $N_{q,k}^{RS}\leq T$, uniformly over completed epochs and
admissible corruption budgets.
For the corruption component, let
\[
b_q
:=
2B_{\mathrm{LP}}
\left(\frac{4KC}{p^*|\cT_{q-1}^{RS}|}\right)^{\frac{\beta}{\beta+d}}
\asymp
\left(\frac{C}{N_{q,k}^{RS}}\right)^{\frac{\beta}{\beta+d}}.
\]
Since
$H_{q,k}^{RS}\geq(C/N_{q,k}^{RS})^{1/(\beta+d)}$,
\[
\frac{b_q^{\beta'/\beta}}{H_{q,k}^{RS}}
\lesssim
\left(\frac{C}{N_{q,k}^{RS}}\right)^{\frac{\beta'-1}{\beta+d}}
=o(1)
\]
when $C>0$, by \eqref{eq:robust-sample-corruption-ratio} and
$\beta'>1$; when $C=0$, the ratio is zero. Because
$(u+v)^{\beta'/\beta}\lesssim
u^{\beta'/\beta}+v^{\beta'/\beta}$ for $u,v\geq0$, it follows from
\eqref{eq:boundary-distance-robust} that
$d_q^{RS}/H_{q,k}^{RS}=o(1)$. Hence
$d_q^{RS}\leq H_{q,k}^{RS}/2$ for sufficiently large $T$.

Finally,
$\mathcal B(\bx_0,H_{q,k}^{RS}-d_q^{RS})
\subseteq\mathcal B(\widetilde\bx,H_{q,k}^{RS})$.
Using $\cR_k\subseteq S_{q,k}^{RS}$ and the regularity of $\cR_k$ gives
\begin{align*}
\operatorname{Leb}
\bigl[S_{q,k}^{RS}\cap
\mathcal B(\widetilde\bx,H_{q,k}^{RS})\bigr]
&\geq
c_0\operatorname{Leb}
\bigl[\mathcal B(\bx_0,H_{q,k}^{RS}-d_q^{RS})\bigr]\\
&\geq
\frac{c_0}{2^d}
\operatorname{Leb}
\bigl[\mathcal B(\widetilde\bx,H_{q,k}^{RS})\bigr].
\end{align*}
Thus $S_{q,k}^{RS}$ is weakly
$(c_0/2^d,H_{q,k}^{RS})$-regular at all points of
$S_{q,k}^{RS}\cap G$. \hfill\Halmos

\subsubsection{Proof of Proposition \ref{thm:robusttailproboftwoevents}}
Let $\mathcal F_{q-1}$ denote the history available at the beginning of
epoch $q$. On
$\overline{\mathcal G}_{q-1}^{RS}\cap
\overline{\mathcal M}_{q-1}^{RS}$,
Lemma~\ref{lem:RksubsetSqk} gives
$\mathcal R_k\subseteq S_{q,k}^{RS}$ for every $k\in\mathcal K$.
Hence, whenever $\bx_t\in\mathcal R_k$, arm $k$ remains active and is
selected with probability at least $1/K$. By Assumption
\ref{assump:iiddensity},
\begin{align}
\PP\left(\pi_t=k\mid\mathcal F_{q-1}\right)
&\geq
\frac{1}{K}\PP\left(\bx_t\in\mathcal R_k\right)
\geq \frac{p^*}{K}.
\label{eq:robust-arm-sampling-prob}
\end{align}
Conditional on $\mathcal F_{q-1}$, the contexts and the random arm
selections within epoch $q$ are independent across rounds. Therefore,
for
\[
N_{q,k}^{RS}
=
\sum_{t\in\mathcal T_q^{RS}}
\mathbb I\{\pi_t=k\},
\]
the threshold
$p^*|\mathcal T_q^{RS}|/(2K)$ is one half of the lower bound on the
conditional mean. Hence a standard Chernoff bound gives
\begin{align}
\PP\left(
N_{q,k}^{RS}
<
\frac{p^*}{2K}|\mathcal T_q^{RS}|
\,\bigg|\,
\mathcal F_{q-1}
\right)
&\leq
\exp\left\{-c|\mathcal T_q^{RS}|\right\}
\leq \frac{1}{T},
\label{eq:robust-arm-sampling-chernoff}
\end{align}
for some constant $c>0$, where the last inequality follows from the
logarithmic factor in the definition of $|\mathcal T_q^{RS}|$.
Averaging \eqref{eq:robust-arm-sampling-chernoff} over
$\mathcal F_{q-1}$ and taking a union bound over
$k\in\mathcal K$ yields
\begin{equation}\label{eq:robustsecondeq}
\PP\left(
(\mathcal M_q^{RS})^c
\mid
\overline{\mathcal G}_{q-1}^{RS},
\overline{\mathcal M}_{q-1}^{RS}
\right)
\leq \frac{K}{T}.
\end{equation}
On
\(\overline{\mathcal G}_{q-1}^{RS}\cap\overline{\mathcal M}_q^{RS}\),
Proposition \ref{prop:robustsqkregular} establishes part (i) of
\(\mathcal G_q^{RS}\). Moreover,
\(N_{q,k}^{RS}\geq n_q\). The scheduled epoch length gives
\[
n_q
\geq
\left(
\frac{\gamma_q^{RS}\log(T\delta_A^{-d})}{C_K}
\right)^{\frac{2\beta+d}{2\beta}}
(\log T)^{
\frac{2\beta+d}{\beta'-1}
-\frac{2\beta+d}{2\beta}}.
\]
Together with $\epsilon'_q\geq
(\gamma_q^{RS})^{-1/2}
(\log T)^{(\beta'-1-2\beta)/(2\beta'-2)}$, this yields
\[
C_K n_q^{\frac{2\beta}{2\beta+d}}(\epsilon'_q)^2
\geq \log(T\delta_A^{-d}).
\]
Since
$H_{q,k}^{RS}\geq(N_{q,k}^{RS})^{-1/(2\beta+d)}$,
\begin{align*}
C_KN_{q,k}^{RS}(H_{q,k}^{RS})^d(\epsilon'_q)^2
\geq
C_K(N_{q,k}^{RS})^{\frac{2\beta}{2\beta+d}}(\epsilon'_q)^2\geq
C_Kn_q^{\frac{2\beta}{2\beta+d}}(\epsilon'_q)^2
\geq\log(T\delta_A^{-d}).
\end{align*}
Consequently,
\[
\delta_A^{-d}\exp\left\{
-C_KN_{q,k}^{RS}(H_{q,k}^{RS})^d(\epsilon'_q)^2
\right\}
\leq\frac1T.
\]
On $\mathcal M_q^{RS}$,
\[
B_{\mathrm{LP}}
\left(\frac{C}{N_{q,k}^{RS}}\right)^{\frac{\beta}{\beta+d}}
\leq
B_{\mathrm{LP}}
\left(\frac{4KC}{p^*|\cT_q^{RS}|}\right)^{\frac{\beta}{\beta+d}}.
\]
Consequently, Lemma \ref{lem:robustA.6} gives
\begin{align*}
\sup_{\bx\in S_{q,k}^{RS}}
\left|\widehat f_{q,k}^{RS}(\bx)-f_k^*(\bx)\right|
\leq \frac{\epsilon'_q}{2}
+B_{\mathrm{LP}}
\left(\frac{4KC}{p^*|\cT_q^{RS}|}\right)^{\frac{\beta}{\beta+d}}=\frac{\epsilon_q^{RS}}2
\end{align*}
outside the failure event in that lemma. This is part (ii) of
$\mathcal G_q^{RS}$.

Applying Lemma \ref{lem:robustA.6} and taking a union bound over the arms
therefore gives
\begin{equation}\label{eq:robustfirsteq}
 \PP\!\left((\mathcal G_q^{RS})^c
 \mid\overline{\mathcal G}_{q-1}^{RS},
       \overline{\mathcal M}_q^{RS}\right)
 \leq\frac{K(4+2M_\beta^2)}T.
\end{equation}
Combining \eqref{eq:robustsecondeq} and
\eqref{eq:robustfirsteq} recursively over the epochs proves
\[
 \PP\left((\overline{\mathcal G}_q^{RS})^c
       \cup(\overline{\mathcal M}_q^{RS})^c\right)
 \leq\frac{K(5+2M_\beta^2)q}{T}.
\]
\hfill\Halmos

\subsubsection{Proof of Theorem \ref{thm:robustsmoothbanditfairness}}
The policy is uniform in the first epoch. In each later epoch, on
\(\overline{\mathcal G}_{q-1}^{RS}\cap
\overline{\mathcal M}_{q-1}^{RS}\), the estimates used in the chaining
step are uniformly within \(\epsilon_{q-1}^{RS}/2\) of their true
values. The chaining argument in the proof of Theorem
\ref{thm:smoothbanditfairness} then shows that an arm with a lower true
reward is never assigned a larger selection probability. Proposition
\ref{thm:robusttailproboftwoevents} and a union bound over
\(Q^{RS}=O(\log T)\) epochs give failure probability
\(\widetilde O(1/T)\). If \(Q^{RS}=1\), uniform exploration is fair
directly. \hfill\Halmos

\subsubsection{Proof of Theorem \ref{thm:robustupperbound}}
If $Q^{RS}=1$, the policy is uniformly random and
\[
 R_T\leq T\leq\bar n_1^{RS}=\widetilde O(1\vee C).
\]
For sufficiently large $T$, this is $\widetilde O(C)$ when
$\alpha\beta\geq d$. When $\alpha\beta<d$, the relation
$T\leq\bar n_1^{RS}$ implies $C=\widetilde\Omega(T)$, and hence
\[
 C^{\frac{\beta(1+\alpha)}{\beta+d}}
 T^{\frac{d-\alpha\beta}{\beta+d}}
 =\widetilde\Omega(T).
\]
Thus the stated bound holds when $Q^{RS}=1$. In the remainder, suppose
$Q^{RS}>1$. The first epoch satisfies
\[
 |\cT_1^{RS}|=\bar n_1^{RS}=\widetilde\Theta(1\vee C).
\]
For each $q\geq2$, on
$\overline{\mathcal G}_{q-1}^{RS}\cap\overline{\mathcal M}_{q-1}^{RS}$,
the active sets and estimation bounds satisfy the conditions of Lemma
\ref{lem:chaining-regret} with $\eta=\epsilon_{q-1}^{RS}$. Therefore,
\[
\EE[R_q\mid\overline{\mathcal G}_{q-1}^{RS},
                \overline{\mathcal M}_{q-1}^{RS}]
\lesssim |\cT_q^{RS}|(\epsilon_{q-1}^{RS})^{1+\alpha}.
\]
The bad-event contribution is \(\widetilde O(1)\) by Proposition
\ref{thm:robusttailproboftwoevents}. Hence
\begin{equation}\label{eq:robust-smooth-regret-sum}
 R_T\lesssim |\cT_1^{RS}|
 +\sum_{q=2}^{Q^{RS}}|\cT_q^{RS}|
       (\epsilon_{q-1}^{RS})^{1+\alpha}
 +\widetilde O(1).
\end{equation}
The threshold used in this sum is
\[
\epsilon_{q-1}^{RS}
=\epsilon'_{q-1}
+2B_{\mathrm{LP}}
\left(\frac{4KC}{p^*|\cT_{q-1}^{RS}|}\right)^{\frac{\beta}{\beta+d}}.
\]
For $2\leq q\leq Q^{RS}-1$, adjacent scheduled scales give
$|\cT_q^{RS}|\asymp
(\gamma_{q-1}^{RS})^{(2\beta+d)/(2\beta)}
(\log T)^{(2\beta+d)/(\beta'-1)}$.
Hence $\epsilon'_{q-1}\lesssim\sqrt{\log T}\,
|\cT_q^{RS}|^{-\beta/(2\beta+d)}$, including the floor term
since $|\cT_q^{RS}|\leq T$. Up to logarithmic factors,
\[
\epsilon_{q-1}^{RS}
\lesssim |\cT_q^{RS}|^{-\beta/(2\beta+d)}
+\left(\frac{C}{|\cT_q^{RS}|}\right)^{\beta/(\beta+d)}.
\]
For the final epoch, $|\cT_{Q^{RS}}^{RS}|\leq\bar n_{Q^{RS}}^{RS}$ and
$\bar n_{Q^{RS}}^{RS}\lesssim\bar n_{Q^{RS}-1}^{RS}
=|\cT_{Q^{RS}-1}^{RS}|$. Therefore, with
$r:=\beta(1+\alpha)/(\beta+d)$ within this proof,
\[
\begin{aligned}
 |\cT_{Q^{RS}}^{RS}|(\epsilon_{Q^{RS}-1}^{RS})^{1+\alpha}
 \lesssim{}&
 |\cT_{Q^{RS}-1}^{RS}|^{
 1-\frac{\beta(1+\alpha)}{2\beta+d}}+C^r|\cT_{Q^{RS}-1}^{RS}|^{1-r}.
\end{aligned}
\]

By the definition of $Q^{RS}$ and the geometric epoch schedule, $Q^{RS}=O(\log T)$. Moreover, since the scheduled epoch lengths are nondecreasing and adjacent scheduled lengths differ by at most a fixed factor, the total length of all epochs up to $Q^{RS}$ is within a logarithmic factor of the last nontruncated epoch. Hence $|\cT_{Q^{RS}-1}^{RS}|=\widetilde\Theta(T)$.

Since $1-\frac{\beta(1+\alpha)}{2\beta+d}>0$, the quantity $|\cT_q^{RS}|^{1-\frac{\beta(1+\alpha)}{2\beta+d}}$ increases with the epoch length. Therefore, the sum over the nontruncated epochs is dominated, up to the factor $Q^{RS}=O(\log T)$, by the last nontruncated epoch. Together with the final-epoch bound above, this gives
\begin{align*}
\sum_{q=2}^{Q^{RS}-1}|\cT_q^{RS}|^{1-\frac{\beta(1+\alpha)}{2\beta+d}}
+|\cT_{Q^{RS}-1}^{RS}|^{1-\frac{\beta(1+\alpha)}{2\beta+d}}=\widetilde O\left(T^{1-\frac{\beta(1+\alpha)}{2\beta+d}}\right)
=\widetilde O\left(T^{\frac{\beta+d-\alpha\beta}{2\beta+d}}\right).
\end{align*}

It remains to control $C^r\sum_{q=2}^{Q^{RS}-1}|\cT_q^{RS}|^{1-r}+C^r|\cT_{Q^{RS}-1}^{RS}|^{1-r}$.
If $\alpha\beta<d$, then $r=\frac{\beta(1+\alpha)}{\beta+d}<1$, so $1-r>0$. Hence the terms increase with the epoch length, and the last nontruncated epoch again dominates up to logarithmic factors. Using $|\cT_{Q^{RS}-1}^{RS}|=\widetilde\Theta(T)$, we obtain
\begin{align*}
C^r\sum_{q=2}^{Q^{RS}-1}|\cT_q^{RS}|^{1-r}+C^r|\cT_{Q^{RS}-1}^{RS}|^{1-r}=\widetilde O(C^rT^{1-r})
=\widetilde O\left(C^{\frac{\beta(1+\alpha)}{\beta+d}}T^{\frac{d-\alpha\beta}{\beta+d}}\right).
\end{align*}

If $\alpha\beta\geq d$, then $r\geq1$, so $1-r\leq0$. In this case the terms decrease as the epoch length increases, and therefore the first nontruncated scale gives the largest contribution. Since adjacent scheduled epoch lengths differ by at most a fixed factor, $|\cT_2^{RS}|\asymp|\cT_1^{RS}|$ up to logarithmic factors, while $|\cT_1^{RS}|=\widetilde\Theta(1\vee C)$. Consequently,
\[
C^r\sum_{q=2}^{Q^{RS}-1}|\cT_q^{RS}|^{1-r}+C^r|\cT_{Q^{RS}-1}^{RS}|^{1-r}
\lesssim Q^{RS}C^r|\cT_1^{RS}|^{1-r}.
\]
Without loss of generality we let $C\geq1$. Then this is $\widetilde O(C^rC^{1-r})=\widetilde O(C)$. Moreover, $R_1\leq|\cT_1^{RS}|=\widetilde O( C)$. 
Substituting these bounds into \eqref{eq:robust-smooth-regret-sum} proves the theorem. \hfill\Halmos

\subsubsection{Proof of Theorem \ref{thm:robustlowerbound}}

\noindent\textit{Case $\alpha\beta<d$.}

For $\alpha\beta<d$, we use the localized bump construction underlying the smooth-bandit lower bound in \cite{hu2022smooth,audibert2007fast}, together with a pathwise truncated attack.
Let $\delta_0 \in (0, \frac{1}{2})$ be fixed. Fix a sufficiently large integer $q_0\geq2$, depending only on the problem-class parameters, so that the separated grid construction below is available, $\lceil q^{d-\alpha\beta}\rceil\geq2$, and $\lceil q^{d-\alpha\beta}\rceil q^{-d}\leq1/2$ for every $q\geq q_0$. We set the instance parameters as  
\begin{align*}
   q = q_0\vee \left\lceil \left(4C_\phi \frac{T}{C} \right)^{\frac{1}{\beta+d}} \right\rceil, \quad  
    m = \lceil q^{d-\alpha\beta} \rceil, \quad  
    \omega = q^{-d},
\end{align*}
where $T$ is assumed to be sufficiently large. 

Let
\[
G_q = \left\{
\left(\frac{2j_1+1}{2q},\dots,\frac{2j_d+1}{2q}\right):
j_i\in\{0,\dots,q-1\},\ i=1,\dots,d
\right\}
\]
be the $d$-dimensional grid. For any $\bx\in[0,1]^d$, define
$g_q(\bx)$ as the nearest point in $G_q$ (with ties broken arbitrarily),
and let $\mathrm{Cube}_q(z)
=
\{\bx'\in[0,1]^d:g_q(\bx')=z\},
\qquad z\in G_q.$
Choose $m$ approximately evenly spaced grid points
$\bx_1,\ldots,\bx_m\in G_q$ such that $\min_{i\neq j}\|\bx_i-\bx_j\|_\infty
\geq c_d m^{-1/d}$
for some constant $c_d>0$ depending only on $d$.
Such a choice is possible for sufficiently large fixed $q_0$.
Let
$\mathcal X_i=\mathrm{Cube}_q(\bx_i)$ and
$\mathcal X_0=[0,1]^d\setminus\bigcup_{i=1}^m\mathcal X_i$.

The density $ \mu_X$ of the covariate distribution $ \mathbb{P}_X $ is defined as
\[
\mu_X(\bx) = 
\begin{cases} 
\frac{\omega}{\operatorname{Leb}(\mathcal{B}(0, \frac{1}{4q}))} & \text{if } \bx \in \bigcup_{i=1}^m \mathcal{B}(\bx_i, \frac{1}{4q}) \\
\frac{1 - m\omega}{\operatorname{Leb}(\mathcal{X}_0)} & \text{if } \bx \in \mathcal{X}_0 \\
0 & \text{otherwise}.
\end{cases}
\]
Thus, $\cX=\mathcal{X}_0 \cup \bigcup_{i=1}^m \mathcal{B}(\bx_i, \frac{1}{4q})$. 
Following \cite{hu2022smooth}, we utilize an infinitely differentiable bump function constructed via:  
\[
u_1(x) = \begin{cases}  
\exp\left( -\frac{1}{(\frac{1}{2}-\bx)(\bx-\frac{1}{4})} \right), & \bx \in (\frac{1}{4}, \frac{1}{2}) \\  
0, & \text{otherwise}  
\end{cases}
\]  
and define $u: \mathbb{R}_+ \to \mathbb{R}_+$ through the normalization:  
\[
u(\bx) = \left( \int_{1/4}^{1/2} u_1(t) \, dt \right)^{-1} \int_{\bx}^\infty u_1(t) \, dt.
\] 
Let $ C_\phi \in (0, \delta_0] $ and function $ \phi: \mathbb{R}^d \to \mathbb{R}_+ $ be defined as  
\[
\phi(\bx) = C_\phi u(\|\bx\|).
\] 
Let $\varphi_j(\bx) = q^{-\beta} \phi\left(q(\bx - \bx_j)\right) \mathbb{I}_{\mathcal{X}_j}(\bx)$ for $j=1,\ldots,m$.
We define the expected reward function for arm 2 as $\eta_2 \equiv \frac{1}{2}$, while $\eta_{1,i}: [0,1]^d \to \mathbb{R} $ represents the expected reward function for arm 1 with
\[
\eta_{1,i}(\bx) = \frac{1}{2} + \sum_{j=1,j\neq i}^m \varphi_j(\bx).
\]  
Hence, there are $m$ possible configurations of arm 1. The constraint $ C_\phi \leq \delta_0 $ ensures $ \eta_{1,i}(\bx) \in [\frac{1}{2}, \frac{1}{2} + \delta_0] \subset [0,1]$. 

When $K=2$, we need to check that the constructed instances satisfy Assumptions \ref{assump:iiddensity}--\ref{assump:Qxliplower}. First, Assumptions \ref{assump:iiddensity}-\ref{assump:betaholder} can be verified directly following the original steps in \cite{hu2022smooth}, thus we omit it for simplicity. For instance $i$, the two optimal regions are
\[
\mathcal R_1=\cX,
\qquad
\mathcal R_2=\mathcal X_0\cup\mathcal B(\bx_i,1/(4q)).
\]

We verify Assumption~\ref{assump:c0r0regular}. The selected cells are separated by order $m^{-1/d}$, so a standard packing argument implies that any ball $\mathcal B(\bx,r)$ intersects at most $C_d(1+mr^d)$ selected cells. For $r\leq c/q$, by taking $q_0$ sufficiently large, such a ball intersects at most one selected cell. After rescaling by $q$, the required lower volume fraction is therefore bounded below by a positive constant depending only on $d$ and $c$. For $c/q<r\leq1/2$, we have
\[
\operatorname{Leb}\!\left(\mathcal B(\bx,r)\cap\bigcup_{j=1}^m\mathcal X_j\right)
\leq C_dq^{-d}(1+mr^d).
\]
Since $mq^{-d}=O(q^{-\alpha\beta})$ and
$\operatorname{Leb}(\mathcal B(\bx,r)\cap[0,1]^d)\geq2^{-d}\operatorname{Leb}(\mathcal B(\bx,r))$ for $\bx\in[0,1]^d$ and $r\leq1/2$, choosing first $c$ and then $q_0$ sufficiently large yields constants $c_*>0$ and $r_*>0$, independent of $q$, $T$, and $i$, such that both $\mathcal R_1$ and $\mathcal R_2$ are $(c_*,r_*)$-regular. Thus Assumption \ref{assump:c0r0regular} holds uniformly over the constructed instances.
\begin{itemize}
   \item Assumption \ref{assum:margin}. The margin probability can be expressed as
\begin{align*}
&\mathbb{P}\left(0<\left|\eta_{1,i}(X)-\frac{1}{2}\right| \leq t\right) \\
&= \sum_{j\neq i}
\int_{\mathcal{B}\left(\bx_j, \frac{1}{4q}\right)}
\mathbb{I}\left\{0<\phi\left[q\left(\bx-\bx_j\right)\right]\leq tq^\beta\right\}
\frac{\omega}{\operatorname{Leb}\left[\mathcal{B}\left(0,\frac{1}{4q}\right)\right]}
\,d\bx \\
&= \frac{(m-1)\omega}
{\operatorname{Leb}\left[\mathcal{B}\left(0,\frac{1}{4}\right)\right]}
\int_{\mathcal{B}\left(0,\frac{1}{4}\right)}
\mathbb{I}\left\{\phi(y)\leq tq^\beta\right\}\,dy \\
&= (m-1)\omega\cdot
\mathbb{I}\left\{t\geq C_\phi q^{-\beta}\right\}.
\end{align*}
Since $(m-1)\omega\leq 2q^{-\alpha\beta}$,
Assumption \ref{assum:margin} is satisfied with
$C_0=2C_\phi^{-\alpha}$.
\item Assumption~\ref{assump:Qxliplower}.
Let
$\Delta_{\max}:=\max_{k\in\cK}\Delta_{k,\max}$.
For a sufficiently small fixed $C_\phi$,
\[
\Delta_{\max}
\leq
C_\phi\left(4C_\phi\frac{T}{C}\right)^{-\frac{\beta}{\beta+d}}
\leq
T^{-\frac{\beta}{2\beta+d}}
+\left(\frac CT\right)^{\frac{\beta}{\beta+d}}.
\]
Hence $\Delta_{k,\max}\leq\Delta_{\max}\leq\rho$ for every arm,
verifying Assumption~\ref{assump:Qxliplower}.

\end{itemize}

Write the fairness failure probability as $\delta=\widetilde O(1/T)$. For each $i\in\{1,\ldots,m\}$, define a separate attack against instance $\eta_{1,i}$ as follows. Let
\[
 B_{t-1}^{(i)}:=\sum_{s=1}^{t-1}|c_s^{(i)}|
\]
and set
\begin{equation}\label{eq:robust-smooth-truncated-attack}
 c_t^{(i)}
 :=\varphi_i(\bx_t)\mathbb I\{\pi_t=1\}
   \mathbb I\!\left\{B_{t-1}^{(i)}+\varphi_i(\bx_t)\leq C\right\}.
\end{equation}
When the attack is not truncated, the observed mean of arm 1 is
\[
 \eta_{1,i}(\bx)+\varphi_i(\bx)
 =\frac12+\sum_{j=1}^m\varphi_j(\bx),
\]
which is common to all $i$. By construction, \eqref{eq:robust-smooth-truncated-attack} satisfies $\sum_{t=1}^T|c_t^{(i)}|\leq C$ on any sample path.

Let
\[
 N_i:=\sum_{t=1}^T\mathbb I\{\bx_t\in\mathcal X_i\},
 \qquad
 \cH_i:=\{N_i\leq4T\omega\},
 \qquad
 \cH:=\bigcap_{i=1}^m\cH_i.
\]
Since  $|\varphi_i|\leq C_\phi q^{-\beta}$,
\[
 4T\omega C_\phi q^{-\beta}
 =4TC_\phi q^{-(\beta+d)}\leq C.
\]
Thus no attack is truncated on $\cH$. With $N_i\sim\operatorname{Bin}(T,\omega)$, a Chernoff bound and a union bound give
\begin{equation}\label{eq:robust-smooth-truncation-probability}
 \PP(\cH^c)\leq m\exp(-cT\omega)
\end{equation}
for an absolute constant $c>0$.

Couple all $m$ experiments by using the same contexts, noises, and internal policy randomization. On $\cH$, all attacked experiments have the same observed conditional mean
\(
\frac12+\sum_{j=1}^m\varphi_j
\)
for arm 1 and the same mean $1/2$ for arm 2. Their observed histories can therefore be coupled to coincide. For each $i$, let $\cE_i$ be the global fairness event for the true instance $\eta_{1,i}$. Because the two arms tie throughout $\mathcal X_i$, fairness on $\cE_i$ requires
\[
 \PP(\pi_t=1\mid\bx_t=\bx,\cF_{t-1})
 =\PP(\pi_t=2\mid\bx_t=\bx,\cF_{t-1}),
 \qquad \bx\in\mathcal X_i,\quad t\in[T].
\]
Applying a union bound to $\cE_1,\ldots,\cE_m$ and using \eqref{eq:robust-smooth-truncation-probability} gives
\begin{equation}\label{eq:robust-smooth-all-fair-events}
 \PP\left(\cH\cap\bigcap_{i=1}^m\cE_i\right)
 \geq1-m\delta-m\exp(-cT\omega).
\end{equation}

Now evaluate regret on the true instance $i=1$. On each cube
$\mathcal X_j$, $j\neq1$, arm $1$ is better than arm $2$ by
$\varphi_j(\bx)$. On
$\mathcal H\cap\bigcap_{i=1}^m\mathcal E_i$, however, all coupled
experiments share the same observed history, while $\mathcal E_j$
requires equal selection probabilities on $\mathcal X_j$. The common
policy therefore selects each arm with probability $1/2$ on
$\mathcal X_j$. By \eqref{eq:robust-smooth-all-fair-events}, this joint
event has probability at least
$1-m\delta-m\exp(-cT\omega)$. For each $j\in\{1,\ldots,m\}$, conditional on $\bx_t\in\mathcal X_j$, the context is uniformly distributed on $\mathcal B(\bx_j,1/(4q))$. Hence, by the change of variables $u=q(\bx-\bx_j)$,
\[
\EE[\varphi_j(\bx_t)\mid\bx_t\in\mathcal X_j]
=q^{-\beta}\frac{1}{\operatorname{Leb}(\mathcal B(0,1/4))}
\int_{\mathcal B(0,1/4)}\phi(u)\,du
=:c_\phi q^{-\beta},
\]
where $c_\phi>0$ because $\phi$ is nonnegative and not identically zero.

Define $Z_t:=\sum_{j=2}^m\varphi_j(\bx_t)\mathbb I\{\bx_t\in\mathcal X_j\}$. Since the sets $\mathcal X_j$ are disjoint, $0\le Z_t\le C_\phi q^{-\beta}$, and
\[
\EE Z_t=\sum_{j=2}^m\PP(\bx_t\in\mathcal X_j)\EE[\varphi_j(\bx_t)\mid\bx_t\in\mathcal X_j]
=(m-1)\omega c_\phi q^{-\beta}.
\]
Moreover, $Z_1,\ldots,Z_T$ are i.i.d. Define $\widetilde Z_t:=Z_t/(C_\phi q^{-\beta})\in[0,1]$. Then
$\EE\widetilde Z_t=c_\phi(m-1)\omega/C_\phi$.
A Chernoff bound therefore gives
\[
\PP\left(\sum_{t=1}^T Z_t<\frac{c_\phi}{2}T(m-1)\omega q^{-\beta}\right)
\leq \exp(-c'T(m-1)\omega)
\leq \exp(-c''Tm\omega),
\]
for positive constants $c',c''$, where the last inequality uses $m\geq2$ for sufficiently large $T$. Thus, defining
\[
\mathcal L:=\left\{\sum_{t=1}^T\sum_{j=2}^m\varphi_j(\bx_t)\mathbb I\{\bx_t\in\mathcal X_j\}
\geq \frac{c_\phi}{2}T(m-1)\omega q^{-\beta}\right\},
\]
we have $\PP(\mathcal L^c)\leq \exp(-c''Tm\omega)$.
Because $1\leq C\leq T$ and $q_0$ is fixed, the modified definition still gives
$q\asymp(T/C)^{1/(\beta+d)}$. Moreover, $q\geq q_0$ implies $m\geq2$ and hence
$m-1\asymp m\asymp q^{d-\alpha\beta}$.
Consequently,
\begin{align*}
 R_T
 &\geq \frac{c_\phi}{4}T(m-1)\omega q^{-\beta}
 \left(1-m\delta-m e^{-cT\omega}-e^{-c'Tm\omega}\right)\\
 &=\Omega\left(Tm\omega q^{-\beta}\right)\\
 &=\Omega\left(
 C^{\frac{\alpha\beta+\beta}{\beta+d}}
 T^{\frac{d-\alpha\beta}{\beta+d}}
 \right).
\end{align*}
Under $\delta=\widetilde O(1/T)$ and $\alpha\beta<d$, $m\delta=o(1)$. Proposition \ref{thm:thm3inhusmooth} supplies the uncorrupted lower bound. Since the minimax risk is at least the maximum of the corrupted and uncorrupted constructions and $\max\{u,v\}\geq(u+v)/2$, we obtain
\[
 \inf_{\pi\in\Pi_C^{\mathrm{fair}}}
 \sup_{\PP\in\mathcal P}\sup_{c\in\cC_C}
 R_T(\pi;\PP,c)
 =\Omega\left(
 T^{\frac{\beta+d-\alpha\beta}{2\beta+d}}
 +C^{\frac{\alpha\beta+\beta}{\beta+d}}
 T^{\frac{d-\alpha\beta}{\beta+d}}
 \right).
\]
This proves the first case of Theorem \ref{thm:robustlowerbound}.

\medskip
\noindent\textit{Case $\alpha\beta\geq d$.}
Let $\cA$, $\cB$, and $\cD$ be three separated compact regular sets and put $\cX=\cA\cup\cB\cup\cD$. Fix $\Delta\in(0,1/4)$ and choose $p>0$ so that $0<p\leq\min\{1/4,C_0\Delta^\alpha/2\}$. Let the context distribution assign mass $p$ to each of $\cA$ and $\cB$, mass $1-2p$ to $\cD$, and have a density bounded above and below on $\cX$. Choose smooth functions $g_A,g_B:\cX\to[0,1]$ satisfying
\[
g_A=1\text{ on }\cA,\quad g_A=0\text{ on }\cB\cup\cD,\qquad
g_B=1\text{ on }\cB,\quad g_B=0\text{ on }\cA\cup\cD.
\]
Such smooth functions exist because the three compact components are separated by a positive distance. Consider the two instances
\[
\mathsf P_A:\ f_1^A=\tfrac12+\Delta g_B,\ f_2^A=\tfrac12,\qquad
\mathsf P_B:\ f_1^B=\tfrac12+\Delta g_A,\ f_2^B=\tfrac12.
\]
Under $\mathsf P_A$, the arms tie on $\cA\cup\cD$ and arm 1 is better by $\Delta$ on $\cB$. Under $\mathsf P_B$, they tie on $\cB\cup\cD$ and arm 1 is better by $\Delta$ on $\cA$.

Set $N_C=\lfloor C/\Delta\rfloor\wedge T$. Up to time $N_C$, attack $\mathsf P_A$ by adding $\Delta$ to the reward of arm 1 on $\cA$, and attack $\mathsf P_B$ by adding $\Delta$ to the reward of arm 1 on $\cB$. Thus,
\begin{align*}
c_t^A&=\Delta\mathbb I\{t\leq N_C,\pi_t=1,\bx_t\in\cA\},\\
c_t^B&=\Delta\mathbb I\{t\leq N_C,\pi_t=1,\bx_t\in\cB\}.
\end{align*}
Each attack incurs corruption of at most $\Delta$ in each of the first $N_C$ rounds, and therefore, pathwise,
\[
\sum_{t=1}^T|c_t^A|\leq \Delta N_C\leq C,\qquad
\sum_{t=1}^T|c_t^B|\leq \Delta N_C\leq C.
\]
Thus both attacks belong to $\cC_C$.

To see the effect of the attacks, consider the conditional mean of the observed reward when arm 1 is selected. Under the attacked $\mathsf P_A$, it equals
\[
\tfrac12+\Delta\text{ on }\cA\cup\cB,\qquad \tfrac12\text{ on }\cD.
\]
Indeed, the attack raises the mean on $\cA$ from $1/2$ to $1/2+\Delta$, while the uncorrupted mean on $\cB$ is already $1/2+\Delta$. Under the attacked $\mathsf P_B$, exactly the same conditional means are obtained, with the roles of $\cA$ and $\cB$ reversed. Arm 2 has conditional mean $1/2$ everywhere in both instances. Hence the two attacked environments induce the same conditional reward law for any context and action during the first $N_C$ rounds. By using the same contexts, reward noises, and internal policy randomization, their observed histories can therefore be coupled to coincide through round $N_C$.

Let $\cE_A$ and $\cE_B$ denote the fairness events through round $N_C$ for the attacked instances $\mathsf P_A$ and $\mathsf P_B$, respectively. On $\cE_A$, the two arms have equal true rewards on $\cA\cup\cD$. The fairness constraint therefore requires equal selection probabilities there. Since there are two arms,
\[
\PP(\pi_t=1\mid\bx_t=\bx,\cF_{t-1})
=\PP(\pi_t=2\mid\bx_t=\bx,\cF_{t-1})=1/2,\qquad \bx\in\cA\cup\cD.
\]
Likewise, on $\cE_B$ the two arms must be selected with equal probability on $\cB\cup\cD$. Because the two attacked instances share the same observed history under the coupling, on $\cE_A\cap\cE_B$ these two requirements apply to the same policy. Consequently,
\[
\PP(\pi_t=1\mid\bx_t=\bx,\cF_{t-1})
=\PP(\pi_t=2\mid\bx_t=\bx,\cF_{t-1})=1/2
\]
for every $\bx\in\cA\cup\cB\cup\cD$ and all $t\leq N_C$. Since each fairness event fails with probability at most $\delta$, a union bound gives $\PP(\cE_A\cap\cE_B)\geq1-2\delta$.

We now evaluate regret under the true instance $\mathsf P_A$. On $\cB$, arm 1 is better than arm 2 by $\Delta$, whereas on $\cE_A\cap\cE_B$ the policy selects arm 2 with probability $1/2$. Moreover, we have
\[
\PP(\bx_t\in\cB,\cE_A\cap\cE_B)
\geq \PP(\bx_t\in\cB)-\PP((\cE_A\cap\cE_B)^c)
\geq p-2\delta.
\]
Therefore, for all $t\leq N_C$, the expected instantaneous regret is at least $\Delta(p-2\delta)/2$. Summing over the first $N_C$ rounds gives
\[
R_T(\pi;\mathsf P_A,c^A)\geq \frac{\Delta}{2}N_C(p-2\delta).
\]
Here $\Delta$ and $p$ are fixed positive constants and $\delta=\widetilde O(1/T)$. Hence $p-2\delta$ is bounded away from zero for sufficiently large $T$. Moreover, since $1\leq C\leq T$ and $\Delta$ is fixed,
\[
\Delta N_C=\Delta(\lfloor C/\Delta\rfloor\wedge T)=\Omega(C).
\]
Consequently,
\[
R_T(\pi;\mathsf P_A,c^A)=\Omega(C).
\]

It remains to verify membership in the problem class. Because $\cA$, $\cB$, and $\cD$ are compact regular sets separated by positive distances, we may choose the context density to be bounded above and below on each component. This gives Assumptions~\ref{assump:iiddensity} and \ref{assump:c0r0regular}. Furthermore, $g_A$ and $g_B$ are smooth and fixed; by taking the constant $\Delta$ sufficiently small, the functions $f_1^A,f_1^B,f_2^A,f_2^B$ satisfy the requirement in Assumption~\ref{assump:betaholder}.

For the margin condition, the positive reward gap takes only the value $\Delta$. In either instance, it occurs on exactly one component of probability mass $p$. Thus, for every $u\geq0$,
\[
\PP(0<|f_1(\bx)-f_2(\bx)|\leq u)=p\,\mathbb I\{u\geq\Delta\}.
\]
If $u<\Delta$, this probability is zero. If $u\geq\Delta$, the choice $p\leq C_0\Delta^\alpha/2$ implies $p\leq C_0\Delta^\alpha\leq C_0u^\alpha$. Hence Assumption~\ref{assum:margin} holds.

Finally, consider Assumption~\ref{assump:Qxliplower}. Choose $\mathfrak r$ smaller than the distance between any two of the disjoint components $\cA,\cB,\cD$. Under either $\mathsf P_A$ or $\mathsf P_B$, any point outside an arm's optimal region is then at distance greater than $\mathfrak r$ from that region. Consequently, the set quantified over in Assumption~\ref{assump:Qxliplower} is empty, and the boundary-growth condition holds vacuously. Thus both instances satisfy Assumption~\ref{assump:Qxliplower}.

This proves the $\Omega(C)$ lower bound for $\alpha\beta\geq d$. At the boundary $\alpha\beta=d$, Proposition~\ref{thm:thm3inhusmooth} additionally gives the uncorrupted lower bound, so combining the two constructions yields the stated lower bound in that case. \hfill\Halmos

\subsection{Proofs in Section \ref{EC_pf_fairbandit}}\label{subsec:ECproofinEC_pf_fairbandit}

The following lemmas are used to analyze the initial estimator based on
the random-exploration sample. Lemma \ref{lem:olstail} is adapted from
the martingale concentration argument used in the proof of Proposition
EC.1 of \cite{bastani2020online}, and Lemma
\ref{lem:minimaleigenvalue} follows from Lemmas EC.22 and EC.23 therein.
Lemma \ref{prop:betahatadapt} converts the resulting parameter-error
bound into a uniform prediction-error bound. The adaptive all-sample
estimator is treated separately below using the predictable-design
argument in Lemma \ref{lem:predictableols}.

\begin{lemma}\label{lem:olstail}
Let ${Z_s : s \in \cJ}$ be a collection of random vectors and assume that each element in $Z_s$ lies in $[-r,r]$. 
Assume entries of the random noise vector $\bepsilon\in\RR^{|\mathcal{J}|}$ are independent $\sigma$-sub-Gaussian random variables. If $|\mathcal{J}|$ is a fixed positive integer, then for all constants $\tilde{\chi}>0$ and $\phi>0$, we have
\begin{align*}
    \PP(\|\hat{\beta}_k(\mathcal{J})-\beta_k\|_2\geq \tilde{\chi})\leq \exp{(-\tilde{C}|\mathcal{J}|\tilde{\chi}^2+\log 2d)}+\PP(\lambda_{\min}(\hat{\Sigma}(\mathcal{J}))\leq \phi^2),
\end{align*}
where $\tilde{C}=\frac{\phi^4}{2dr^2\sigma^2}$ and $\hat{\Sigma}(\mathcal{J})=\frac{1}{|\mathcal{J}|}\sum_{s\in\mathcal{J}}Z_sZ_s^{\mathrm{T}}$. 
\end{lemma}

\begin{lemma}\label{lem:minimaleigenvalue}
Assume the conditions in Lemma \ref{lem:olstail} hold. For a subset $\mathcal{J}^{\prime}\subseteq\mathcal{J}$ such that $\{\bz_t|t\in\mathcal{J}^{\prime}\}$ is an i.i.d. sample drawn from a distribution $\mathsf{P}_Z$, with $\lambda_{\min }(E_{\bz\sim\mathsf{P}_Z}[\bz\bz^\mathcal{T}])\geq\phi_1^2$ and $|\mathcal{J}^{\prime}|/|\mathcal{J}|\geq p/2$ for positive constants $\phi_1$ and $p$, if $|\mathcal{J}|$ is a fixed positive integer, it holds that
    \begin{align*}
        \PP\left[\lambda_{\min }(\hat{\Sigma}(\mathcal{J})) \leq \frac{\phi_1^2 p}{4}\right] \leq \exp \left[-p \tilde{C}\left(\phi_1\right)|\mathcal{J}| / 2+\log d\right],
    \end{align*}
    where $\tilde{C}\left(\phi_1\right):=\min \left(\frac{1}{2}, \frac{\phi_1^2}{8 r^2}\right)$.
\end{lemma}

\begin{lemma}\label{prop:betahatadapt} 
Assume that each element in $Z$ lies in $[-r,r]$ and that each entry of the random noise vector $\bepsilon\in\RR^{|\mathcal{J}|}$ is independent $\sigma$-sub-Gaussian random variables. If $|\mathcal{J}|$ is a fixed positive integer, then for all constants $\chi>0$, $\phi>0$, we have
\begin{align*}
    \PP(\max_{\|a\|_2\leq \sqrt{dr^2}}|(\hat{\beta}_k(\mathcal{J})-\beta_k)^{\mathrm{T}}a|\geq \chi)\leq \exp{(-D_3|\mathcal{J}|\chi^2+\log 2d)}+\PP(\lambda_{\min}(\hat{\Sigma}(\mathcal{J}))\leq \phi^2),
\end{align*}
where $D_3=\frac{\phi^4}{2d^2r^4\sigma^2}$ and $\hat{\Sigma}(\mathcal{J})=\frac{1}{|\mathcal{J}|}\sum_{s\in\mathcal{J}}Z_sZ_s^{\mathrm{T}}$. 
\end{lemma}

\proof{Proof of Lemma \ref{prop:betahatadapt}.}
By the Cauchy-Schwarz inequality, for any $a\in \RR^d$, it holds that
\begin{align*}
|(\hat{\beta}_k(\mathcal{J})-\beta_k)^{\mathrm{T}}a|\leq \|\hat{\beta}_k(\mathcal{J})-\beta_k\|_2\|a\|_2, 
\end{align*}
which leads to
\begin{align}\label{eq:A.3}
   \PP(\max_{\|a\|_2\leq \sqrt{dr^2}}|(\hat{\beta}_k(\mathcal{J})-\beta_k)^{\mathrm{T}}a|\geq \chi)&\leq   \PP(\|\hat{\beta}_k(\mathcal{J})-\beta_k\|_2\max_{\|a\|_2\leq \sqrt{dr^2}}\|a\|_2\geq \chi)\nonumber\\
   &=\PP(\|\hat{\beta}_k(\mathcal{J})-\beta_k\|_2\geq \chi/\max_{\|a\|_2\leq \sqrt{dr^2}}\|a\|_2)\nonumber\\
   &\leq \PP(\|\hat{\beta}_k(\mathcal{J})-\beta_k\|_2\geq \chi/\sqrt{dr^2}).
\end{align}

According to Lemma \ref{lem:olstail}, we can conclude that
\begin{align*}
    \PP(\|\hat{\beta}_k(\mathcal{J})-\beta_k\|_2\geq \chi/\sqrt{dr^2})\leq \exp{\left(-\frac{\tilde{C}}{dr^2}|\mathcal{J}|{\chi}^2+\log 2d\right)}+\PP(\lambda_{\min}(\hat{\Sigma}(\mathcal{J}))\leq \phi^2).
\end{align*}
Therefore, by using \eqref{eq:A.3} and replacing $\frac{\tilde{C}}{dr^2}$ with $D_3$, we can write
\begin{align*}
\PP(\max_{\|a\|_2\leq \sqrt{dr^2} }|(\hat{\beta}_k(\mathcal{J})-\beta_k)^{\mathrm{T}}a|\geq \chi)&\leq \exp{(-D_3|\mathcal{J}|{\chi}^2+\log 2d)}+\PP(\lambda_{\min}(\hat{\Sigma}(\mathcal{J}))\leq \phi^2),
\end{align*} 
which finishes the proof.\hfill\Halmos

\endproof

\begin{lemma}[Predictable-design OLS concentration]\label{lem:predictableols}
For $k\in\cK$ and $t\geq1$, define
\[
 G_{k,t}:=\sum_{s=1}^t\mathbb I\{\pi_s=k\}\bz_s\bz_s^{\mathrm T},
 \qquad
 S_{k,t}:=\sum_{s=1}^t\mathbb I\{\pi_s=k\}\bz_s\varepsilon_{k,s}.
\]
Let
\[
 c_{\mathrm G}:=\frac{\widetilde p\lambda^*}{8K},
 \qquad c_{\mathrm S}:=4\sigma r\sqrt d,
\]
and define
\[
\mathcal E_t^{\mathrm{lin}}
:=\cB
\cap\left\{\min_{k\in\cK}\lambda_{\min}(G_{k,t})\geq c_{\mathrm G}t\right\}
\cap\left\{\max_{k\in\cK}\|S_{k,t}\|_2\leq c_{\mathrm S}\sqrt{t\log T}\right\}.
\]
For $C_a\geq C_{a,0}$ and every $t>|\cT_0|$,
\[
\PP\bigl((\mathcal E_t^{\mathrm{lin}})^c\bigr)\leq8KT^{-4}.
\]
Moreover, on $\mathcal E_t^{\mathrm{lin}}$,
\[
 \max_{k\in\cK}\|G_{k,t}^{-1}S_{k,t}\|_2
 \leq \frac{c_{\mathrm S}}{c_{\mathrm G}}\sqrt{\frac{\log T}{t}}.
\]
\end{lemma}

\textit{Proof.}
The indicator $\mathbb I\{\pi_s=k\}$ and $\bz_s$ are measurable before the reward noise is observed. Hence, together with Assumption \ref{assum_parabound}, for each coordinate $j$, the sequence
$\{\mathbb I\{\pi_s=k\}z_{s}^{(j)}\varepsilon_{k,s}\}_{s\geq1}$
is a conditionally mean-zero $\sigma r$-sub-Gaussian martingale
difference sequence. By the Azuma--Hoeffding inequality for martingale differences, for any $u>0$,
\[
\PP\left(\left|S_{k,t}^{(j)}\right|\geq u\right)
\leq
2\exp\left(-\frac{u^2}{2\sigma^2r^2t}\right).
\]
Taking $u=4\sigma r\sqrt{t\log T}$ yields
\[
\PP\left(
\left|S_{k,t}^{(j)}\right|
>
4\sigma r\sqrt{t\log T}
\right)
\leq 2T^{-8}.
\]
A union bound over $j\in[d]$ and $k\in\cK$, together with $T>2d$, therefore yields
\[
 \PP\left(\max_{k\in\cK}\|S_{k,t}\|_2>
 c_{\mathrm S}\sqrt{t\log T}\right)\leq KT^{-4}.
\]
It remains to lower-bound the Gram matrix. During the initial exploration phase, the arm assignment is uniform and independent of the context. Hence, for each $k\in\cK$,
\[
\mathbb E[\mathbb I\{\pi_s=k,\bz_s\in Q_k\}\bz_s\bz_s^{\mathrm T}]
=\frac{1}{K}\mathbb P(\bz_s\in Q_k)\mathbb E[\bz_s\bz_s^{\mathrm T}\mid\bz_s\in Q_k].
\]
By Assumption~\ref{assum_excite}, $\mathbb P(\bz_s\in Q_k)\geq\widetilde p$ and $\lambda_{\min}(\mathbb E[\bz_s\bz_s^{\mathrm T}\mid\bz_s\in Q_k])\geq\lambda^*$. Therefore,
\[
\lambda_{\min}(\mathbb E[\mathbb I\{\pi_s=k,\bz_s\in Q_k\}\bz_s\bz_s^{\mathrm T}])
\geq \widetilde p\lambda^*/K.
\]
Define $Y_{s,k}:=\mathbb I\{\pi_s=k,\bz_s\in Q_k\}\bz_s\bz_s^{\mathrm T}$. The matrices $\{Y_{s,k}:s\in\cT_0\}$ are independent and positive semidefinite, and $\lambda_{\max}(Y_{s,k})\leq\|\bz_s\|_2^2\leq dr^2$. Moreover,
\[
\lambda_{\min}\left(\sum_{s\in\cT_0}\mathbb E[Y_{s,k}]\right)
\geq \widetilde p\lambda^*|\cT_0|/K.
\]
Applying the matrix Chernoff inequality gives
\[
\mathbb P\left(\lambda_{\min}\left(\sum_{s\in\cT_0}Y_{s,k}\right)
< \frac{\widetilde p\lambda^*}{2K}|\cT_0|\right)
\leq d\exp\left(-\frac{\widetilde p\lambda^*}{8Kdr^2}|\cT_0|\right).
\]
Since $|\cT_0|=C_a\log T$, $C_a\geq C_{a,0}$, and
$T>2d$,
\[
d\exp\left(-\frac{\widetilde p\lambda^*}{8Kdr^2}|\cT_0|\right)
\leq T^{-4}.
\]

Finally,
\[
G_{k,|\cT_0|}=\sum_{s\in\cT_0}\mathbb I\{\pi_s=k\}\bz_s\bz_s^{\mathrm T}
\succeq \sum_{s\in\cT_0}Y_{s,k}.
\]
Thus, taking a union bound over $k\in\cK$, with probability at least $1-KT^{-4}$,
\[
\lambda_{\min}(G_{k,|\cT_0|})\geq \widetilde p\lambda^*|\cT_0|/(2K),\qquad k\in\cK.
\]
After the exploration phase, on $\cB$, whenever $\bz_s\in Q_k$ the initial estimator separates arm $k$ from all other arms by at least $h/2$, so Algorithm \ref{alg:fairols} selects arm $k$. Therefore,
\[
 G_{k,t}\succeq G_{k,|\cT_0|}
 +\sum_{s=|\cT_0|+1}^t\mathbb I\{\bz_s\in Q_k\}\bz_s\bz_s^{\mathrm T}.
\]
The matrices in the second sum are i.i.d., bounded, and have expectation whose minimum eigenvalue is at least $\widetilde p\lambda^*$. For $s>|\cT_0|$, define $Y_{s,k}:=\mathbb I\{\bz_s\in Q_k\}\bz_s\bz_s^{\mathrm T}$. Since the contexts are i.i.d., the matrices $\{Y_{s,k}:s>|\cT_0|\}$ are independent and positive semidefinite. Moreover, by Assumption~\ref{assum_excite},
\[
\lambda_{\min}(\mathbb E[Y_{s,k}])
=\mathbb P(\bz_s\in Q_k)\lambda_{\min}(\mathbb E[\bz_s\bz_s^{\mathrm T}\mid\bz_s\in Q_k])
\geq \widetilde p\lambda^*,
\]
while $\lambda_{\max}(Y_{s,k})\leq\|\bz_s\|_2^2\leq dr^2$. Therefore, when $t>2|\cT_0|$, the matrix Chernoff inequality gives
\[
\lambda_{\min}\left(\sum_{s=|\cT_0|+1}^t Y_{s,k}\right)
\geq \frac{\widetilde p\lambda^*}{2}(t-|\cT_0|)
\]
Here
$t-|\cT_0|>|\cT_0|=C_a\log T$. Thus the same condition
$C_a\geq C_{a,0}$ gives the stated simultaneous failure probability
$2KT^{-4}$.

We now combine the initial and post-exploration bounds. If $|\cT_0|<t\leq2|\cT_0|$, then $G_{k,t}\succeq G_{k,|\cT_0|}$, and hence
\[
\lambda_{\min}(G_{k,t})\geq \frac{\widetilde p\lambda^*}{2K}|\cT_0|
\geq \frac{\widetilde p\lambda^*}{4K}t\geq c_{\mathrm G}t.
\]
If instead $t>2|\cT_0|$, then $t-|\cT_0|\geq t/2$. On $\cB$, any context in $Q_k$ forces the selection of arm $k$, so
\[
G_{k,t}\succeq \sum_{s=|\cT_0|+1}^t \mathbb I\{\bz_s\in Q_k\}\bz_s\bz_s^{\mathrm T}.
\]
Consequently,
\[
\lambda_{\min}(G_{k,t})\geq \frac{\widetilde p\lambda^*}{2}(t-|\cT_0|)
\geq \frac{\widetilde p\lambda^*}{4}t\geq c_{\mathrm G}t.
\]
Thus, in both cases, $\min_{k\in\cK}\lambda_{\min}(G_{k,t})\geq c_{\mathrm G}t$.

Proposition \ref{prop:randomtailbound} gives
$\PP(\cB^c)\leq4KT^{-4}$. The score bound above contributes
$KT^{-4}$, while the two matrix-Chernoff bounds contribute
$KT^{-4}$ and $2KT^{-4}$. Hence a union bound gives
\[
\PP\bigl((\mathcal E_t^{\mathrm{lin}})^c\bigr)\leq8KT^{-4}.
\]
Finally, on $\mathcal E_t^{\mathrm{lin}}$, $G_{k,t}$ is invertible and the OLS normal equations give $\widehat\beta_k(\cI_{k,t})-\beta_k=G_{k,t}^{-1}S_{k,t}$. Hence
\[
\|G_{k,t}^{-1}S_{k,t}\|_2
\leq \|G_{k,t}^{-1}\|_{\mathrm{op}}\|S_{k,t}\|_2
\leq \frac{1}{c_{\mathrm G}t}\cdot c_{\mathrm S}\sqrt{t\log T}
=\frac{c_{\mathrm S}}{c_{\mathrm G}}\sqrt{\frac{\log T}{t}},
\]
uniformly over $k\in\cK$, which proves the last statement. \hfill\Halmos

\subsubsection{Proof of Proposition \ref{prop:randomtailbound}} 

Fix an arm $k$; we write it as arm $1$ in this proof. Let $\mathcal I'_{k,0}=\{t\in\mathcal I_{k,0}:\bz_t\in Q_1\}$. Conditional on its cardinality, $\{\bz_t:t\in\mathcal I'_{k,0}\}$ is an i.i.d. sample from $\PP_{Z\mid Z\in Q_1}$, and
\begin{align*}
    |\mathcal{I}_{k,0}^{\prime}|=\sum_{t\in\mathcal{T}_{0}}\mathbb{I}(\bz_t\in Q_1)\mathbb{I}(t\in\mathcal{I}_{k,0}).
\end{align*}
By Assumption \ref{assum_excite} and the random exploration execution of Algorithm \ref{alg:fairols} before $|\mathcal{T}_{0}|$, it holds that $\PP(\bz_t\in Q_1,t\in\mathcal{I}_{k,0})\geq \tilde{p}/K$ for all $t\in \mathcal{T}_{0}$. Then we can obtain that $\mathbb{E}(|\mathcal{I}_{k,0}^{\prime}|)\geq |\mathcal{T}_{0}|\tilde{p}/K$. Using Hoeffding's inequality, the size of $\mathcal{I}_{k,0}^{\prime}$ satisfies
\begin{align*}
 \PP\left[\left|\mathcal{I}_{k, 0}^{\prime}\right|\leq \frac{\tilde{p}}{2K}\left|\mathcal{T}_{0}\right|\right]\leq e^{-\frac{\tilde{p}^2}{2K^2}\left|\mathcal{T}_{0}\right|}.
\end{align*}
Since $\mathcal{I}_{k,0}\subseteq \mathcal{T}_0$, we can infer that
\begin{align*}
 \PP\left[\left|\mathcal{I}_{k, 0}^{\prime}\right|\leq \frac{\tilde{p}}{2K}\left|\mathcal{I}_{k,0}\right|\right]\leq  \PP\left[\left|\mathcal{I}_{k, 0}^{\prime}\right|\leq \frac{\tilde{p}}{2K}\left|\mathcal{T}_{0}\right|\right]\leq e^{-\frac{\tilde{p}^2}{2K^2}\left|\mathcal{T}_{0}\right|}.
\end{align*}

As stated in Assumption \ref{assum_excite}, we have that $\min_{i\in\mathcal{K}}\lambda_{\min}\{\mathbb{E}(\bz_t\bz_t^{\mathrm{T}}|\bz_t\in Q_i)\}\geq\lambda^*$, which implies $\lambda_{\min}\{\mathbb{E}_{\bz\sim\PP_{Z|Z\in Q_1}}(\bz\bz^{\mathrm{T}})\}\geq\lambda^*$. Define the event $\mathcal{A}=\{|\mathcal{I}_{k,0}|\geq \frac{1}{2K}|\mathcal{T}_{0}| \}$. Then according to Lemma \ref{lem:minimaleigenvalue}, when letting $\mathcal{J}=\mathcal{I}_{k,0}$, $\mathcal{J}^{\prime}=\mathcal{I}_{k,0}^{\prime}$, $p=\tilde{p}/K$ and $\phi_1^2=\lambda^*$, 
it follows that
    \begin{align}\label{eq:sigmahateigen}
    &\PP\left[\lambda_{\min }(\hat{\Sigma}(\mathcal{I}_{k,0})) \leq \frac{\lambda^{*} \tilde{p}}{4K}\mid \cA\right]\nonumber\\
    =&\sum_{n\geq\lceil|\mathcal T_0|/(2K)\rceil}
    \PP\left[\lambda_{\min }(\hat{\Sigma}(\mathcal{I}_{k,0}))
    \leq \frac{\lambda^{*}\tilde p}{4K}\mid |\mathcal I_{k,0}|=n\right]
    \PP\left(|\mathcal I_{k,0}|=n\mid\cA\right)\nonumber\\
    \leq&\sup_{n\geq\lceil|\mathcal T_0|/(2K)\rceil}
    \PP\left[\lambda_{\min }(\hat{\Sigma}(\mathcal{I}_{k,0}))
    \leq \frac{\lambda^{*}\tilde p}{4K}\mid |\mathcal I_{k,0}|=n\right]\nonumber\\
    \leq &\exp \left[-\frac{\tilde{p} }{2K^2}\tilde{C}\left(\sqrt{\lambda^*}\right)|\mathcal{T}_{0}| / 2+\log d\right]+ \exp \left[{-\frac{\tilde{p}^2}{2K^2}\left|\mathcal{T}_{0}\right|}\right] \nonumber\\
    \leq &2T^{-4},
    \end{align}
where the last inequality follows from
$C_a\geq C_{a,0}$ and the definition of $C_{a,0}$.

 Under event $\mathcal{A}$, by plugging \eqref{eq:sigmahateigen} into Lemma \ref{prop:betahatadapt}, we have that,
\begin{align}\label{eq:maxbetakleqexp}
&\PP(\max_{\bx\in\cX}|(\hat{\beta}_k(\mathcal{I}_{k,0})-\beta_k)^{\mathrm{T}}\bz|\geq \chi)\nonumber\\\leq &\exp{(- \frac{D_1}{2K}|\mathcal{T}_{0}|{\chi}^2+\log 2d)}+\PP(\lambda_{\min}(\hat{\Sigma}(\mathcal{I}_{k,0}))\leq \frac{\lambda^{*} \tilde{p}}{4K})\nonumber\\
\leq & \exp{(- \frac{D_1}{2K}C_a\log T{\chi}^2+\log 2d)}+2T^{-4}.
\end{align}
By letting $\chi=\frac{h}{4}$, it follows that, under event $\mathcal{A}$,
$C_a\geq C_{a,0}$ gives
\begin{align*}
    \PP(\max_{\bx\in\cX}|(\hat{\beta}_k(\mathcal{I}_{k,0})-\beta_k)^{\mathrm{T}}\bz|\geq \frac{h}{4})\leq 3T^{-4}.
\end{align*}

Recall that for all $t\in\mathcal{T}_0$, the algorithm randomly pull arms. Thus, using the 
Hoeffding's inequality again,
\begin{align}\label{eq:probeventA}
 \PP\left[\mathcal{A}\right]&\geq 1-e^{-\frac{1}{2K^2}\left|\mathcal{T}_{0}\right|} \geq 1-\frac{1}{T^4}.
\end{align}
By the union bound, we have for all $k\in\mathcal{K}$,
\begin{align*}
    \PP(\max_{\bx\in\cX}|(\hat{\beta}_k(\mathcal{I}_{k,0})-\beta_k)^{\mathrm{T}}\bz|\geq \frac{h}{4})\leq 4T^{-4},
\end{align*}
which further implies that,
\begin{align*}
    \PP(\max_{\bx\in\cX}\max_{k\in\mathcal{K}}|(\hat{\beta}_k(\mathcal{I}_{k,0})-\beta_k)^{\mathrm{T}}\bz|\geq \frac{h}{4})\leq 4KT^{-4}.
\end{align*}
This finishes the proof. \hfill \Halmos

\subsubsection{Proof of Proposition \ref{prop:allsampletailbound}}
First, we have that
\[
 \widehat\beta_k(\cI_{k,t})-\beta_k=G_{k,t}^{-1}S_{k,t}.
\]
On $\mathcal E_t^{\mathrm{lin}}$, uniformly for $\bx\in\cX$,
\begin{align*}
 \left|\bz^{\mathrm T}\bigl(\widehat\beta_k(\cI_{k,t})-\beta_k\bigr)\right|
 &\leq \|\bz\|_2\,\|G_{k,t}^{-1}S_{k,t}\|_2\\
 &\leq \sqrt{dr^2}\,\frac{c_{\mathrm S}}{c_{\mathrm G}}
 \sqrt{\frac{\log T}{t}}.
\end{align*}
Since $t+1\leq2t$ and
$C_b\geq C_{b,0}=2\sqrt{2dr^2}\,c_{\mathrm S}/c_{\mathrm G}$,
\[
\sqrt{dr^2}\,\frac{c_{\mathrm S}}{c_{\mathrm G}}
\sqrt{\frac{\log T}{t}}
\leq
\sqrt{2dr^2}\,\frac{c_{\mathrm S}}{c_{\mathrm G}}
\sqrt{\frac{\log T}{t+1}}
\leq
\frac{C_b}{2}\sqrt{\frac{\log T}{t+1}}
=\frac{\epsilon_{t+1}}2.
\]
Lemma \ref{lem:predictableols} yields
\[
 \PP\left(\max_{\bx\in\cX}\max_{k\in\cK}
 \left|\bz^{\mathrm T}\bigl(\widehat\beta_k(\cI_{k,t})-\beta_k\bigr)\right|
 \geq\frac{\epsilon_{t+1}}2\right)\leq8KT^{-4}.
\]
This proves the proposition. \hfill\Halmos

\subsection{Proofs in Section \ref{ec_pf_nonpara}}\label{ec_pf_lemmasinfair_nonpara}

\subsubsection{Proof of Lemma \ref{lem:gapawayfromoptimal}}

For a fixed $r>0$, define
\[
\mathcal A_r
:=
\bigcup_{k\in\cK}
\bigl(\{k\}\times
\{\bx\in\cX:\operatorname{dist}(\bx,\cR_k)\geq r\}\bigr).
\]
If $\mathcal A_r$ is empty, the statement is immediate. Otherwise,
$\mathcal A_r$ is compact because $\cX$ is compact, $\cK$ is finite, and
each $\cR_k$ is closed. The map $(k,\bx)\mapsto\Delta_k(\bx)$ is continuous
on each component of $\mathcal A_r$. Moreover,
$\Delta_k(\bx)>0$ whenever
$\operatorname{dist}(\bx,\cR_k)\geq r$. Hence
\[
c_r
:=
\min_{(k,\bx)\in\mathcal A_r}\Delta_k(\bx)>0,
\]
which proves the claim. \hfill\Halmos

\subsubsection{Proof of Lemma \ref{lem:RksubsetSqk}}

By the definition of $\mathcal{R}_k$, if $\bx\in \mathcal{R}_k$, we have for any $j\in \cK$,
\begin{align*}
    f^*_k(\bx)-f_j^*(\bx)\geq 0.
\end{align*}
 We prove that $k\in \cK_{q,u(\bx)}$, which implies ${\mathcal{R}}_k\subseteq S_{k,q}$, by induction. First, obviously $k\in \cK_{1,u(\bx)}$. For $h\leq q-1$, assume $k\in \cK_{h,u(\bx)}$, which implies that $\bx\in S_{h,k}$. Then denote $b_h= \arg\max_{j\in\cK_{h,u(\bx)}}\hat{f}_{h,j}(\bx)$, we have that
\begin{align*}
   \hat{f}_{h,b_h}(\bx)-\hat{f}_{h,k}(\bx)&= (\hat{f}_{h,b_h}(\bx)-f^*_{b_h}(\bx))+f^*_{b_h}(\bx)-f^*_k(\bx)-(\hat{f}_{h,k}(\bx)-f^*_k(\bx)) \nonumber\\
   & \leq \left|\hat{f}_{h,b_h}(\bx)-f^*_{b_h}(\bx)\right|+f^*_{b_h}(\bx)-f^*_k(\bx)+\left|\hat{f}_{h,k}(\bx)-f^*_k(\bx)\right| \nonumber\\
   &\leq \frac{1}{2}\epsilon_{h}+\frac{1}{2}\epsilon_{h}\nonumber\\
    & =  \epsilon_{h},
\end{align*}
where the second inequality is obtained by the definition of ${\mathcal{R}}_k$ and event $\mathcal{G}_{h}$. Thus, $\hat{f}_{h,k} (G_{u(\bx)})$ and $\max\limits_{l \in \cK_{h,u(\bx)}}\hat{f}_{h,l}(G_{u(\bx)})$ are $(2\epsilon_{h})$-chained in $\{\hat{f}_{h,k} (G_{u(\bx)}):k\in \cK_{h,u(\bx)}\}$,
which implies that $k \in \cK_{h+1,u(\bx)}$. Therefore, with the induction assumption we have $k\in \cK_{h+1,u(\bx)}$, from which we can conclude that $k\in \cK_{q,u(\bx)}$. \hfill\Halmos

\subsubsection{Proof of Lemma \ref{lem:epsilongreatthan1/2delta}}

Note that $Q$ is the smallest integer such that $\sum_{q=1}^Q\bar n_q\geq T$. By the definition of $|\cT_q|$, for any positive integer $Q_0$,
\begin{align*}
    \sum_{q=1}^{Q_0}\bar n_q&\geq \sum_{q=1}^{Q_0}\frac{2K}{p^*}\left(\frac{4^q\log \left(T \delta_A^{-d}\right)}{C_{K}}\right)^{\frac{2 \beta+d}{2 \beta}}\left(\log T\right)^{\frac{2\beta+d}{\beta^{\prime}-1}-\frac{2\beta+d}{2\beta}}\nonumber\\&\geq \sum_{q=1}^{Q_0}\frac{2K}{p^*}\left(\frac{\log T}{C_{K} }\right)^{\frac{2 \beta+d}{2 \beta}}\left(\log T\right)^{\frac{2\beta+d}{\beta^{\prime}-1}-\frac{2\beta+d}{2\beta}}2^{\frac{2 \beta+d}{ \beta}q}\nonumber\\&\geq \frac{2K}{p^*}(2^{\frac{2 \beta+d}{\beta}Q_0})\left(\log T\right)^{\frac{2\beta+d}{\beta^{\prime}-1}-\frac{2\beta+d}{2\beta}}.
\end{align*}
Let $Q_0=\lceil \frac{\beta}{(2\beta+d)\log2}\log(\frac{Tp^*}{2K}\left(\log T\right)^{-\frac{2\beta+d}{\beta^{\prime}-1}+\frac{2\beta+d}{2\beta}}) \rceil$, we have $ \sum_{q=1}^{Q_0}\bar n_q\geq T$, which leads to that $Q\leq Q_0 \leq \lceil \frac{\beta}{(2\beta+d)\log2}\log(T\left(\log T\right)^{-\frac{2\beta+d}{\beta^{\prime}-1}+\frac{2\beta+d}{2\beta}})\rceil$. It follows that
\begin{align*}
    \epsilon_q\geq \epsilon_Q=2^{-Q}(\log T)^{\frac{\beta^{\prime}-1-2\beta}{2\beta^{\prime}-2}}\geq \frac{1}{2}T^{-\frac{\beta}{2\beta+d}}\geq (1+2L_1\sqrt d)\delta_A.
\end{align*}
\hfill\Halmos

\subsubsection{Proof of Lemma \ref{lem:toosmallallsupport}}

Fix an arm $k$ with
$\Delta_{k,\max}\leq T^{-\beta/(2\beta+d)}$.
We argue by induction. The claim holds at $q=1$ because all arms
are active on $\cX$. Suppose $q\geq2$ and $S_{q-1,k}=\cX$.
Lemma~\ref{lem:epsilongreatthan1/2delta} and the geometric tolerances give
$\epsilon_{q-1}\geq\epsilon_{Q-1}=2\epsilon_Q
\geq T^{-\beta/(2\beta+d)}$.

Fix a cell representative $G_j$ and choose
$b\in\arg\max_{l\in\cK_{q-1,j}}\widehat f_{q-1,l}(G_j)$.
Both $b$ and $k$ are active. On $\overline{\mathcal G}_{q-1}$,
\begin{align*}
0\leq\widehat f_{q-1,b}(G_j)-\widehat f_{q-1,k}(G_j)
&\leq f_b^*(G_j)-f_k^*(G_j)+\epsilon_{q-1}\\
&\leq\Delta_{k,\max}+\epsilon_{q-1}
\leq2\epsilon_{q-1}.
\end{align*}
Thus $k$ is directly $(2\epsilon_{q-1})$-linked to an arm with the
largest estimate and is retained. Singleton active sets remain unchanged.
Since this holds in every cell, $S_{q,k}=\cX$, completing the induction.
\hfill\Halmos

\subsubsection{Proof of Lemma \ref{lem:A.6}}

Fix $k$ and condition on
$\overline{\mathcal G}_{q-1}$, $\overline{\mathcal M}_{q-1}$, and
$N_{q,k}=n$. Conditional further on the history available before epoch
$q$, the active sets are fixed before the contexts in epoch $q$ arrive.
The samples assigned to arm $k$ therefore form an i.i.d. sample from the
corresponding arm-specific design on $S_{q,k}$. The density bounds in
Assumption \ref{assump:iiddensity} and the assumed weak regularity of
$S_{q,k}$ verify the design conditions used in Lemma 16 of
\cite{hu2022smooth}. Applying that lemma and using Lemma
\ref{lem:epsilongreatthan1/2delta} gives the displayed fixed-arm probability bound.  Further averaging over the
 history completes the proof.

\hfill\Halmos

\subsubsection{Proof of Proposition \ref{thm:thm3inhusmooth}}

Assumptions \ref{assump:iiddensity}--\ref{assum:margin} coincide with the corresponding conditions in \citet{hu2022smooth}, so their lower-bound instance satisfies these assumptions. In that construction there are two arms, and for a fixed $\delta_0\in(0,1/2)$,
\[
\Delta_{\max}
\leq
C_\phi
\left(\frac{T}{4e(1/4-\delta_0^2)}\right)^{-\frac{\beta}{2\beta+d}}
\leq
T^{-\frac{\beta}{2\beta+d}}.
\]
Hence $\Delta_{k,\max}\leq\Delta_{\max}\leq
T^{-\beta/(2\beta+d)}$ for every arm $k$, verifying
Assumption~\ref{assump:Qxliplower} and membership in the stated class.
\hfill\Halmos

\subsection{Proofs in Section \ref{ec_pf_robust}}\label{ec_pf_lemmasinrobust}

\subsubsection{Proof of Lemma \ref{lem:robustepsilongreatthan1/2delta}}
By definition, \({\epsilon'_q}\geq T^{-\beta/(2\beta+d)}\), whereas
\[
 (1+2L_1\sqrt d)\delta_A
 =(1+2L_1\sqrt d)T^{-\frac{\beta}{2\beta+d}}(\log T)^{-1}.
\]
The stated inequality follows for sufficiently large \(T\).
\hfill\Halmos

\subsubsection{Proof of Lemma \ref{lem:robusttoosmallallsupport}}

Fix an arm $k$ with $\Delta_{k,\max}\leq\rho$, where
$\rho=T^{-\beta/(2\beta+d)}+(C/T)^{\beta/(\beta+d)}$.
We argue by induction. The claim holds at $q=1$.
Suppose $q\geq2$ and $S_{q-1,k}^{RS}=\cX$.
Since the preceding epoch is completed, $|\cT_{q-1}^{RS}|\leq T$.
The threshold definition, $B_{\mathrm{LP}}\geq1$, and $4K/p^*\geq1$ give
\[
\epsilon_{q-1}^{RS}
=\epsilon'_{q-1}+2B_{\mathrm{LP}}
\left(\frac{4KC}{p^*|\cT_{q-1}^{RS}|}\right)^{\beta/(\beta+d)}
\geq\rho\geq\Delta_{k,\max}.
\]
Fix a cell representative $G_j$ and choose
$b\in\arg\max_{l\in\cK_{q-1,j}^{RS}}
\widehat f_{q-1,l}^{RS}(G_j)$.
On $\overline{\mathcal G}_{q-1}^{RS}$,
\begin{align*}
0\leq\widehat f_{q-1,b}^{RS}(G_j)-\widehat f_{q-1,k}^{RS}(G_j)
&\leq f_b^*(G_j)-f_k^*(G_j)+\epsilon_{q-1}^{RS}\\
&\leq\Delta_{k,\max}+\epsilon_{q-1}^{RS}
\leq2\epsilon_{q-1}^{RS}.
\end{align*}
Thus $k$ is directly linked to an arm with the largest estimate
and is retained. Singleton active sets remain unchanged.
Induction gives $S_{q,k}^{RS}=\cX$ for all $q\leq Q^{RS}$.
\hfill\Halmos

\subsubsection{Proof of Lemma \ref{lem:robustA.6}}
Fix an arm $k$ and condition on $\overline{\mathcal G}_{q-1}^{RS}$, $\overline{\mathcal M}_{q-1}^{RS}$, and $N_{q,k}^{RS}=n$. Throughout this proof, write
\[
H:=H_{q,k}^{RS}=n^{-\frac{1}{2\beta+d}}\vee\left(\frac{C}{n}\right)^{\frac{1}{\beta+d}}.
\]
For $\bx_0\in S_{q,k}^{RS}\cap G$, define
\[
\widehat{\mathcal A}_{q,k}(\bx_0)
=\frac{1}{nH^d}\sum_{t\in\mathcal T_{q,k}^{RS}}
W\left(\frac{\bx_t-\bx_0}{H}\right)
U\left(\frac{\bx_t-\bx_0}{H}\right)U^{\mathrm T}\left(\frac{\bx_t-\bx_0}{H}\right),
\]
where $W(\bx)=\mathbb I\{\|\bx\|_2\leq1\}$ and
\[
U(u)=\left(u^r\right)_{|r|\leq\mathfrak b(\beta)}
\in\mathbb R^{M_\beta}.
\]
Let
\[
f_{k,\mathfrak b(\beta)}(\bx;\bx_0)
=\sum_{|r|\leq\mathfrak b(\beta)}
\frac{(\bx-\bx_0)^r}{r!}D^rf_k^*(\bx_0).
\]
For any scalar sequence $\xi=(\xi_t)_{t\in\mathcal T_{q,k}^{RS}}$,
define the local-polynomial score
\[
\Gamma(\xi;\bx_0)
:=\max_{|r|\leq\mathfrak b(\beta)}
\left|
\frac{1}{nH^d}\sum_{t\in\mathcal T_{q,k}^{RS}}
\xi_tU_r\left(\frac{\bx_t-\bx_0}{H}\right)
W\left(\frac{\bx_t-\bx_0}{H}\right)
\right|.
\]
Write
\begin{align*}
\Gamma_{1,1}
&:=\Gamma\bigl((\varepsilon_{k,t})_{t\in\mathcal T_{q,k}^{RS}};\bx_0\bigr),\\
\Gamma_{1,2}
&:=\Gamma\bigl((c_t)_{t\in\mathcal T_{q,k}^{RS}};\bx_0\bigr),\\
\Gamma_2
&:=\Gamma\bigl((f_k^*(\bx_t)-f_{k,\mathfrak b(\beta)}(\bx_t;\bx_0))_{t\in\mathcal T_{q,k}^{RS}};\bx_0\bigr).
\end{align*}
The local-polynomial normal equations and
$\|v\|_2\leq\sqrt{M_\beta}\|v\|_\infty$ give
\begin{align}
|\widehat f_{q,k}^{RS}(\bx_0)-f_k^*(\bx_0)|
&\leq \frac{\sqrt{M_\beta}}{\lambda_{\min}(\widehat{\mathcal A}_{q,k}(\bx_0))}
(\Gamma_{1,1}+\Gamma_{1,2}+\Gamma_2).
\label{eq:robust-lp-decomposition}
\end{align}

By the weak regularity of $S_{q,k}^{RS}$ and the density bounds, Lemma 10 of \cite{hu2022smooth}, with the bandwidth replaced by $H$, gives
\begin{align}
\mathbb P\bigl(\lambda_{\min}(\widehat{\mathcal A}_{q,k}(\bx_0))<\lambda_0 \mid \overline{\mathcal G}_{q-1}^{RS},\overline{\mathcal M}_{q-1}^{RS},N_{q,k}^{RS}=n\bigr)\leq 2M_\beta^2\exp(-C_KnH^d).
\label{eq:robust-local-design-eigenvalue}
\end{align}

The Taylor-remainder bound from the proof of Lemma 16 in
\cite{hu2022smooth} gives, whenever
$\lambda_{\min}(\widehat{\mathcal A}_{q,k}(\bx_0))\geq\lambda_0$,
\begin{equation}
\frac{\sqrt{M_\beta}}
{\lambda_{\min}(\widehat{\mathcal A}_{q,k}(\bx_0))}\Gamma_2
\leq
\frac{2\sqrt{M_\beta}Lv_dp_{\max}}{p^*\lambda_0}H^\beta.
\label{eq:robust-local-taylor}
\end{equation}
The bandwidth definition yields
\begin{equation}
H^\beta
\leq n^{-\frac{\beta}{2\beta+d}}
+\left(\frac Cn\right)^{\frac{\beta}{\beta+d}}.
\label{eq:robust-bandwidth-bias}
\end{equation}
Since $n\geq n_q$, the epoch schedule gives
$n^{-\beta/(2\beta+d)}\lesssim
(\gamma_q^{RS})^{-1/2}(\log T)^{-\beta/(\beta'-1)}
\leq\epsilon'_q/\sqrt{\log T}$.
Combining this bias bound with the clean concentration argument for
$\Gamma_{1,1}$ and the definition of $C_K$ gives

\begin{align}
&\frac{\sqrt{M_\beta}}
{\lambda_{\min}(\widehat{\mathcal A}_{q,k}(\bx_0))}\Gamma_{1,1}
+\frac{2\sqrt{M_\beta}Lv_dp_{\max}}{p^*\lambda_0}
n^{-\frac{\beta}{2\beta+d}}
\leq \frac{\epsilon'_q}{2(1+2L_1\sqrt d)}
\label{eq:robust-local-clean-bound}
\end{align}
except on an event of probability at most $4\exp\left(-C_KnH^d{\epsilon'_q}^2\right).$

The direct corruption contribution is bounded pathwise. Since
$|W(u)|\leq1$ and $|U_r(u)|\leq1$ whenever $W(u)>0$,
\begin{align}
\Gamma_{1,2}
&\leq \frac{1}{nH^d}\sum_{t\in\mathcal T_{q,k}^{RS}}|c_t|
\leq \frac{C}{nH^d}
\leq \left(\frac Cn\right)^{\frac{\beta}{\beta+d}}.
\label{eq:robust-local-corruption}
\end{align}
Therefore, whenever
$\lambda_{\min}(\widehat{\mathcal A}_{q,k}(\bx_0))\geq\lambda_0$,
\begin{equation}
\frac{\sqrt{M_\beta}}
{\lambda_{\min}(\widehat{\mathcal A}_{q,k}(\bx_0))}\Gamma_{1,2}
\leq
\frac{\sqrt{M_\beta}}{\lambda_0}
\left(\frac Cn\right)^{\frac{\beta}{\beta+d}}.
\label{eq:robust-direct-corruption}
\end{equation}
Combining \eqref{eq:robust-bandwidth-bias},
\eqref{eq:robust-local-clean-bound}, and
\eqref{eq:robust-direct-corruption}, we obtain
\[
\frac{\sqrt{M_\beta}}
{\lambda_{\min}(\widehat{\mathcal A}_{q,k}(\bx_0))}
(\Gamma_{1,1}+\Gamma_{1,2}+\Gamma_2)
\leq
\frac{\epsilon'_q}{2(1+2L_1\sqrt d)}
+
B_{\mathrm{LP}}
\left(\frac Cn\right)^{\frac{\beta}{\beta+d}}.
\]

Combining these bounds with \eqref{eq:robust-lp-decomposition} gives
\begin{align}
&\mathbb P\Biggl(
|\widehat f_{q,k}^{RS}(\bx_0)-f_k^*(\bx_0)|
>\frac{\epsilon'_q}{2(1+2L_1\sqrt d)}
+B_{\mathrm{LP}}\left(\frac Cn\right)^{\frac{\beta}{\beta+d}}
\ \Bigm|\
\overline{\mathcal G}_{q-1}^{RS},\overline{\mathcal M}_{q-1}^{RS},
N_{q,k}^{RS}=n\Biggr)\nonumber\\
&\qquad\leq
(4+2M_\beta^2)\exp\left(-C_KnH^d{\epsilon'_q}^2\right),
\label{eq:robust-grid-point-bound}
\end{align}
where $\epsilon'_q\leq1$ for sufficiently large $T$ allows the failure probability in
\eqref{eq:robust-local-design-eigenvalue} to be absorbed into the same
exponent.

Taking a union bound over $S_{q,k}^{RS}\cap G$ gives
\begin{align}
&\mathbb P\Biggl(
\sup_{\bx_0\in S_{q,k}^{RS}\cap G}
|\widehat f_{q,k}^{RS}(\bx_0)-f_k^*(\bx_0)|
>\frac{\epsilon'_q}{2(1+2L_1\sqrt d)}
+B_{\mathrm{LP}}\left(\frac Cn\right)^{\frac{\beta}{\beta+d}}
\ \Bigm|\
\overline{\mathcal G}_{q-1}^{RS},\overline{\mathcal M}_{q-1}^{RS},
N_{q,k}^{RS}=n\Biggr)\nonumber\\
&\qquad\leq
\delta_A^{-d}(4+2M_\beta^2)
\exp\left(-C_KnH^d{\epsilon'_q}^2\right).
\label{eq:robust-grid-uniform-bound}
\end{align}

Finally, for any $\bx\in S_{q,k}^{RS}$, let $g(\bx)\in G$ be the representative point used by the estimator. Since $\bx$ and $g(\bx)$ lie in the same cell,
\[
\|\bx-g(\bx)\|_2\leq\sqrt d\,\delta_A.
\]
Since $\widehat f_{q,k}^{RS}(\bx)=\widehat f_{q,k}^{RS}(g(\bx))$,
Assumption \ref{assump:betaholder} gives
\[
|\widehat f_{q,k}^{RS}(\bx)-f_k^*(\bx)|
\leq
|\widehat f_{q,k}^{RS}(g(\bx))-f_k^*(g(\bx))|
+L_1\sqrt d\,\delta_A.
\]
By Lemma \ref{lem:robustepsilongreatthan1/2delta},
\[
\frac{\epsilon'_q}{2(1+2L_1\sqrt d)}
+L_1\sqrt d\,\delta_A
\leq\frac{\epsilon'_q}{2}.
\]
Applying \eqref{eq:robust-grid-uniform-bound} therefore gives
\begin{align*}
&\mathbb P\Biggl(
\sup_{\bx\in S_{q,k}^{RS}}
|\widehat f_{q,k}^{RS}(\bx)-f_k^*(\bx)|
>\frac{\epsilon'_q}{2}
+B_{\mathrm{LP}}\left(\frac Cn\right)^{\frac{\beta}{\beta+d}}
\ \Bigm|\
\overline{\mathcal G}_{q-1}^{RS},\overline{\mathcal M}_{q-1}^{RS},
N_{q,k}^{RS}=n\Biggr)\\
&\qquad\leq
\delta_A^{-d}(4+2M_\beta^2)
\exp\left(-C_KnH^d{\epsilon'_q}^2\right).
\end{align*}
This proves the lemma. \hfill\Halmos

\section{Experimental Details}\label{ecsec:expdetail}

\subsection{Algorithm Inputs}

The theoretical algorithms involve several model-dependent constants that
are generally unknown in practice. For the numerical experiments, we use
finite-sample calibrations of these constants while preserving the functional
form of the corresponding theoretical schedules. The same calibrated
parameters are used across all replications within each experimental setting.
For benchmark algorithms, we use the parameterizations adopted in their
original implementations whenever applicable.

\subsubsection{Synthetic Dataset}

For the linear contextual bandit algorithms, we use the following
configurations. For the Fair OLS Bandit, we set
$C_a=20$, $C_b=1$, and $h=1.2$. For the Robust Fair OLS Bandit,
we use the same implementation as Algorithm~\ref{alg:robustfairols}, with
$\gamma_{\mathrm{lin}}^{\mathrm{exp}}=4$ and
$\kappa^{\mathrm{exp}}=2$. Accordingly, the initial exploration length is
$|\mathcal T_0^{RL}|=\left\lceil C_a\log T+
\gamma_{\mathrm{lin}}^{\mathrm{exp}}C\right\rceil$, and the adaptive
threshold is
$\epsilon_t^{RL}=C_b\sqrt{\log T/t}+
\kappa^{\mathrm{exp}}C/t$.
For the OLS Bandit, we use $q=2$ and $h=1.2$, which provide a
moderate forced-sampling rate for the simulated instance.
For the UCB Bandit, the confidence bound for arm $k$ at time $t$
is computed as
\[
\text{UCB}_k(t)
=
\hat{\mu}_k(t)
+
\sqrt{
x_t^{\mathrm T}
\left(
0.01I+\sum_{s\in N_k(t)}x_sx_s^{\mathrm T}
\right)^{-1}
x_t
}
\left(
0.05\sqrt{d\log\left(T(1+4t/0.01)\right)}
+0.2
\right),
\]
where $\hat{\mu}_k(t)$ is the estimated mean reward, $I$ is the identity
matrix, and $N_k(t)$ denotes the set of rounds in which arm $k$ has been
selected up to time $t$.

For the smooth contextual bandit algorithms, we use finite-horizon
calibrations of the epoch schedules in Algorithms~\ref{alg:fairsmoothalg}
and \ref{alg:robustfairsmoothalg}. Since
$\log(T\delta_A^{-d})$ is of logarithmic order in $T$, we use simplified
logarithmic factors without changing the polynomial dependence on the
epoch index or horizon.

For the Fair Smooth Bandit, we set
$|\mathcal T_q|=
\left\lceil
0.2\,4^{q(2\beta+d)/(2\beta)}
(\log T)^{(2\beta+d)/(\beta'-1)}
+\log T
\right\rceil$,
with
$\epsilon_q=
2^{-q}(\log T)^{(\beta'-1-2\beta)/(2\beta'-2)}$
and
$H_{q,k}=N_{q,k}^{-1/(2\beta+d)}$,
following Algorithm~\ref{alg:fairsmoothalg}. The final epoch is truncated
when the cumulative number of rounds reaches $T$.

For the Robust Fair Smooth Bandit, we use
$\gamma_{q,\mathrm{exp}}^{RS}
=
0.03(1\vee C)^{2\beta/(2\beta+d)}\vee4^q$
and
$\bar n_q^{RS}
=
\left\lceil
0.2(\gamma_{q,\mathrm{exp}}^{RS})^{(2\beta+d)/(2\beta)}
(\log T)^{(2\beta+d)/(\beta'-1)}
+\log T
\right\rceil$,
with the final epoch truncated as in
Algorithm~\ref{alg:robustfairsmoothalg}. The corresponding error tolerance
and bandwidth are
$$\epsilon_q^{RS}
=
\left[
2^{-q}(\log T)^{(\beta'-1-2\beta)/(2\beta'-2)}
\vee
T^{-\beta/(2\beta+d)}
\right]
+
0.3(C/|\mathcal T_q^{RS}|)^{\beta/(\beta+d)}$$
and
$H_{q,k}^{RS}
=
(N_{q,k}^{RS})^{-1/(2\beta+d)}
\vee
(C/N_{q,k}^{RS})^{1/(\beta+d)}$.
For the Simplified Smooth Bandit, we set
$c_1=0.5$ and $c_2=1$, following the baseline implementation.

\subsubsection{Semi-Synthetic Dataset}

For the linear contextual bandit algorithms, we use the following
configurations. For the Fair OLS Bandit, we set
$C_a=50$, $C_b=5$, and $h=0.8$. For the Robust Fair OLS Bandit,
we set
$\gamma_{\mathrm{lin}}^{\mathrm{exp}}=2$ and
$\kappa^{\mathrm{exp}}=0.01$, so that
$|\mathcal T_0^{RL}|
=
\left\lceil
C_a\log T+\gamma_{\mathrm{lin}}^{\mathrm{exp}}C
\right\rceil$
and
$\epsilon_t^{RL}
=
C_b\sqrt{\log T/t}
+
\kappa^{\mathrm{exp}}C/t$.
For the OLS Bandit, we use $q=2$ and $h=0.8$.
The UCB Bandit uses the same confidence-bound specification as
in the synthetic experiment.

The smooth-bandit implementations use the same functional forms as in the
synthetic experiment, with constants calibrated to the semi-synthetic
setting. For the Fair Smooth Bandit, we use
$|\mathcal T_q|
=
\left\lceil
0.15\,4^{q(2\beta+d)/(2\beta)}
(\log T)^{(2\beta+d)/(\beta'-1)}
+\log T
\right\rceil$,
with
$\epsilon_q
=
2^{-q}(\log T)^{(\beta'-1-2\beta)/(2\beta'-2)}$
and
$H_{q,k}=N_{q,k}^{-1/(2\beta+d)}$.

For the Robust Fair Smooth Bandit, we use
$\gamma_{q,\mathrm{exp}}^{RS}
=
0.008(1\vee C)^{2\beta/(2\beta+d)}
\vee4^q$
and
$\bar n_q^{RS}
=
\left\lceil
0.15(\gamma_{q,\mathrm{exp}}^{RS})^{(2\beta+d)/(2\beta)}
(\log T)^{(2\beta+d)/(\beta'-1)}
+\log T
\right\rceil$.
The corresponding error tolerance and bandwidth are
$\epsilon_q^{RS}
=
\left[
2^{-q}(\log T)^{(\beta'-1-2\beta)/(2\beta'-2)}
\vee
T^{-\beta/(2\beta+d)}
\right]
+
0.05(C/|\mathcal T_q^{RS}|)^{\beta/(\beta+d)}$
and
$H_{q,k}^{RS}
=
(N_{q,k}^{RS})^{-1/(2\beta+d)}
\vee
(C/N_{q,k}^{RS})^{1/(\beta+d)}$.
For the Simplified Smooth Bandit, we set
$c_1=0.5$ and $c_2=1$.

\subsection{Comprehensive Experimental Results}
To complement the figures in Section \ref{sec_numerical}, we present detailed experimental results in Table~\ref{tab:data_summary}. This table provides comprehensive summary statistics for all algorithms across the four experimental settings, showing the mean and standard deviation of two key evaluation metrics: cumulative regret and unfair decisions. 
\begin{table}
\TABLE
\small
{Summary of Mean and Standard Deviation for Regret and Unfairness across Four Experiments\label{tab:data_summary}}
{\begin{tabular}{@{}l@{\quad}c@{\quad}c@{\quad}c@{\quad}c@{}}
\hline\up 
Algorithm & Regret (Mean) & Regret (SD) & Unfairness (Mean) & Unfairness (SD) \\ \hline\up 
\textbf{Exp 1: Synthetic Data Uncorrupted} & & & & \\
\textit{--- Linear Setting} & & & & \\
\hspace{1em}Fair OLS (Ours)          & 425.28 & 19.40 & 12.00   & 20.41  \\ 
\hspace{1em}OLS Bandit               & 853.25 & 93.32 & 1083.20 & 111.54 \\ 
\hspace{1em}Greedy                   & 625.25 & 194.08 & 785.20  & 218.79 \\ 
\hspace{1em}UCB                      & 318.52 & 21.20 & 352.10  & 23.09  \\ 
\hspace{1em}Random                   & 11413.19 & 46.58 & 0.00    & 0.00   \\ 
\textit{--- Smooth Setting} & & & & \\
\hspace{1em}Fair Smooth (Ours)       & 344.82 & 6.61  & 0.10    & 0.30   \\ 
\hspace{1em}Smooth Bandit            & 412.78 & 12.28  & 939.40 & 10.32  \\ 
\hspace{1em}Random                   & 2082.66 & 24.59 & 0.00    & 0.00   \\ \hline\up 
\textbf{Exp 2: Synthetic Data Attack} & & & & \\
\textit{--- Linear Setting} & & & & \\
\hspace{1em}Robust Fair OLS (Ours)   & 2770.77 & 75.02 & 446.10  & 174.63 \\ 
\hspace{1em}Fair OLS (Ours)          & 7069.32 & 2120.75 & 4752.10 & 590.15 \\ 
\hspace{1em}OLS Bandit               & 6673.63 & 1133.66 & 5575.20 & 256.50 \\ 
\hspace{1em}Greedy                   & 5192.26 & 76.94  & 5053.90 & 75.53  \\ 
\hspace{1em}UCB                      & 3525.15 & 744.49 & 3902.40 & 523.55 \\ 
\hspace{1em}Random                   & 22859.78 & 131.63 & 0.00    & 0.00   \\ 
\textit{--- Smooth Setting} & & & & \\
\hspace{1em}Robust Fair Smooth (Ours)& 1382.96 & 47.95  & 147.70  & 62.55  \\ 
\hspace{1em}Fair Smooth (Ours)       & 1933.40 & 107.22 & 3043.20 & 250.21 \\ 
\hspace{1em}Smooth Bandit            & 1634.61 & 40.50  & 2994.70 & 99.57  \\ 
\hspace{1em}Random                   & 4161.99 & 26.69  & 0.00    & 0.00   \\ \hline\up 

\textbf{Exp 3: Semi-Synthetic Data Uncorrupted} & & & & \\
\textit{--- Linear Setting} & & & & \\
\hspace{1em}Fair OLS (Ours)          & 1076.82 & 27.02  & 167.50  & 9.84   \\ 
\hspace{1em}OLS Bandit               & 914.56  & 238.24 & 955.90  & 263.57 \\ 
\hspace{1em}Greedy                   & 943.63  & 473.61 & 1029.10 & 617.07 \\ 
\hspace{1em}UCB                      & 666.74  & 9.13   & 692.60  & 12.02  \\ 
\hspace{1em}Random                   & 2147.15 & 26.74  & 0.00    & 0.00   \\ 
\textit{--- Smooth Setting} & & & & \\
\hspace{1em}Fair Smooth (Ours)       & 763.67  & 22.95  & 135.30  & 15.93  \\ 
\hspace{1em}Smooth Bandit            & 1001.31 & 14.62  & 911.70  & 18.49  \\ 
\hspace{1em}Random                   & 2154.09 & 17.73  & 0.00    & 0.00   \\ \hline\up 

\textbf{Exp 4: Semi-Synthetic Data Attack} & & & & \\
\textit{--- Linear Setting} & & & & \\
\hspace{1em}Robust Fair OLS (Ours)   & 1774.66 & 47.27  & 251.40  & 76.08  \\ 
\hspace{1em}Fair OLS (Ours)          & 2736.99 & 5.70   & 3408.70 & 10.39  \\ 
\hspace{1em}OLS Bandit               & 2779.30 & 1.60   & 3653.90 & 2.02   \\ 
\hspace{1em}Greedy                   & 2782.42 & 0.41   & 3660.30 & 0.46   \\ 
\hspace{1em}UCB                      & 2769.76 & 2.71   & 3640.40 & 1.96   \\ 
\hspace{1em}Random                   & 2147.15 & 26.74  & 0.00    & 0.00   \\ 
\textit{--- Smooth Setting} & & & & \\
\hspace{1em}Robust Fair Smooth (Ours)& 1620.39 & 45.30  & 239.20  & 68.63  \\ 
\hspace{1em}Fair Smooth (Ours)       & 2606.14 & 44.01  & 2895.30 & 83.48  \\ 
\hspace{1em}Smooth Bandit            & 1979.19 & 66.93  & 2155.80 & 131.98 \\ 
\hspace{1em}Random                   & 2150.66 & 34.26  & 0.00    & 0.00   \down \\ \hline
\end{tabular}}
{Note: ``Uncorrupted'' refers to the stochastic setting without adversarial corruption; ``Attack'' denotes the presence of $C$-budgeted corruption. ``Linear'' and ``Smooth'' settings correspond to different reward functions and applicable algorithms. All metrics are calculated over 10 independent runs.}
\end{table}

\end{document}